\documentclass{article}

\usepackage{microtype}
\usepackage{graphicx}
\usepackage{booktabs} % for professional tables

\usepackage{hyperref}

\usepackage[accepted]{icml2025}

\usepackage{amsmath}
\usepackage{amssymb}
\usepackage{mathtools}
\usepackage{amsthm}
\usepackage{multirow}
\usepackage{makecell}
\usepackage{tabu}
\usepackage{color}
\usepackage{colortbl}
\usepackage{bm}
\usepackage{algorithm}
\usepackage{algorithmic}
\usepackage{subfigure}

\usepackage[capitalize,noabbrev]{cleveref}
\theoremstyle{plain}
\newtheorem{theorem}{Theorem}[section]
\newtheorem{proposition}[theorem]{Proposition}

\newtheorem{corollary}[theorem]{Corollary}
\newtheorem{property}[theorem]{Property}
\theoremstyle{definition}
\newtheorem{definition}[theorem]{Definition}
\newtheorem{assumption}[theorem]{Assumption}
\theoremstyle{remark}
\newtheorem{remark}[theorem]{Remark}

\DeclareMathOperator*{\argmax}{arg\,max}
\DeclareMathOperator*{\argmin}{arg\,min}

\usepackage[textsize=tiny]{todonotes}

\icmltitlerunning{Data Attribution in Contrastive Learning}

\begin{document}

\twocolumn[
\icmltitle{Dissecting Representation Misalignment in \\ Contrastive Learning via Influence Function}



\icmlsetsymbol{equal}{*}

\begin{icmlauthorlist}
\icmlauthor{Lijie Hu}{equal,1,2}
\icmlauthor{Chenyang Ren}{equal,1,2,3}
\icmlauthor{Huanyi Xie}{1,2,4}
\icmlauthor{Khouloud Saadi}{1,2}\\
\icmlauthor{Shu Yang}{1,2}
\icmlauthor{Zhen Tan}{7}
\icmlauthor{Jingfeng Zhang}{5,6}
\icmlauthor{Di Wang}{1,2}
\end{icmlauthorlist}

\icmlaffiliation{1}{Provable Responsible AI and Data Analytics (PRADA) Lab}
\icmlaffiliation{2}{King Abdullah University of Science and Technology}
\icmlaffiliation{3}{Shanghai Jiao Tong University}
\icmlaffiliation{4}{KTH Royal Institute of Technology}
\icmlaffiliation{5}{The University of Auckland}
\icmlaffiliation{6}{RIKEN Center for Advanced Intelligence Project (AIP)}
\icmlaffiliation{7}{Arizona State University}

\icmlcorrespondingauthor{Di Wang}{di.wang@kaust.edu.sa}

\icmlkeywords{Machine Learning, ICML}

\vskip 0.3in
]



\printAffiliationsAndNotice{\icmlEqualContribution} % otherwise use the standard text.

\begin{abstract}
Contrastive learning, commonly applied in large-scale multimodal models, often relies on data from diverse and often unreliable sources, which can include misaligned or mislabeled text-image pairs. 
This frequently leads to robustness issues and hallucinations, ultimately causing performance degradation.
Data valuation is an efficient way to detect and trace these misalignments. Nevertheless, existing methods are computationally expensive for large-scale models. 
Although computationally efficient, classical influence functions are inadequate for contrastive learning models, as they were initially designed for pointwise loss.
Furthermore, contrastive learning involves minimizing the distance between positive sample modalities while maximizing the distance between negative sample modalities.
This necessitates evaluating the influence of samples from both perspectives.
To tackle these challenges, we introduce the Extended Influence Function for Contrastive Loss (ECIF), an influence function crafted for contrastive loss. ECIF considers both positive and negative samples and provides a closed-form approximation of contrastive learning models, eliminating the need for retraining. Building upon ECIF, we develop a series of algorithms for data evaluation, misalignment detection, and misprediction trace-back tasks. Experimental results demonstrate our ECIF advances the transparency and interpretability of CLIP-style embedding models by offering a more accurate assessment of data impact and model alignment compared to traditional baseline methods. 
\end{abstract}

\section{Introduction}
% Multi-modal Large Language models (MLLMs) \citep{yin2023survey,koh2024generating} have garnered significant attention for their ability to integrate and understand various types of data, such as image, text, and audio. Despite their growing application, existing MLLMs often suffer from robustness issues ~\citep{carlini2021poisoning} and hallucinations, primarily stemming from misaligned text-image pairs in the training data~\citep{Kim_2023_ICCV}. These misalignments, manifesting as semantic mismatches, contextual inconsistencies, or discrepancies between abstract and concrete elements, can severely degrade model performance. MLLMs assume consistent alignment between image-text pairs, but when this assumption fails, it leads to incorrect interpretations, ultimately degrading model performance. Consequently, improving dataset transparency is crucial, as model developers need the ability to trace and identify problematic data samples. However, diagnosing issues caused by misaligned data, such as mislabeled or biased samples, is difficult when working with large text-image datasets.

Contrastive learning has become a cornerstone in the development of multimodal models due to its ability to align representations from different modalities—such as images, text, and audio—within a shared semantic space~\citep{chen2020simpleframeworkcontrastivelearning, yin2023survey,koh2024generating}. However, models trained with contrastive learning often suffer from robustness issues~\citep{carlini2021poisoning} and hallucinations, which are primarily attributed to misaligned text-image pairs in the training data~\citep{Kim_2023_ICCV}. These misalignments, manifesting as semantic mismatches, contextual inconsistencies, or discrepancies between abstract and concrete elements, can severely degrade model performance. Contrastive learning relies on the assumption of consistent alignment between image-text pairs; however, when this assumption fails, it leads to incorrect interpretations, ultimately degrading model performance. Consequently, improving dataset transparency is crucial, as model developers need the ability to trace and identify problematic data samples. However, diagnosing issues caused by misaligned data, such as mislabeled or biased samples, is difficult when working with large text-image datasets.

Although the critical role of training data in shaping the capabilities of multimodal models is well recognized, robust evaluation mechanisms for data quality remain lacking~\citep{NEURIPS2022}. To address this, various data valuation methods~\citep{jia2019towards,ghorbani2019data,yoon2020data,han2020explaining} have been introduced to enhance dataset transparency by quantifying the contribution of individual data points to model performance. These approaches typically assign higher contribution scores to training instances whose inclusion significantly boosts model performance compared to their exclusion. Some methods, such as Shapley Value~\citep{pmlr-v151-kwon22a}, require multiple retraining processes with different subsets of data, which is computationally expensive and impractical for large models. To overcome this limitation, influence function-based methods have gained popularity, as they estimate data contributions using gradient information, thereby avoiding retraining~\citep{choe2024your}.  

However, applying influence functions to models trained with contrastive learning presents significant challenges. (i) First, the influence function was initially designed for M-estimators~\citep{huber1981robust}, which operate with pointwise loss. However, multimodal models rely on noise-contrastive estimation~\citep{radford2021learning,gutmann2010noise,he2020momentum} as their training objective. This objective encourages the model to draw positive pairs closer in feature space while pushing negative pairs apart, making the influence function unsuitable for direct application to contrastive loss. (ii) Second, the influence of negative pairs in contrastive learning has gained increasing attention recently~\citep{oord2019, yuksekgonul2023when}. \citet{robinson2021contrastive} emphasized the importance of negative samples, especially ``hard'' negatives - samples that are mapped close in feature space but should ideally be far apart. However, the original definition of the influence function does not consider the roles of positive and negative samples. This oversimplified analysis is particularly prone to underestimating the impact of certain hard negative samples on the learning process~\citep{chen2020multi}. (iii) Lastly, computing the required gradients and Hessian matrices for influence functions is highly demanding in terms of computational and memory resources, making it impractical in the large-scale, high-dimensional setting of contrastive learning~\citep{NEURIPS2023_996e2b44,Li_2023_CVPR}.

To address these challenges, we propose the \textit{Extended Influence Function for Contrastive Loss (ECIF)}, a novel method designed to quantify data attribution specifically for contrastive learning. ECIF enjoys a closed-form approximation of the original contrastive loss, thus eliminating the need for re-training - a process that is impractical in the era of large models. It also accounts for the dual role of data points as both positive and negative samples, providing a more comprehensive understanding of their impact on model training. This approach provides a more accurate measurement of misalignment. Our contributions are summarized as follows:
\begin{itemize}
    \item We propose ECIF, the first dual-perspective data valuation method for contrastive learning, which quantifies the impact of data points as both positive and negative samples. This comprehensive approach enables a more accurate measurement of data contribution, particularly addressing the influence of negative samples in contrastive learning.
    \item Based on ECIF, we develop corresponding algorithms for different tasks, including identifying the most valuable data (related to specific tasks), misalignment detection, and misprediction trace-back. 
    % \item \warn{ECIF can be applied for counterfactual evaluation during fine-tuning, enabling the identification of the most valuable data related to specific tasks, tracing back mispredicted data, and detecting misalignments or mislabeled samples.}
    \item Comprehensive experimental results demonstrate that ECIF can effectively and efficiently remove the influence of samples compared to retraining and identify influential data in the training set. Moreover, our methods based on ECIF are also effective in identifying influential data (harmful data and valuable data) for fine-tuning,  misprediction trace-back, and detecting misaligned data. 
\end{itemize}

\section{Related Work}
\noindent {\bf Contrastive Learning.} Recently, self-supervised contrastive learning \citep{pmlr-v119-chen20j} has emerged as a highly effective approach for acquiring representations without the need for labeled data \citep{donahue2019large}. This model utilizes a contrastive loss, which pushes dissimilar data pairs apart while pulling similar pairs closer together.
%Models leveraging contrastive losses have achieved significantly better performance compared to other methods.

Contrastive learning plays a pivotal role in advancing multimodal models by integrating and understanding information across diverse modalities, such as text and images \citep{radford2021learning,jiang2024hallucination}. In multi-modal contrastive learning tasks, proper alignment of the training data ensures accurate cross-modal associations, enabling models to learn and extract consistent feature representations \citep{wang2020understanding}. One of the key challenges in training with noisy, large-scale image-text pairs sourced from the internet is achieving effective alignment between these modalities. To address this, researchers have developed various methods, such as those proposed by \citet{gao2022pyramidclip} and \citet{yao2021filip}, which introduce finer-grained and more extensive interactions between text and images to improve cross-modal alignment. Despite extensive research on contrastive learning, we are the first to explore the interactive influence between pairs using influence functions. Our work bridges this gap by applying influence functions in contrastive learning, allowing for a deeper understanding of both positive and negative samples. This comprehensive approach enhances the accuracy of misalignment measurements in data pairs, providing a more thorough assessment of data valuation.

\noindent {\bf Influence Function.} Influence function, initially a staple in robust statistics \citep{cook2000detection,cook1980characterizations}, has seen extensive adoption within deep learning since \citep{koh2017understanding}. Its versatility spans various applications, including detecting mislabeled data, interpreting models, addressing model bias, and facilitating machine unlearning tasks. For data removal, recent work using influence function including unlearning features and labels \citep{warnecke2021machine}, forgetting a subset of image data for training deep neural networks \citep{golatkar2020eternal,golatkar2021mixed}, removing the influence of nodes and edges in graph neural networks~\cite{wu2023gif}, and model debiasing~\citep{chen2024fast}.  Besides, various studies have applied influence functions to interpret models across different domains, including natural language processing \citep{han2020explaining} and image classification \citep{basu2021influence}, while also addressing biases in classification models \citep{wang2019repairing}, word embeddings \citep{brunet2019understanding}, and finetuned models \citep{chen2020multi}.
Recent advancements, such as the LiSSA method \citep{agarwal2017second,kwon2023datainf, grosse2023studying} and kNN-based techniques \citep{guo2021fastif}, have been proposed to enhance the computational efficiency of computing the influence function.  Despite numerous studies on influence functions, we are the first to extend them to contrastive learning. Moreover, compared to traditional models, contrastive learning introduces additional complexity in influence function analysis, as it requires considering data points in both positive and negative roles. Our dual-perspective approach of ECIF offers a more comprehensive view of data impact, leading to more accurate measurements of misalignment in text-image pairs. Bridging the theoretical gap between positive and negative pairs has posed significant challenges in our work, which has been addressed in our proof.
\section{Preliminaries}
\label{sec:preliminary}
\noindent{\bf Contrastive Loss.}
Contrastive loss is an effective tool in multi-modal models for aligning and learning relationships between different types of data, such as images and text \footnote{For simplicity, we focus on two modalities (text and image) in the paper. Our method can be generalized to multi-modalities directly.}. Specifically, given a set of paired data consisting of text $x^{T}$ and image $x^{I}$, we aim to construct embedding vectors $u$ and $v$ for text and image, respectively, via the encoder parameterized as $\theta$. In a batch of $N$ text-image pairs, each pair $(x^T_k, x^I_k)$ is embedded as $(u_k, v_k)$. We denote the text embeddings for this batch as $U = (u_1, \ldots, u_N)$, and similarly, the image embeddings as $V = (v_1, \ldots, v_N)$.

Contrastive loss is designed to minimize the distance between embeddings of matching pairs while maximizing the distance between non-matching pairs. Define the cosine similarity function as $s(u, v) = \frac{u\cdot v^{\mathrm{T}}}{\|u\|\|v\|}/\tau$, where $\tau$ is a trainable temperature parameter. For brevity, we will omit detailing $\tau$ in subsequent discussions. For each batch, we construct a similarity matrix $S$ with $S_{i,j} = s(u_i, v_j)$. Then, the self-supervised contrastive loss is defined as 
\begin{align}
    & L_{\text{Batch}}(U, V; \theta) \notag \\ 
    =& \sum_{i=1}^N -\log (e_i \cdot \sigma(S_{i, *}) -\log (e_i \cdot \sigma(S^{\mathrm{T}}_{*, i})) \label{eq:0} \\ 
    =&\sum_{i=1}^N L_{T2I}(u_i, V; \theta)+ L_{I2T}(v_i, U; \theta), 
\label{eq:2}
\end{align}
where $e_i$ is the $i$-th standard basis vector in $N$-dimensional space, $\sigma$ is the softmax function. Observing from (\ref{eq:0}), we can separate the loss to image-to-text (I2T) and text-to-image (T2I) denoted in (\ref{eq:2}) and  define loss function on similarity matrix as $L_{T2I}(S;\theta)$ (and $L_{I2T}(S;\theta)$). We will incorporate an $L_2$ regularization term into the loss function, which allows us to avoid overfitting. Thus, for a given set of batches $\mathcal{B}$, the objective loss can be written as 
\begin{equation}\label{eq:3}
    L_{\text{Total}}(\mathcal{B}; \theta)=\sum_{(U,V)\in \mathcal{B}} L_{\text{Batch}}(U, V;\theta) + \frac{\delta}{2} \|\theta\|^2_2.
\end{equation}

\paragraph{Influence Functions.} 
The influence function quantifies how an estimator relies on the value of each point in the sample. Consider a neural network $\hat{\theta}= \argmin \sum_{i=1}^n \ell(z_i; \theta)$ with pointwise loss function $\ell$ and dataset $D=\{z_i\}_{i=1}^n$. When we remove a point $z_m$ from the training dataset, the corresponding optimal model is denoted as $\hat{\theta}_{-z_m}$. The influence function provides an efficient way to approximate $\hat{\theta}_{-z_m} - \hat{\theta}$ for a strongly convex and twice differentiable $\ell$. By up-weighing $z_m$ by $\epsilon$, we denote the substitutional parameter via the $\epsilon$-parameterized response function as
\begin{equation}\label{eq:response}
   \hat{\theta}_{-z_m}(\epsilon) = \argmin  \frac{1}{n}\sum_{i=1}^{n} \ell(z_{i} ; \theta) + \frac{\epsilon}{n} \ell(z_m;\theta).
\end{equation}
Then we can obtain an estimator for the actual change in parameters as:$\lim_{\epsilon \to -1} \hat{\theta}_{-z_m}(\epsilon) -\hat{\theta} = -{H}_{\hat{\theta}}^{-1} \cdot\nabla_\theta \ell (z_m; \hat{\theta})$, where ${H}_{\hat{\theta}} = \sum_{i=1}^n\nabla_\theta^2 \ell(z_i; {\hat{\theta}}) + \delta I$ is the Hessian matrix at the point of $\hat{\theta}$. 

For a differentiable model evaluation function $f$, such as calculating the total model loss over a test set, the change resulting from removing $z_m$ in the evaluation results can be approximated by 
\begin{align*}
 f(\hat{\theta}_{-z_m})-f(\hat{\theta}) &\approx \nabla_{\theta} f(\hat{\theta}) (\hat{\theta}_{-z_m} - \hat{\theta}) \\ \notag
 &\approx  -\nabla_{\theta} f(\hat{\theta}) \cdot{H}_{\hat{\theta}}^{-1} \nabla_{\theta} \ell(z_m; \hat{\theta}).   
\end{align*}
Scaling gradient-based methods to contrastive loss is challenged by the high computational and memory demands due to the gradients' high dimensionality. \citet{choe2024your} introduced a low-rank gradient projection algorithm (LOGRA) to enhance the efficiency of gradient projection. They observed that the gradient from backpropagation is structured as a sum of Kronecker products of forward and backward activations. LOGRA applies an additional Kronecker-product structure to the projection matrix  $P\triangleq P_i \otimes P_o$. It first projects the forward and backward activations onto low-dimensional spaces using $P_i$ and $P_o$, respectively, and then reconstructs the projected gradient directly from these reduced activations. For more details, see Appendix \ref{app:logra}.

%For non-convex loss, \cite{bartlett1953approximate} proposed that the Hessian HˆθH_{{\hat{\theta}}} can be replaced by ˆH=Gˆθ+δI\hat{H} =G_{\hat{\theta}} + \delta I
%where GˆθG_{\hat{\theta}} is the Fisher information matrix defined by n−1∑ni=1∇θℓ(zi;ˆθ)∇θℓ(zi;ˆθ)Tn^{-1}\sum_{i = 1}^n \nabla_{\theta}\ell(z_i;\hat{\theta})\nabla_{\theta}\ell(z_i;  \hat{\theta})^{\mathrm{T}}. \warn{We can further employ the Eigenvalue-corrected Kronecker-Factored Approximate Curvature (EK-FAC) method to further accelerate the computation. See Appendix for details.}
\section{Influence Function in Contrastive Learning}\label{method}
In this section, we will consider how to estimate the attribution of a given sample $(x^T, x^I)$ in the contrastive loss (\ref{eq:3}) using the influence function method. Generally, in the original influence function method, then term in the loss function that only contains the full information from the target sample is up-weighted by $\epsilon$. Then, a response function in (\ref{eq:response}) related to $\epsilon$ is derived. Within this analytical framework, when $\epsilon$ is set to $-1$, the resultant loss and model parameters are the same as those obtained by removing the sample via retraining. However, in the context of contrastive learning, because the information of the sample point appears in every term of the loss function for its batch, it is not feasible to isolate the relevant information of this sample within a batch into an independent term and then perform an up-weight operation on this sample to derive the influence function. 

Thus, we must execute a fine-grained analysis of the specific contribution of sample $(x^T, x^I)$ within contrastive loss. Assume $(x^T, x^I)$ is assigned as the $n$-th pair in the $m$-th batch, in which the text and image data are embedded into matrix $U_m$ and $V_m$. Then $(x^T, x^I)$ serves as positive samples for each other in the $n$-th pairing loss $L_{T2I}(u_n, V_m; \theta)$ and $L_{I2T}(v_n, U_m; \theta)$ in (\ref{eq:2}). And it serves as a negative sample in other pairing losses. 

It can be noted through simple observation about (\ref{eq:2}) that when the data serves as a positive sample, its influence can be explicitly isolated. However, its information is coupled with other data when acting as a negative sample, necessitating further analysis. We provide the derivation of the influence function for these two scenarios separately.

% However, this method encounters difficulties with contrastive loss, which depends on a similarity matrix. Since adjusting weights does not alter the dimensions of the similarity matrix, this constraint prevents us from identifying a series of loss functions through weight modification that can convergent to the loss function obtained when a sample is removed. Therefore, we will conduct a fine-grained analysis of the specific contribution of sample (xT,xI)(x^T, x^I) within contrastive loss, quantifying the contribution as positive sample and negative sample.

% Then what are the positive and negative samples? Assume (xT,xI)(x^T, x^I) is assigned as the nn-th pair in the mm-th batch, in which the text and image data are embedded into matrix UmU_m and VmV_m. Then (xT,xI)(x^T, x^I) serves as positive samples for each other in the nn-th pairing loss LT2I(un,Vm;θ)L_{T2I}(u_n, V_m; \theta) and LI2T(vn,Um;θ)L_{I2T}(v_n, U_m; \theta) . And xIx^I and xTx^T serve as negative samples in other pairing losses.

\subsection{Influence as Positive Samples}\label{positive IF}
To quantify the impact of $x^T$ and $x^I$ as positive samples, ideally, we can retrain the model after removing the corresponding $n$-th pairing tasks, i.e., removing $L_{T2I}(u_n, V_m; \theta)$ and $L_{I2T}(v_n, U_m; \theta)$ in the loss function. Thus, following the idea of influence function, we can up-weight these two parts by $\epsilon$ and obtain an up-weighted loss function as the following with $\text{Pos}(x^T,x^I, {\theta})=L_{T2I}(u_n, V_m; \theta) + L_{I2T}(v_n, U_m; \theta)$.
\begin{align*}
    L_{\text{Total}, \epsilon}(\theta) & = {\sum_{(U,V)\in \mathcal{B}}}L_{\text{Batch}}(U,V;{\theta}) +\frac{\delta}{2}\|\theta\|^2_2  \\
    & + \epsilon \cdot \text{Pos}((x^T,x^I); {\theta}). 
\end{align*}
And the parameters are obtained by $\hat{\theta}_{\epsilon} = \argmin_{\theta} L_{\text{Total}, \epsilon}(\theta)$. Then the influence function related to parameters can be deduced as:
\begin{equation}\label{eq:4}
    \text{positive-IF}((x^T, x^I); \hat{\theta}) = -H_{\hat{\theta}}^{-1} \cdot \nabla_{\theta} \text{Pos}((x^T,x^I); \hat{\theta}).
\end{equation}
where $H_{\hat{\theta}} = \nabla_{\theta}^2 {\sum_{(U,V)\in \mathcal{B}}}L_{\text{Batch}}(U,V;\hat{\theta})+\delta I$ is the Hessian matrix at $\hat{\theta}$. The proof can be found in Section \ref{posi:single}. 
% The parameter ˆθ-posobtainedwhentwoobtainedwhentwoxTandandxI\hat{\theta}_{\text{-pos}}obtainedwhentwo obtained when two x^Tand and x^I pairing tasks are removed and the model is retrained on the following loss function by
% \begin{equation}\label{positive:loss}
%     L_{\text{Total}, \text{-positive}} = {\sum_{(U,V)\in \mathcal{B}}}L_{\text{Batch}}(U,V;{\theta})   - \text{Pos}(x^T,x^I, {\theta})
% \end{equation}
% And ˆθ-poscanbeestimatedbycanbeestimatedbyˆθ+positive-IF((xT,xI))\hat{\theta}_{\text{-pos}}canbeestimatedby can be estimated by \hat{\theta} + \text{positive-IF}((x^T, x^I)).

%Thus, it can also be used to modify the model as an alternative to retraining.
\noindent {\bf Extension to Multiple Samples.} The influence evaluation described above can be extended to a subset $\mathcal{D}^*\subset \mathcal{D}$. Let set $S$ to index the batches containing data from $\mathcal{D}^*$. For every $m\in S$, define an index set $E_m$ to specify the position of data from $\mathcal{D}^*$ within the $m$-th batch. We encapsulate the assigned results as $\text{Seg} = \{(m, E_m)\vert m\in S\}$. By employing a derivation method similar to that used for a single data point, we can obtain the parameter-related influence function for $\mathcal{D}^*$ by summing the influence (\ref{eq:4}) for all samples in $\mathcal{D}^*$.
\begin{proposition}\label{posi-if}
The influence function for dataset $\mathcal{D}^*$ serving as positive samples (positive-IF) can be approximated by 
    \begin{equation*}
    \text{positive-IF}(\mathcal{D}^*, \text{Seg}; \hat{\theta}) = -H_{\hat{\theta}}^{-1} \cdot \nabla_{\theta} \text{Pos}(\mathcal{D}^*, \text{Seg}, \hat{\theta}),
\end{equation*}
where
\begin{align*}
    &\text{Pos}(\mathcal{D}^*, \text{Seg}; \hat{\theta}) \\
    &= \sum_{m\in S}\sum_{n\in E_m}\left( L_{T2I}(u_n, V_m; \hat{\theta}) + L_{I2T}(v_n, U_m; \hat{\theta})\right).
\end{align*}
\end{proposition}

\subsection{Influence as Negative Samples}
In Section \ref{positive IF}, we quantified the impact of $x^T$ and $x^I$ as positive samples by removing related pairing tasks. 
Next, we estimate their impact as negative samples by removing them from tasks that serve as negative samples. To achieve this, we need to delve into the specific form of contrastive loss.

Taking the text2image (T2I) loss for the $k$-th text embedding $u_k$ as an example, we first calculate its similarity with all image embeddings in the batch to form a similarity vector $S(u_k, V)$, which is then processed through a softmax layer $\sigma(\cdot)$ to yield a probability distribution. The $k$-th element indicates the probability of correctly pairing the text \( u_k \) with its corresponding image: $[\sigma(S(u_k, V))]_k = \frac{e^{S_{k,k}}} {\sum_{\substack{j\in [B]}} e^{S_{k,j}}}$, where $B$ is the batchsize.
% \begin{equation*}
%     [\sigma(S(u_k, V))]_k = \frac{e^{S_{k,k}}} {\sum_{\substack{j\in [B]}} e^{S_{k,j}}},
% \end{equation*}
% = \frac{e(u_k, v_k)} {\sum_{\substack{j\in [B]}} e(u_k,v_j)}
% where BB is the batchsize.
The model is encouraged to enhance the probability of correct pairing by minimizing the negative logarithm of this value. 
% And after removing (xT,xI)(x^T, x^I) influence as positive samples, the T2I loss function of this batch becomes:
% \begin{equation}\label{batch loss:-posi}
%     \text{T2I-}L_{\text{-pos}}^m((x^T, x^I)) = \sum_{\substack{k\in [B]\\k\neq n}} \frac{e^{s(u_k,v_k)}} {\sum_{\substack{j\in [B]}} e^{s(u_k,v_j)}}
% \end{equation}
For $n \neq k$, $v_n$ serves as a negative sample in this task and appears in the ${S_{k,n}}$ term in the denominator. Thus, after removing the impact of $(x^T, x^I)$ as a negative sample from the $m$-th batch, the loss function corresponding to this batch should become:
\begin{align}\label{eq:6}
    L_{\text{T2I}, \text{ -neg}}^m((x^T, x^I), S;\theta) &= \sum_{\substack{k\in [B]\\k\neq n}} -\log \frac{e^{S_{k,k}}} {\sum_{\substack{j\in [B]\\j\neq n}} e^{S_{k,j}}}  \\ \notag
    &+ \text{Pos}((x^T,x^I); {\theta}). 
\end{align}
The original influence function method evaluates a data point's impact by adjusting its weight via a separate term in the loss function and getting the response function (\ref{eq:response}). In contrastive learning, however, the influence of data points as negative samples is coupled with information from other data, which can observed from (\ref{eq:6}). We will separate an influence term related to the data effect when it serves as a negative sample. Actually, the modification in (\ref{eq:6}) is analogous to eliminating the $n$-th row and column from the original similarity matrix. Leveraging the idea of deriving the influence function, we aim to develop a response function that converges to the target loss by up-weighting some specific components. 

Considering that similarities vectors are processed through the softmax layer, if we increase the similarity associated with $u_n$ and $v_n$ to a value approaching negative infinity, then after the exponential operation and the logarithmic function, the influence of $e^{S_{*,n}}$ and $e^{S_{n,*}}$ will become negligible.
Mathematically, let $E_n$ be an $B\times B$ matrix such that its $n$-th column and the $n$-th row comprises ones, while all other entries are zero. 
We add the matrix $\log\zeta\times E_n$ to the similarity matrix. Then the loss function based on the revised similarity matrix becomes:
\begin{align}\label{eq:7}
&L_{\text{T2I}, \zeta}^m ((x^T, x^I), S; \theta) \\ \notag
=&  \sum_{\substack{k\in [B]\\k\neq n}}-\log \frac{e^{S_{k,k}}} {\sum_{\substack{j\in [B]}} e^{S_{k,j}}+  \left(\zeta-1\right) \cdot e^{S_{k,n}}} \\ \notag
&+ \text{Pos}((x^T,x^I); {\theta}). 
\end{align}
We can easily see that as $\zeta$ approaches $0$, the loss function $L_{\text{T2I}, \zeta}^m$ in (\ref{eq:7}) converges to $L_{\text{T2I, -neg}}^m$ in (\ref{eq:6}). When $\zeta=1$, the loss function equals the original one. To further separate this influence as negative samples from the original loss function, we perform a Taylor expansion at $\zeta=1$ and drop the $O(\left(\zeta-1\right)^2)$ term, then $L_{\text{T2I}, \zeta}^m$ becomes 
\begin{align*}
    & L_{\text{T2I}}^m (S; \theta)+ \left(\zeta-1\right)\cdot\sum_{\substack{k\in [B]\\k\neq n}}\left(\frac{\sum_{\substack{j\in [B]}} e^{S_{k,j}}}{e^{S_{k,n}}}\right)\\
    & \xrightarrow{\zeta\rightarrow 0}L_{\text{T2I}}^m (S; \theta)-\sum_{\substack{k\in [B]\\k\neq n}}\left(\frac{\sum_{\substack{j\in [B]}} e^{S_{k,j}}}{e^{S_{k,n}}}\right),
\end{align*}
and the left side is an estimation for (\ref{eq:7}). The minus term indicates the influence of $(x^T, x^I)$ as negative samples. By employing a similar method, one can obtain $L_{\text{I2T}, \zeta}^m$ for the image2text part. Denote $\text{Neg}\left((x^T,x^I);\theta\right)$ as 
\begin{equation*}
  % \text{Neg}\left((x^T,x^I);\theta\right)=
  \sum_{\substack{k\in [B]\\k\neq n}}\left(\frac{\sum_{\substack{j\in [B]}} e^{S_{k,j}}}{e^{S_{k,n}}} +\frac{\sum_{\substack{j\in [B]}} e^{S_{j,k}}}{e^{S_{n,k}}}\right), 
\end{equation*}
Down-weighting the influence as a negative sample by $\zeta$ from $1$ to $0$, this influence in the loss function is then approximately eliminated. Then, the negative-influence function related to parameters can be deduced as:
\begin{equation*}
    \text{negative-IF}((x^T, x^I);\hat{\theta}) = -H_{\hat{\theta}}^{-1} \cdot \nabla_{\theta} \text{Neg}((x^T,x^I);\hat{\theta}). 
\end{equation*}
Similar to the previous section, we can extend a single sample to a set of samples
$\mathcal{D}^*$ and corresponding positional index $\text{Seg}$.
\begin{proposition}\label{nega-if}
The influence function for dataset $\mathcal{D}^*$ serving as negative samples (negative-IF)  can be approximated by 
$$
    \text{negative-IF}(\mathcal{D}^*,\text{Seg}; \hat{\theta}) = -H_{\hat{\theta}}^{-1} \cdot \nabla_{\theta}  \text{Neg}(\mathcal{D}^*,\text{Seg};\hat{\theta}),
$$
with $\text{Neg}(\mathcal{D}^*,\text{Seg};\hat{\theta})$ defined as
\begin{align*}
\sum_{m\in S}\sum_{\substack{k\in [B]/E_m}}\left(\frac{\sum_{\substack{j\in [B]}} e^{S_{k,j}}}{\sum_{n\in E_m}e^{S_{k,n}}} +\frac{\sum_{\substack{j\in [B]}} e^{S_{j,k}}}{\sum_{n\in E_m}e^{S_{n, k}}}\right).
\end{align*}
\end{proposition}
Combining Proposition \ref{posi-if} and \ref{nega-if} together, we then define our influence function method on contrastive learning(ECIF) as follows.
\begin{definition}[ECIF]\label{pro:ECIF}
The extended influence function for contrastive loss (ECIF) of the target dataset  $\mathcal{D}^*$ with its position index set $\text{Seg} = \{(m, E_m)\vert m\in S\}$ is defined as
\begin{align*}
    &\text{ECIF}(\mathcal{D}^*, \text{Seg};\hat{\theta}) \\
    \triangleq &\left(\text{positive-IF}(\mathcal{D}^*,\text{Seg};\hat{\theta}), \text{negative-IF}(\mathcal{D}^*,\text{Seg};\hat{\theta})\right).
\end{align*}
\end{definition}
Thus, we can employ ECIF to estimate the changes in model parameters resulting from data removal.\footnote{For brevity, we refer to positive-IF as posi-IF and negative-IF as nega-IF.}
% \begin{corollary}
% The m
% \begin{equation*}
%     \tilde{\theta} = \hat{\theta} +  \left.\frac{\mathrm{d} {\theta}_{\epsilon, \zeta = 0}(\mathcal{D}^*, \text{Seg};\hat{\theta})}{\mathrm{d} \epsilon}\right|_{\epsilon = 0} + \left.\frac{\mathrm{d} {\theta}_{\epsilon = 0, \zeta}(\mathcal{D}^*, \text{Seg};\hat{\theta})}{\mathrm{d} \zeta}\right|_{\zeta = 0}.
% \end{equation*}
 We also give an upper bound on the error between the estimated influence given by ECIF and the actual influence obtained by model retraining in Appendix \ref{app:error_bound} for convex loss. We show that under certain scenarios, the approximation error becomes tolerable theoretically. 
%\end{corollary}
\section{Applications of ECIF}
We have proposed ECIF to evaluate the contribution of training data in contrastive learning. The ECIF method enables us to estimate the change in the learned parameters $\hat{\theta}$ if a training example pair is removed. Based on this, in this section, we will apply ECIF to two applications: misalignment detection and misprediction trace-back.
%This is achieved by performing a first-order approximation to estimate the change in ˆθ\hat{\theta} around ϵ=0\epsilon = 0 and ζ=0\zeta = 0 respectively and then linearly superimpose these two changes as the total change in ˆθ\hat{\theta}. 

\subsection{Misalignment Detection}
Contrastive learning typically assumes a consistent alignment between all image-text pairs, and thus, misaligned data can lead to incorrect interpretations of these relationships, ultimately degrading model performance. Intuitively, given a high-quality validation data $D'$, if $D^*$ is a misaligned set, then the loss of $D'$ over the original model $\hat{\theta}$ should be greater than it over the model after deleting these misaligned data. And such a difference can be approximated by ECIF.

\begin{property}\label{task-IS}
Considering a specific  set $\mathcal{D}'$ with text and image embeddings $U'$ and $V'$, and a dataset $D^*$ to be removed, then we have 
%. Then its Task-related Influence Score related to D∗\mathcal{D}^*  with Seg={(m,Em)|m∈S}\text{Seg} = \{(m, E_m)\vert m\in S\} is defined as
\begin{align}
   & L_{\text{Batch}}(U', V'; \hat{\theta}(-D^*))-L_{\text{Batch}}(U', V'; \hat{\theta}) \\
\approx &  \nabla L_{\text{Batch}}(U{'},V{'};\hat{\theta})^{\mathrm{T}} (\hat{\theta}(-D^*)- \hat{\theta})  \notag \\
  = & - \nabla L_{\text{Batch}}(U{'},V{'};\hat{\theta})^{\mathrm{T}} \cdot \notag \\
    &\left( \text{posi-IF}(\mathcal{D}^*,\text{Seg}; \hat{\theta}) 
  +\text{nega-IF}(\mathcal{D}^*,\text{Seg}; \hat{\theta})\right) \label{eq:5} 
\end{align}
where $\hat{\theta}(-D^*)$ is the optimal model for the loss eliminating $D^*$, $  \text{posi-IF}(\mathcal{D}^*;\text{Seg};\hat{\theta})$ and $\text{nega-IF}(\mathcal{D}^*,\text{Seg};\hat{\theta})$ are obtained from Proposition \ref{pro:ECIF} for $D^*$. We define term (\ref{eq:5}) as the task-related influence score, denoted as  $\text{IS}(\mathcal{D}', \mathcal{D}^*, \text{Seg};\hat{\theta})$.
\end{property}
\begin{remark}
     Task-related influence score estimates the actual impact of a data subset on a specific task. The sign of this score indicates whether the evaluated set $\mathcal{D}^*$ has a positive or negative impact on the correct execution of the test task, while the absolute value of the score represents the magnitude of this impact. Therefore, the misalignment detection problem is sum up as $\argmax_{\mathcal{D}^*\subset \mathcal{D}} \text{IS}(\mathcal{D}', \mathcal{D}^*, \text{Seg};\hat{\theta})$. See  Appendix Algorithm \ref{alg:task-related} for details. 
\end{remark}

% Similarly, if we get a set of segments as G={Segi|i∈[t]}\mathcal{G} = \{\text{Seg}_i| i\in[t]\}, the segment-averaged version Influence Score is defined as
% \begin{equation*}
%     \text{IS}(\mathcal{D}', \mathcal{D}^*, \mathcal{G}) =  - \nabla L^I_{\text{Batch}}(U{'},V{'};\hat{\theta})^{\mathrm{T}} \cdot \left( \text{Pos}(\mathcal{D}^*,\text{Seg},\hat{\theta})+\text{Neg}(\mathcal{D}^*,\text{Seg},\hat{\theta})\right)
% \end{equation*}
% \end{proposition}
%If only use image data, effect is better, thus we can try to replace the LBatch(U′,V′;ˆθ)L_{\text{Batch}}(U{'},V{'};\hat{\theta}) by ∑Ni=1LI2T(vi,U;θ)\sum_{i=1}^N L_{I2T}(v_i, U; \theta), denoted as LI2T(U′,V′;ˆθ)L_{I2T}(U{'},V{'};\hat{\theta}).

% \subsection{Task-focused Influence Metric}
\subsection{Misprediction Trace Back}
From the transparency perspective, if the model makes prediction errors on certain tasks, the model trainers should be able to trace back to the samples associated with these erroneous predictions in the training set.
 
If we utilize the previous method for backtracking and choose the correct-labeled data that the model mispredicted to serve as the dataset $\mathcal{D}'$, then there is a significant possibility that the identified data are misaligned samples unrelated to the prediction errors. This is because, in the definition of task-relative IS, the term on the right side of the multiplication sign represents the change in model parameters. Even if certain samples are not related to the task we are tracing back, they may still have a high task-relative IS due to their substantial impact on the model parameters. Thus, compared to the above application, we need to constrain the change of model parameters. 

To address this, consider imposing a constraint $\delta$ on the permissible changes in model parameters when tracing back from mispredicted data while accounting for the process of upweighting the influence of samples as positive by $\epsilon$ and as negative by $\zeta$. Then we transform the trace back problem to identify which training example $x$ we should re-weight to most significantly impact the loss on the test sample set $\mathcal{D}'$ when given a small permissible change in model parameters.
%  are motivated by \citet{pmlr-v108-barshan20a} and\begin{equation}\label{relative:question}
%     \begin{split}
%       \arg\max_{x \in \mathcal{D}} &\max_{\epsilon, \zeta} \left|  \nabla L_{\text{Batch}}(U{'},V{'};\hat{\theta})^{\mathrm{T}} \cdot \left( \epsilon\cdot\text{positive-IF}(x;\hat{\theta})+\left(\zeta-1\right)\cdot\text{negative-IF}(x;\hat{\theta})\right) \right|\\
% &\text{s.t.} \quad \left\| \epsilon\cdot\text{positive-IF}(x)+\left(\zeta-1\right)\cdot\text{negative-IF}(x) \right\|^2 \leq \delta^2,  
%     \end{split}
%\end{equation}
  \begin{align}
     & \arg\max_{x \in \mathcal{D}, \epsilon, \zeta} |  L_{\text{Batch}}(U{'},V{'};\hat{\theta} + \Delta\hat{\theta}_{\epsilon, \zeta} (x)) \notag \\
     & -L_{\text{Batch}}(U{'},V{'};\hat{\theta}) |  \notag \\
     \quad\text{s.t.} & \left\| \Delta \hat{\theta}_{\epsilon, \zeta} (x)\right\|^2 \leq \delta^2\\
    \approx &   \arg\max_{x \in \mathcal{D},\epsilon, \zeta} |\nabla L_{\text{Batch}}(U{'},V{'};\hat{\theta})^{\mathrm{T}}\Delta\hat{\theta}_{\epsilon, \zeta} (x)| \notag \\
    \quad\text{s.t.} &\left\| \Delta \hat{\theta}_{\epsilon, \zeta} (x)\right\|^2 \leq \delta^2,   \label{relative:question}
\end{align}
where 
$\Delta\hat{\theta}_{\epsilon, \zeta}=\epsilon\cdot\text{posi-IF}(x;\hat{\theta})+\left(\zeta-1\right)\cdot\text{nega-IF}(x;\hat{\theta})$ 
is the model parameter change estimated by ECIF when the influence of sample $x=(x^T, x^I)$ is upweighted by $\epsilon$ and $\zeta$.
\begin{proposition}\label{prop.5.3}
Define $I=[\text{posi-IF}(x;\hat{\theta}),\text{nega-IF}(x;\hat{\theta})]$. If the $2\times 2$ matrix $I^{\mathrm{T}}\cdot I$ is irreversible, then equation \eqref{relative:question} is equivalent to
\begin{align*}
&\arg\max_{x \in \mathcal{D}}  {\|\text{nega-IF}(x;\hat{\theta}) \|}_2^{-1} \\
&\left| \nabla L_{\text{Batch}}(U{'},V{'};\hat{\theta})^{\mathrm{T}}\cdot  \text{nega-IF}(x;\hat{\theta}) \right|.
\end{align*}
Else,  $I^{\mathrm{T}}\cdot I$ is reversible, then (\ref{relative:question}) is equivalent to
\begin{align*}
  & \arg\max_{x \in \mathcal{D}} {\|\nabla L_{\text{Batch}}(U{'},V{'};\hat{\theta})\|}_2^{-1} \\ &\left|\nabla L_{\text{Batch}}(U{'},V{'};\hat{\theta})^{\mathrm{T}}  I \left[I^{\mathrm{T}} I\right]^{-1}  
    I^{\mathrm{T}}\nabla L_{\text{Batch}}(U{'},V{'};\hat{\theta})\right|.
\end{align*}
\end{proposition}
The proposition above reduces the original argmax trace back problem to a simpler argmax problem. Consequently, we define the simplified argmax objective as a novel influence metric \text{\bf relative-IS}. 
This metric, by adding constraints on parameter perturbations, helps us more accurately identify task-relevant samples. See Appendix Algorithm \ref{alg:relative-IS} for details. 
\begin{table*}[htbp]
\centering
\vspace{-12pt}
\caption{Performance comparison of retraining and ECIF on different datasets.}
\resizebox{0.8\linewidth}{!}{
\begin{tabular}{llcccccc}
\toprule
\multirow{2}{*}{\textbf{Sample}} & \multirow{2}{*}{\textbf{Method}} & \multicolumn{2}{c}{\textbf{FGVCAircraft}} & \multicolumn{2}{c}{\textbf{Food101}} & \multicolumn{2}{c}{\textbf{Flowers102}} \\
\cline{3-8}
 & & \textbf{Accuracy(\%)} & \textbf{RT (second)} & \textbf{Accuracy(\%)} & \textbf{RT (second)} & \textbf{Accuracy(\%)} & \textbf{RT (second)} \\
\midrule
\multirow{2}{*}{Random} & Retrain & 23.07$\pm$0.29 & 391.40 & 84.93$\pm$0.17 & 291.80 & 68.16$\pm$0.22 & 331.80 \\
 & ECIF & \textbf{22.77$\pm$0.09} & \textbf{30.40} & \textbf{84.87$\pm$0.24} & \textbf{29.12} & \textbf{68.53$\pm$0.12} & \textbf{29.16} \\
\midrule
\multirow{2}{*}{Valuable} & Retrain & 22.93$\pm$0.33 & 311.20 & 84.80$\pm$0.16 & 317.60 & 68.23$\pm$0.33 & 328.60 \\
 & ECIF & \textbf{22.73$\pm$0.09} & \textbf{23.80} & \textbf{84.86$\pm$0.05} & \textbf{25.08} & \textbf{68.26$\pm$0.12} & \textbf{26.08} \\
\midrule
\multirow{2}{*}{Harmful} & Retrain & 23.50$\pm$0.11 & 448.00 & 84.83$\pm$0.05 & 291.80 & 68.00$\pm$0.16 & 321.80 \\
 & ECIF & \textbf{23.02$\pm$0.07} & \textbf{25.04} & \textbf{84.90$\pm$0.01} & \textbf{24.88} & \textbf{68.30$\pm$0.01} & \textbf{25.08} \\
\bottomrule
\end{tabular}}
 \label{tab:results}
 \vspace{-5pt}
\end{table*}
\section{Experiments}
In our experiments, we will apply our above methods to different tasks, including identifying influential data (harmful data and valuable data) for fine-tuning through the task-related influence score,  mispredictions trace-back, and detecting misaligned data. 
\subsection{Experimental Settings}
\paragraph{Datasets.} We employ three datasets for utility and efficiency evaluation and the misprediction trace-back: \textit{FGVC-Aircraft dataset}~\citep{maji2013finegrainedvisualclassificationaircraft}, \textit{Food101 dataset}~\citep{bossard2014food}, \textit{Flowers102 dataset}~\citep{4756141}. For the identifying influential data experiments, we include \textit{Describable Textures Dataset(DTD) dataset}~\citep{10.1167/14.9.12} except for the above ones. 
For misalignment detection tasks, we use \textit{Cifar-10dataset}~\citep{Krizhevsky2009Learning}, and \textit{Imagenette}, a smaller subset of 10 easily classified classes from \textit{Imagenet}~\citep{5206848}.

\paragraph{Algorithms.} The tasks described below directly implement the algorithms for the applications in the previous section. Algorithm \ref{alg:ECIF} functions as the foundational algorithm, offering methods to calculate ECIF and providing model editing based on ECIF. Algorithm \ref{alg:task-related} and \ref{alg:task-related-self} compute task-related IS in Property \ref{task-IS} to evaluate samples, indicating both the direction and intensity of their impact on the task. Meanwhile, Algorithm \ref{alg:relative-IS} is for relative-IS in Prop.~\ref{prop.5.3}, which aids in tracing back specific samples.

\paragraph{Ground Truth, Baselines and Evaluation Metric.}
We use retraining as the ground truth, where the CLIP model is fine-tuned from scratch after sample removal. Additionally, we adopt ECIF along with three other methods as baselines (see Appendix \ref{app:baseline}). \textit{ECIF}: This method is a direct implementation of Algorithm \ref{alg:ECIF}, utilizing positive and negative IF to modify the model for sample removal. 
We utilize two main evaluation metrics to assess our models: accuracy and runtime (RT)\footnote{Here, RT refers to the total runtime required for the updating process.}.
Accuracy evaluates the model's performance by measuring the proportion of correctly classified instances out of the total instances. Runtime, measured in minutes, assesses the time required for each method to update the model.

% % \noindent {\bf Evaluation Metric.}
% % We utilize two main evaluation metrics to assess our models: the Accuracy and runtime (RT). The \textit{Accuracy} evaluates the model's performance by measuring the proportion of correctly classified instances out of the total instances, while \textit{runtime}, measured in minutes, assesses the time required for each method to update the model.

\paragraph{Implementation Details.}
Our experiments utilized an Nvidia V100-32G GPU and $10$ CPU cores with $64$ GB memory. For all experiments, we employ the CLIP model `ViT-B/16' and LoRA few-shot learning. For utility evaluation, when testing our method on a random sample-removing task, $10\%$ samples are randomly removed. For valuable (harmful) samples, we remove $10\%$ of the valuable (harmful) data identified by ECIF. Each removal is repeated for $3$ times with different seeds. See Appendix \ref{app:exp_detail} for details about other tasks.

\subsection{Utility and Efficiency Evaluation}
We evaluate the utility and efficiency of ECIF for data evaluation, whose results are in Table \ref{tab:results}. The results underscore the superior performance of ECIF compared to classical retraining. Notably, ECIF retains computational efficiency without sacrificing accuracy. We can easily observe that with random data removal, ECIF achieves an accuracy nearly equivalent to retraining (84.87 compared to 84.93) while significantly reducing runtime from 291.80 minutes to 29.12 minutes on the Food101 dataset. A similar trend was observed in the Flowers102 dataset, where ECIF reduces runtime from 331.80 minutes for retraining to 29.16 minutes, along with a modest 0.37 point improvement in accuracy. These findings demonstrate the ability of ECIF to save approximately 40-50\% of the computational time required for retraining while maintaining comparable accuracy levels.

When valuable data identified by ECIF are removed, the accuracy of both the retrained model and ECIF's edited version closely align, and both are significantly lower than those observed with random removal. This suggests that ECIF is capable of not only accurately editing the model but also effectively identifying influential data. Similar results can also be observed in the context of harmful data removal. See Appendix \ref{sec:exp_multi} for the results on different numbers of removal samples. Besides, we extend ECBM to larger dataset and results are shown in Appendix \ref{app:exp_larger_ds}.

\begin{table*}[ht]
    \centering
     \resizebox{\linewidth}{!}{
    \begin{tabular}{c|c}
\includegraphics[width=\textwidth]{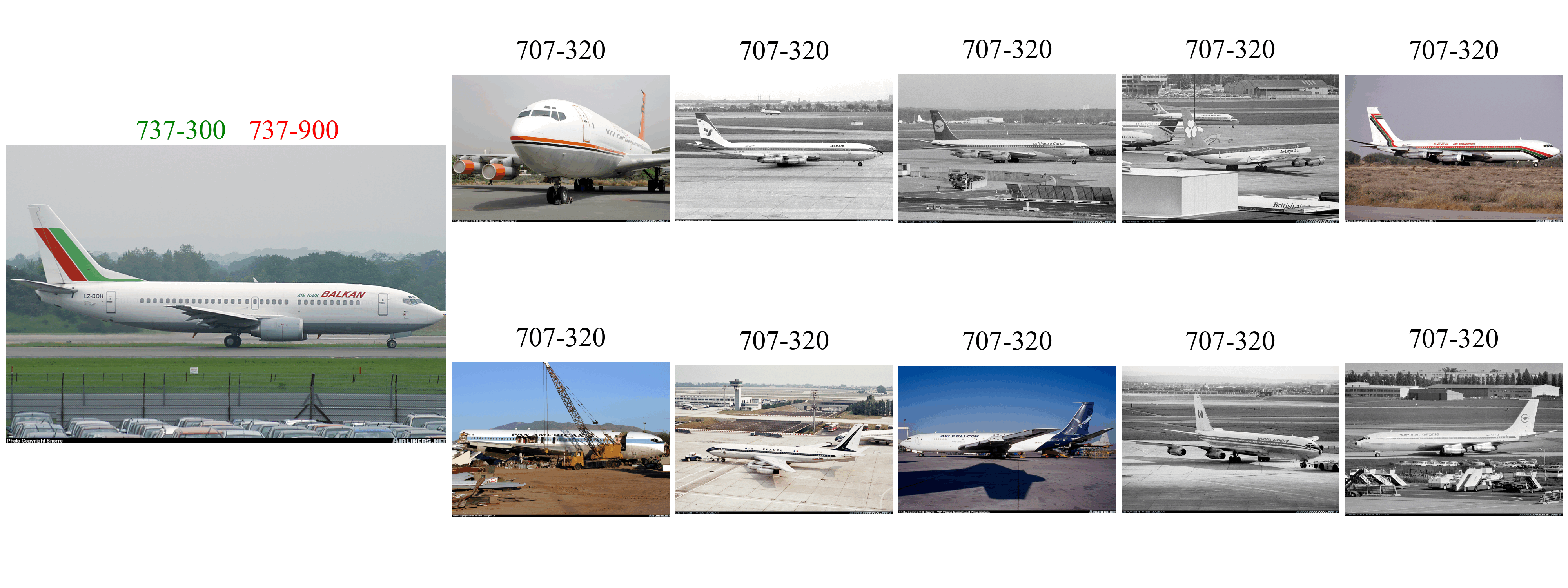}    & \includegraphics[width=\textwidth]{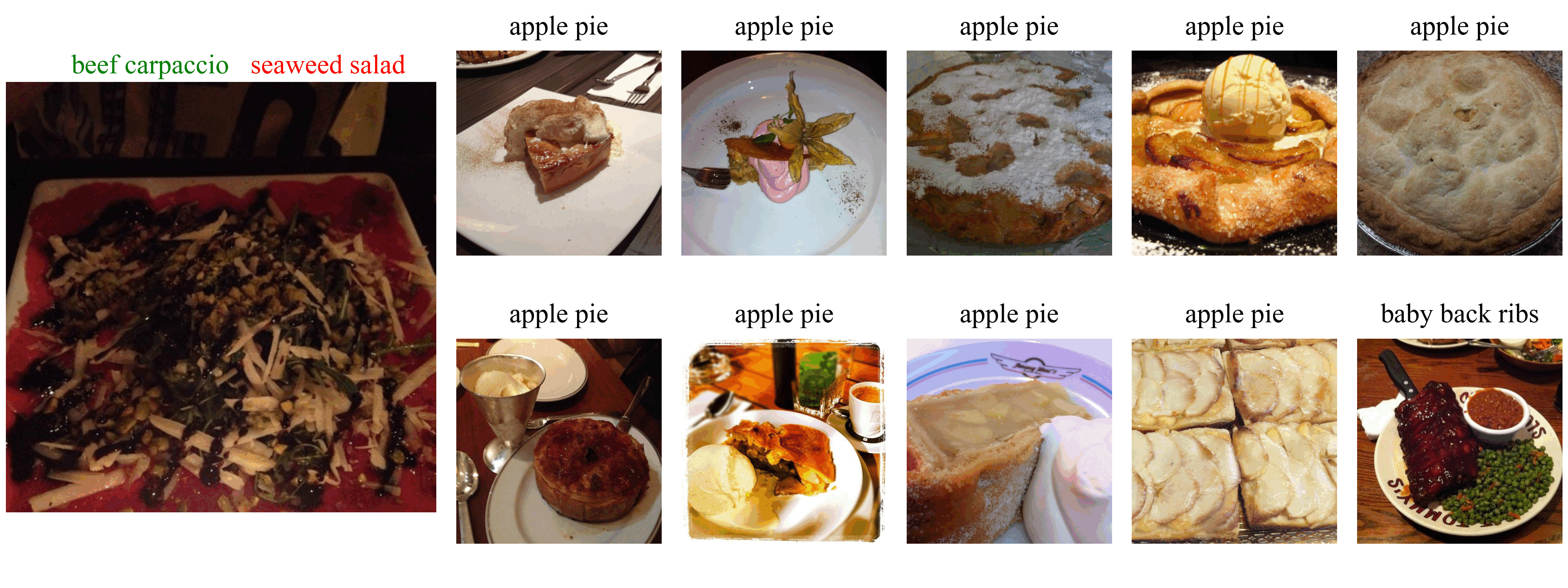}
    \end{tabular}}
    \vspace{-10pt}
    \caption{Top-10 related training data traced by mispredicted data.}
    \label{tab:visualizations}
\vspace{-10pt}
\end{table*}

% \begin{figure}
%     \centering
%     \includegraphics[width=\linewidth]{figs/trace/food101/mispredicted-food101-beef_carpaccio to seaweed_salad.png}
%     \vspace{-30pt}
%         \caption{Top-10 related training data traced by mispredicted data.}
%     \label{tab:visualizations}
%     \vspace{-20pt}
% \end{figure}

\subsection{Identifying Influential Data for Fine-tuning}
%  via task-related IS 

\noindent {\bf Task-related IS can identify the most valuable data.} To numerically assess the precision of data valuation algorithms, we employ the brittleness test~\citep{ilyas2022datamodelspredictingpredictionstraining}, which evaluates the algorithm's ability to accurately identify the most valuable data for a specific task. 
 Our evaluation process is as follows: utilizing the validation set within Algorithm \ref{alg:task-related}, we compute the task-related IS for each individual training data point. We then remove the top-$k$ valuable data points, with $k$ ranging from $5\%$ to $30\%$, retrain the model multiple times using different random seeds, and assess the resultant change in overall model accuracy. 

Results in Figure~\ref{fig:main_negative} reveal that removing valuable data identified by ECIF leads to a consistent decline in model accuracy, from $84.7$ to $84.1$. Conversely, random data removal triggers an increase in model accuracy once the removal proportion reaches $0.3$. This suggests Food101 contains substantial noise, and our algorithm can effectively identify data points that genuinely enhance the model's predictive accuracy.

Besides, we compare our ECBM method with baseline approaches for the data attribution task, and the results are presented in Appendix \ref{app:baseline}.

\begin{figure}[ht]
\vspace{-5pt}
\centering
    \subfigure[Harmful Data Removal]
    {
        \label{fig:main_positive}   
        \includegraphics[width=0.46\linewidth]{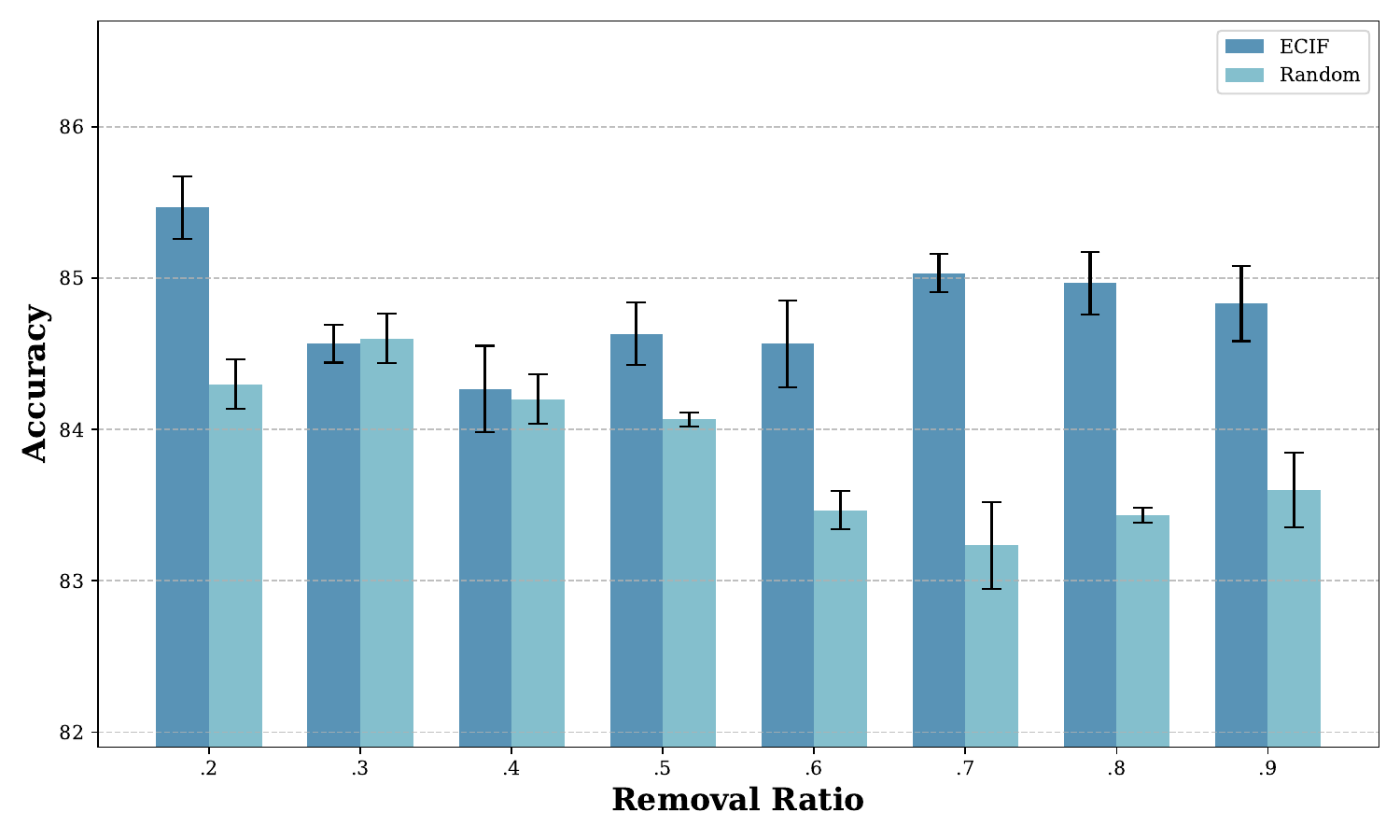}
    }
    \subfigure[Valuable Data Removal]
    {
        \label{fig:main_negative}
        \includegraphics[width=0.46\linewidth]{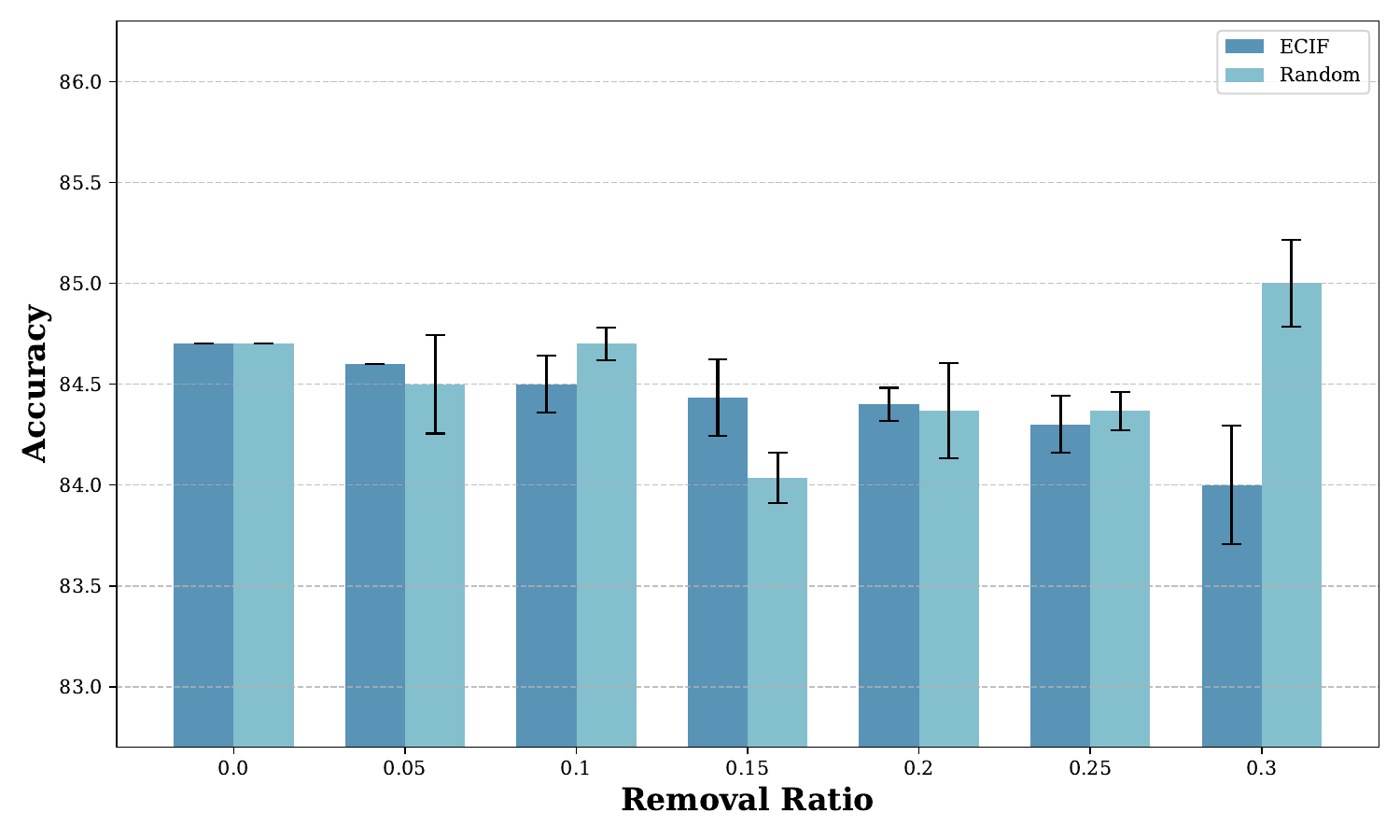}
    }
\vspace{-12pt}
\caption{Accuracy after removing influential data on Food101.}
\label{fig:perturbation}
\vspace{-5pt}
\end{figure}

% \input{image/misaligned_visual}
% \begin{figure*}[t]
% \centering
% \includegraphics[width=0.65\linewidth]{figs/misalign/n20/imagenett.png}
%     \caption{misalignment detection by visualizing the training samples that have the 10 highest IS scores on cifar-10 test set. 20\% of the training samples were mislabeled.}
%   \label{misalignment_vis}
% \end{figure*}
\begin{figure}[ht]
    \centering
    \includegraphics[width=\linewidth]{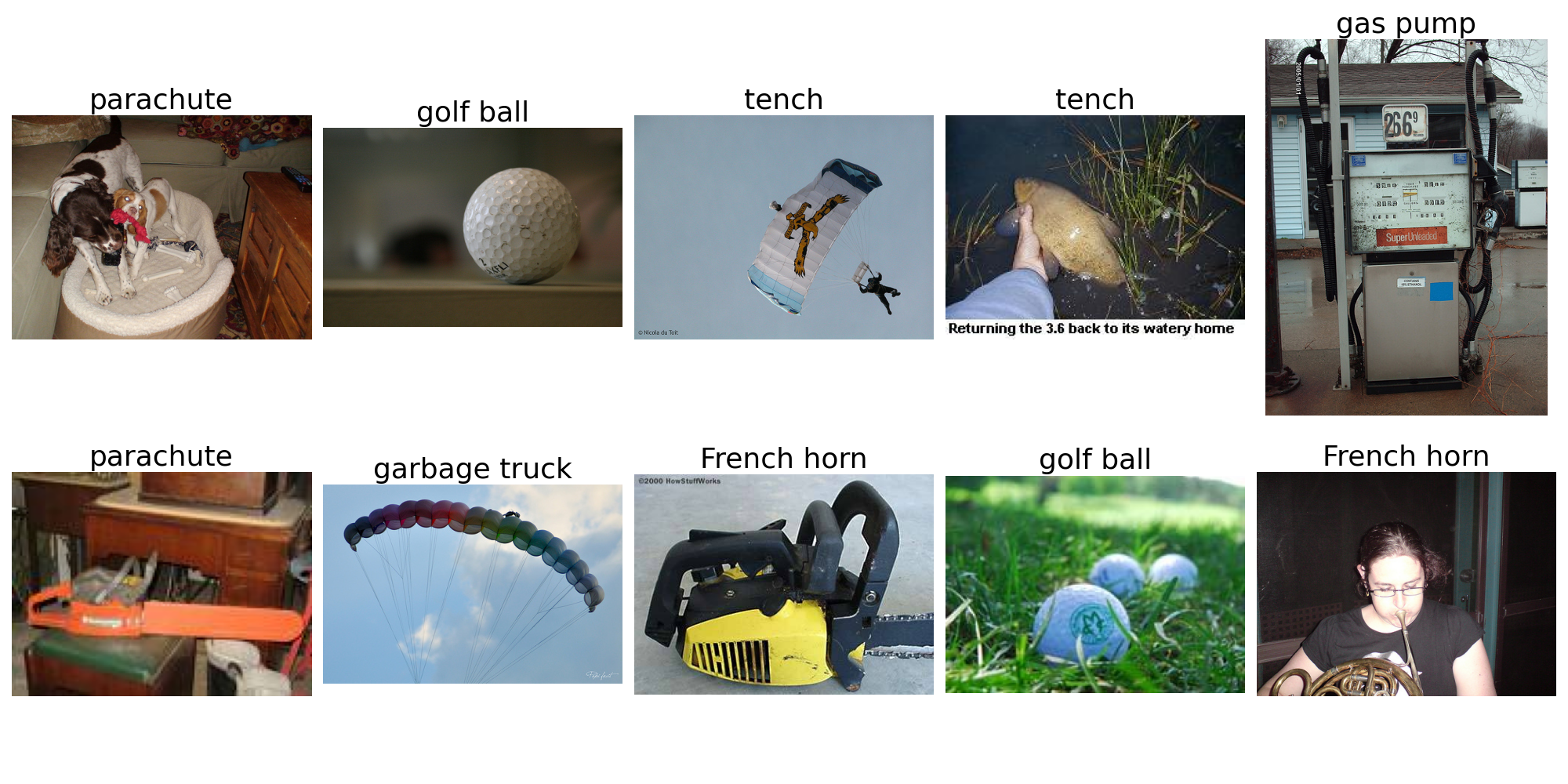}
    \vspace{-5pt}
    \caption{Top-10 misaligned pairs in the 20\% mislabeled dataset.}
    \label{misalignment_vis}
\vspace{-5pt}
\end{figure}

{\bf Task-related IS can identify harmful data.} The influence analysis from Algorithm \ref{alg:task-related} identifies data pairs with negative task-related IS as \textit{harmful} data for the task. To demonstrate the effectiveness of our algorithm in identifying detrimental data to specific tasks, we conducted experiments on several noisy datasets, such as Food101, and used the validation dataset in Algorithm \ref{alg:task-related}. 

We collected the harmful data identified by ECIF and then retrained the model multiple times with varying harmful data removal ratios and different random seeds. We compared its accuracy to that of a model retrained after randomly removing an equivalent number of data points. Results in Figure \ref{fig:main_positive} demonstrate the effectiveness of our approach in improving model performance by eliminating harmful data using task-related IS. Figure~\ref{fig:main_positive} indicates that with varying proportions of harmful data removal, the accuracy of the retrained model consistently fluctuates around its original level. When $10\%$ of harmful data is removed, accuracy increases by approximately $1\%$. Conversely, with random deletions, accuracy continues to decrease. This suggests that the accuracy improvement from removing harmful data with ECIF is not merely due to the removal action itself but rather because the removed data genuinely had a detrimental effect on model training. Additional results on other datasets are demonstrated in Appendix \ref{sec:taskIS_more}.

\subsection{Visualization of Misprediction Trace Back}
We apply Algorithm \ref{alg:relative-IS} to identify training data that are most relevant to specific mispredicted test samples. In this process, we select samples in the test data on which the model made a misclassification. Using the relative IS, we can identify the training data with the highest influence on the misprediction and visualize it. Table \ref{tab:visualizations} shows the results of this misprediction trace-back process (see Appendix \ref{sec:exp_add_pred} for additional results). Each pair of images compares a test sample with its most influential training counterpart. On the left, we show examples from the test set where the model produced incorrect predictions. On the right, the corresponding training data are shown, i.e., these data points hold the highest relative ISs in relation to the mispredicted test samples. This comparison helps shed light on how specific training samples may have contributed to the model's incorrect outputs. According to the visualization results, it can be observed that the samples traced back to the original task exhibit similarities in shape or texture with the original task.

% \begin{table}[htbp]
%     \centering
%     \begin{tabular}{c|c}
%      \includegraphics[width=0.45\textwidth]{figs/trace/707-320.png}    & \includegraphics[width=0.45\textwidth]{figs/trace/727-200.png} 
%      \\ \hline 
%       \includegraphics[width=0.45\textwidth]{figs/trace/baby_back_ribs.png}   & \includegraphics[width=0.45\textwidth]{figs/trace/beef_carpaccio.png} 
%     \end{tabular}
%     \caption{Top-10 related test data tracing of mispredicted data.}
%     \label{tab:visualizations}
% \end{table}

\subsection{Dataset Cleaning: Misalignment Data Detection}
We employed the relative IF to detect misaligned data pairs. Regarding the selection of the validation dataset, we experimented with two approaches: randomly selecting samples from the gold dataset (Algorithm \ref{alg:task-related}) and calculating based on the influence of the evaluated sample points (Algorithm \ref{alg:task-related-self}), in which the test loss is defined as the CLIP score~\citep{hessel2022clipscorereferencefreeevaluationmetric} of the evaluated data pair. 

We first mislabeled $10\%$-$30\%$ training samples and then identified the misaligned pairs by selecting those with the highest negative IS. These pairs are visualized in Table \ref{misalignment_vis} (see Appendix \ref{sec:exp_add_mis} for additional results). The visualization results reveal that the $8$ data points with the highest IS are entirely within the mislabeled data in our training set. This suggests that our algorithm has effectively identified the noise data artificially introduced into the dataset.

% ∑(xT,xI)∈D′−logs(u,v)\sum_{(x^T, x^I) \in \mathcal{D}'} -\log s(u,v), where uu is the embedding of xTx^T, and vv represents the embedding of the image xIx^I, and s(u,v)s(u,v) measures the cosine similarity between the embeddings. This alternative approach may help to improve the detection of misaligned pairs, particularly when the test dataset size is limited. Before detecting, we randomly mislabeled 20\% and 30\% of the cifar-10 training dataset samples. LoRA few-shot adaptation with clip model 'ViT-B/16' is used. The num-shots per class is set to 16. The number of iterations is set to 30. the learning rate is set to 2e−42e-4. the batch size is set 16. 

% \begin{figure}[ht]
%     \centering
% \includegraphics[width=0.5\linewidth]{figs/ablation/ablation.pdf}
% \caption{Impact of remove ratio on Food101 dataset. \warn{Can be different dataset or different algorithms} \label{fig:hyperpara}}
% \vspace{-7pt}
% \end{figure}

\section{Conclusion}
In this paper, we introduced the Extended Influence Function for Contrastive Loss (ECIF), a novel method for data valuation in contrastive learning. ECIF provides a dual-perspective analysis of data by considering both positive and negative samples, offering a comprehensive understanding of their impact on model performance. Utilizing a closed-form approximation, ECIF bypasses the need for retraining, making it both efficient and scalable for large models. 
Our approach is applicable to enhancing fine-tuning, tracing mispredicted data, and detecting misaligned data, with results demonstrating its effectiveness in real-world tasks.
\section{Impact Statement}
This research advances the methodologies for data attribution in contrastive learning. While we recognize the significance of assessing societal impacts, our analysis indicates that there are no immediate ethical risks that necessitate specific mitigation measures beyond standard practices in machine learning.

% \clearpage

\bibliography{icml2025}

\begin{thebibliography}{57}
\providecommand{\natexlab}[1]{#1}
\providecommand{\url}[1]{\texttt{#1}}
\expandafter\ifx\csname urlstyle\endcsname\relax
  \providecommand{\doi}[1]{doi: #1}\else
  \providecommand{\doi}{doi: \begingroup \urlstyle{rm}\Url}\fi

\bibitem[Agarwal et~al.(2017)Agarwal, Bullins, and Hazan]{agarwal2017second}
Agarwal, N., Bullins, B., and Hazan, E.
\newblock Second-order stochastic optimization for machine learning in linear time.
\newblock \emph{Journal of Machine Learning Research}, 18\penalty0 (116):\penalty0 1--40, 2017.

\bibitem[Basu et~al.(2021)Basu, Pope, and Feizi]{basu2021influence}
Basu, S., Pope, P., and Feizi, S.
\newblock Influence functions in deep learning are fragile.
\newblock In \emph{International Conference on Learning Representations (ICLR)}, 2021.

\bibitem[Bossard et~al.(2014)Bossard, Guillaumin, and Van~Gool]{bossard2014food}
Bossard, L., Guillaumin, M., and Van~Gool, L.
\newblock Food-101--mining discriminative components with random forests.
\newblock In \emph{Computer vision--ECCV 2014: 13th European conference, zurich, Switzerland, September 6-12, 2014, proceedings, part VI 13}, pp.\  446--461. Springer, 2014.

\bibitem[Brunet et~al.(2019)Brunet, Alkalay-Houlihan, Anderson, and Zemel]{brunet2019understanding}
Brunet, M.-E., Alkalay-Houlihan, C., Anderson, A., and Zemel, R.
\newblock Understanding the origins of bias in word embeddings.
\newblock In \emph{International conference on machine learning}, pp.\  803--811. PMLR, 2019.

\bibitem[Carlini \& Terzis(2021)Carlini and Terzis]{carlini2021poisoning}
Carlini, N. and Terzis, A.
\newblock Poisoning and backdooring contrastive learning.
\newblock \emph{arXiv preprint arXiv:2106.09667}, 2021.

\bibitem[Chen et~al.(2020{\natexlab{a}})Chen, Si, Li, Chelba, Kumar, Boning, and Hsieh]{chen2020multi}
Chen, H., Si, S., Li, Y., Chelba, C., Kumar, S., Boning, D., and Hsieh, C.-J.
\newblock Multi-stage influence function.
\newblock \emph{Advances in Neural Information Processing Systems}, 33:\penalty0 12732--12742, 2020{\natexlab{a}}.

\bibitem[Chen et~al.(2024)Chen, Yang, Xiong, Bai, Hu, Hao, Feng, Zhou, Wu, and Liu]{chen2024fast}
Chen, R., Yang, J., Xiong, H., Bai, J., Hu, T., Hao, J., Feng, Y., Zhou, J.~T., Wu, J., and Liu, Z.
\newblock Fast model debias with machine unlearning.
\newblock \emph{Advances in Neural Information Processing Systems}, 36, 2024.

\bibitem[Chen et~al.(2020{\natexlab{b}})Chen, Kornblith, Norouzi, and Hinton]{chen2020simpleframeworkcontrastivelearning}
Chen, T., Kornblith, S., Norouzi, M., and Hinton, G.
\newblock A simple framework for contrastive learning of visual representations, 2020{\natexlab{b}}.
\newblock URL \url{https://arxiv.org/abs/2002.05709}.

\bibitem[Chen et~al.(2020{\natexlab{c}})Chen, Kornblith, Norouzi, and Hinton]{pmlr-v119-chen20j}
Chen, T., Kornblith, S., Norouzi, M., and Hinton, G.
\newblock A simple framework for contrastive learning of visual representations.
\newblock In \emph{Proceedings of the 37th International Conference on Machine Learning}, 2020{\natexlab{c}}.

\bibitem[Choe et~al.(2024)Choe, Ahn, Bae, Zhao, Kang, Chung, Pratapa, Neiswanger, Strubell, Mitamura, et~al.]{choe2024your}
Choe, S.~K., Ahn, H., Bae, J., Zhao, K., Kang, M., Chung, Y., Pratapa, A., Neiswanger, W., Strubell, E., Mitamura, T., et~al.
\newblock What is your data worth to gpt? llm-scale data valuation with influence functions.
\newblock \emph{arXiv preprint arXiv:2405.13954}, 2024.

\bibitem[Cook(2000)]{cook2000detection}
Cook, R.~D.
\newblock Detection of influential observation in linear regression.
\newblock \emph{Technometrics}, 42\penalty0 (1):\penalty0 65--68, 2000.

\bibitem[Cook \& Weisberg(1980)Cook and Weisberg]{cook1980characterizations}
Cook, R.~D. and Weisberg, S.
\newblock Characterizations of an empirical influence function for detecting influential cases in regression.
\newblock \emph{Technometrics}, 22\penalty0 (4):\penalty0 495--508, 1980.

\bibitem[Deng et~al.(2009)Deng, Dong, Socher, Li, Li, and Fei-Fei]{5206848}
Deng, J., Dong, W., Socher, R., Li, L.-J., Li, K., and Fei-Fei, L.
\newblock Imagenet: A large-scale hierarchical image database.
\newblock In \emph{2009 IEEE Conference on Computer Vision and Pattern Recognition}, pp.\  248--255, 2009.

\bibitem[Donahue \& Simonyan(2019)Donahue and Simonyan]{donahue2019large}
Donahue, J. and Simonyan, K.
\newblock Large scale adversarial representation learning.
\newblock \emph{Advances in neural information processing systems}, 32, 2019.

\bibitem[Gao et~al.(2022)Gao, Liu, Xu, Zhang, Li, Ji, and Shen]{gao2022pyramidclip}
Gao, Y., Liu, J., Xu, Z., Zhang, J., Li, K., Ji, R., and Shen, C.
\newblock Pyramidclip: Hierarchical feature alignment for vision-language model pretraining.
\newblock \emph{Advances in neural information processing systems}, 35:\penalty0 35959--35970, 2022.

\bibitem[Ghorbani \& Zou(2019)Ghorbani and Zou]{ghorbani2019data}
Ghorbani, A. and Zou, J.
\newblock Data shapley: Equitable valuation of data for machine learning.
\newblock In \emph{International conference on machine learning}, pp.\  2242--2251. PMLR, 2019.

\bibitem[Golatkar et~al.(2020)Golatkar, Achille, and Soatto]{golatkar2020eternal}
Golatkar, A., Achille, A., and Soatto, S.
\newblock Eternal sunshine of the spotless net: Selective forgetting in deep networks.
\newblock In \emph{Proceedings of the IEEE/CVF Conference on Computer Vision and Pattern Recognition}, pp.\  9304--9312, 2020.

\bibitem[Golatkar et~al.(2021)Golatkar, Achille, Ravichandran, Polito, and Soatto]{golatkar2021mixed}
Golatkar, A., Achille, A., Ravichandran, A., Polito, M., and Soatto, S.
\newblock Mixed-privacy forgetting in deep networks.
\newblock In \emph{Proceedings of the IEEE/CVF conference on computer vision and pattern recognition}, pp.\  792--801, 2021.

\bibitem[Grosse et~al.(2023{\natexlab{a}})Grosse, Bae, Anil, Elhage, Tamkin, Tajdini, Steiner, Li, Durmus, Perez, Hubinger, Lukošiūtė, Nguyen, Joseph, McCandlish, Kaplan, and Bowman]{grosse2023studyinglargelanguagemodel}
Grosse, R., Bae, J., Anil, C., Elhage, N., Tamkin, A., Tajdini, A., Steiner, B., Li, D., Durmus, E., Perez, E., Hubinger, E., Lukošiūtė, K., Nguyen, K., Joseph, N., McCandlish, S., Kaplan, J., and Bowman, S.~R.
\newblock Studying large language model generalization with influence functions, 2023{\natexlab{a}}.
\newblock URL \url{https://arxiv.org/abs/2308.03296}.

\bibitem[Grosse et~al.(2023{\natexlab{b}})Grosse, Bae, Anil, Elhage, Tamkin, Tajdini, Steiner, Li, Durmus, Perez, et~al.]{grosse2023studying}
Grosse, R., Bae, J., Anil, C., Elhage, N., Tamkin, A., Tajdini, A., Steiner, B., Li, D., Durmus, E., Perez, E., et~al.
\newblock Studying large language model generalization with influence functions.
\newblock \emph{arXiv preprint arXiv:2308.03296}, 2023{\natexlab{b}}.

\bibitem[Guo et~al.(2021)Guo, Rajani, Hase, Bansal, and Xiong]{guo2021fastif}
Guo, H., Rajani, N., Hase, P., Bansal, M., and Xiong, C.
\newblock Fastif: Scalable influence functions for efficient model interpretation and debugging.
\newblock In \emph{Proceedings of the 2021 Conference on Empirical Methods in Natural Language Processing}, pp.\  10333--10350, 2021.

\bibitem[Gutmann \& Hyv{\"a}rinen(2010)Gutmann and Hyv{\"a}rinen]{gutmann2010noise}
Gutmann, M. and Hyv{\"a}rinen, A.
\newblock Noise-contrastive estimation: A new estimation principle for unnormalized statistical models.
\newblock In \emph{Proceedings of the thirteenth international conference on artificial intelligence and statistics}, pp.\  297--304. JMLR Workshop and Conference Proceedings, 2010.

\bibitem[Han et~al.(2020)Han, Wallace, and Tsvetkov]{han2020explaining}
Han, X., Wallace, B.~C., and Tsvetkov, Y.
\newblock Explaining black box predictions and unveiling data artifacts through influence functions.
\newblock \emph{arXiv preprint arXiv:2005.06676}, 2020.

\bibitem[He et~al.(2020)He, Fan, Wu, Xie, and Girshick]{he2020momentum}
He, K., Fan, H., Wu, Y., Xie, S., and Girshick, R.
\newblock Momentum contrast for unsupervised visual representation learning.
\newblock In \emph{Proceedings of the IEEE/CVF conference on computer vision and pattern recognition}, pp.\  9729--9738, 2020.

\bibitem[Hessel et~al.(2021)Hessel, Holtzman, Forbes, Bras, and Choi]{hessel2021clipscore}
Hessel, J., Holtzman, A., Forbes, M., Bras, R.~L., and Choi, Y.
\newblock Clipscore: A reference-free evaluation metric for image captioning.
\newblock \emph{arXiv preprint arXiv:2104.08718}, 2021.

\bibitem[Hessel et~al.(2022)Hessel, Holtzman, Forbes, Bras, and Choi]{hessel2022clipscorereferencefreeevaluationmetric}
Hessel, J., Holtzman, A., Forbes, M., Bras, R.~L., and Choi, Y.
\newblock Clipscore: A reference-free evaluation metric for image captioning, 2022.

\bibitem[Huber(1981)]{huber1981robust}
Huber, P.~J.
\newblock Robust statistics.
\newblock \emph{Wiley Series in Probability and Mathematical Statistics}, 1981.

\bibitem[Ilyas et~al.(2022)Ilyas, Park, Engstrom, Leclerc, and Madry]{ilyas2022datamodelspredictingpredictionstraining}
Ilyas, A., Park, S.~M., Engstrom, L., Leclerc, G., and Madry, A.
\newblock Datamodels: Predicting predictions from training data, 2022.

\bibitem[Jia et~al.(2019)Jia, Dao, Wang, Hubis, Hynes, G{\"u}rel, Li, Zhang, Song, and Spanos]{jia2019towards}
Jia, R., Dao, D., Wang, B., Hubis, F.~A., Hynes, N., G{\"u}rel, N.~M., Li, B., Zhang, C., Song, D., and Spanos, C.~J.
\newblock Towards efficient data valuation based on the shapley value.
\newblock In \emph{The 22nd International Conference on Artificial Intelligence and Statistics}, pp.\  1167--1176. PMLR, 2019.

\bibitem[Jiang et~al.(2024)Jiang, Xu, Dong, Chen, Ye, Yan, Ye, Zhang, Huang, and Zhang]{jiang2024hallucination}
Jiang, C., Xu, H., Dong, M., Chen, J., Ye, W., Yan, M., Ye, Q., Zhang, J., Huang, F., and Zhang, S.
\newblock Hallucination augmented contrastive learning for multimodal large language model.
\newblock In \emph{Proceedings of the IEEE/CVF Conference on Computer Vision and Pattern Recognition}, pp.\  27036--27046, 2024.

\bibitem[Kim et~al.(2023)Kim, Jo, Kim, and Kim]{Kim_2023_ICCV}
Kim, B., Jo, Y., Kim, J., and Kim, S.
\newblock Misalign, contrast then distill: Rethinking misalignments in language-image pre-training.
\newblock In \emph{Proceedings of the IEEE/CVF International Conference on Computer Vision (ICCV)}, pp.\  2563--2572, October 2023.

\bibitem[Koh et~al.(2024)Koh, Fried, and Salakhutdinov]{koh2024generating}
Koh, J.~Y., Fried, D., and Salakhutdinov, R.~R.
\newblock Generating images with multimodal language models.
\newblock \emph{Advances in Neural Information Processing Systems}, 36, 2024.

\bibitem[Koh \& Liang(2017)Koh and Liang]{koh2017understanding}
Koh, P.~W. and Liang, P.
\newblock Understanding black-box predictions via influence functions.
\newblock In \emph{International conference on machine learning}, pp.\  1885--1894. PMLR, 2017.

\bibitem[Krizhevsky(2009)]{krizhevsky2009learningML}
Krizhevsky, A.
\newblock Learning multiple layers of features from tiny images, 2009.
\newblock URL \url{https://api.semanticscholar.org/CorpusID:18268744}.

\bibitem[Krizhevsky et~al.(2009)Krizhevsky, Hinton, et~al.]{Krizhevsky2009Learning}
Krizhevsky, A., Hinton, G., et~al.
\newblock Learning multiple layers of features from tiny images, 2009.

\bibitem[Kwon \& Zou(2022)Kwon and Zou]{pmlr-v151-kwon22a}
Kwon, Y. and Zou, J.
\newblock Beta shapley: a unified and noise-reduced data valuation framework for machine learning.
\newblock In \emph{Proceedings of The 25th International Conference on Artificial Intelligence and Statistics}, 2022.

\bibitem[Kwon et~al.(2023)Kwon, Wu, Wu, and Zou]{kwon2023datainf}
Kwon, Y., Wu, E., Wu, K., and Zou, J.
\newblock Datainf: Efficiently estimating data influence in lora-tuned llms and diffusion models.
\newblock In \emph{The Twelfth International Conference on Learning Representations}, 2023.

\bibitem[Li et~al.(2023{\natexlab{a}})Li, Wang, and Xie]{NEURIPS2023_996e2b44}
Li, X., Wang, Z., and Xie, C.
\newblock An inverse scaling law for clip training.
\newblock In Oh, A., Naumann, T., Globerson, A., Saenko, K., Hardt, M., and Levine, S. (eds.), \emph{Advances in Neural Information Processing Systems}, volume~36, pp.\  49068--49087. Curran Associates, Inc., 2023{\natexlab{a}}.

\bibitem[Li et~al.(2023{\natexlab{b}})Li, Fan, Hu, Feichtenhofer, and He]{Li_2023_CVPR}
Li, Y., Fan, H., Hu, R., Feichtenhofer, C., and He, K.
\newblock Scaling language-image pre-training via masking.
\newblock In \emph{Proceedings of the IEEE/CVF Conference on Computer Vision and Pattern Recognition (CVPR)}, pp.\  23390--23400, June 2023{\natexlab{b}}.

\bibitem[Maji et~al.(2013)Maji, Rahtu, Kannala, Blaschko, and Vedaldi]{maji2013finegrainedvisualclassificationaircraft}
Maji, S., Rahtu, E., Kannala, J., Blaschko, M., and Vedaldi, A.
\newblock Fine-grained visual classification of aircraft, 2013.
\newblock URL \url{https://arxiv.org/abs/1306.5151}.

\bibitem[Nguyen et~al.(2022)Nguyen, Ilharco, Wortsman, Oh, and Schmidt]{NEURIPS2022}
Nguyen, T., Ilharco, G., Wortsman, M., Oh, S., and Schmidt, L.
\newblock Quality not quantity: On the interaction between dataset design and robustness of clip.
\newblock In Koyejo, S., Mohamed, S., Agarwal, A., Belgrave, D., Cho, K., and Oh, A. (eds.), \emph{Advances in Neural Information Processing Systems}, volume~35, pp.\  21455--21469. Curran Associates, Inc., 2022.

\bibitem[Nilsback \& Zisserman(2008)Nilsback and Zisserman]{4756141}
Nilsback, M.-E. and Zisserman, A.
\newblock Automated flower classification over a large number of classes.
\newblock In \emph{2008 Sixth Indian Conference on Computer Vision, Graphics, Image Processing}, pp.\  722--729, 2008.
\newblock \doi{10.1109/ICVGIP.2008.47}.

\bibitem[Park et~al.(2023)Park, Georgiev, Ilyas, Leclerc, and Madry]{park2023trakattributingmodelbehavior}
Park, S.~M., Georgiev, K., Ilyas, A., Leclerc, G., and Madry, A.
\newblock Trak: Attributing model behavior at scale, 2023.
\newblock URL \url{https://arxiv.org/abs/2303.14186}.

\bibitem[Pruthi et~al.(2020)Pruthi, Liu, Kale, and Sundararajan]{NEURIPS2020_e6385d39}
Pruthi, G., Liu, F., Kale, S., and Sundararajan, M.
\newblock Estimating training data influence by tracing gradient descent.
\newblock In \emph{Advances in Neural Information Processing Systems}, volume~33, pp.\  19920--19930, 2020.

\bibitem[Radford et~al.(2021)Radford, Kim, Hallacy, Ramesh, Goh, Agarwal, Sastry, Askell, Mishkin, Clark, et~al.]{radford2021learning}
Radford, A., Kim, J.~W., Hallacy, C., Ramesh, A., Goh, G., Agarwal, S., Sastry, G., Askell, A., Mishkin, P., Clark, J., et~al.
\newblock Learning transferable visual models from natural language supervision.
\newblock In \emph{International conference on machine learning}, pp.\  8748--8763. PMLR, 2021.

\bibitem[Robinson et~al.(2021)Robinson, Chuang, Sra, and Jegelka]{robinson2021contrastive}
Robinson, J.~D., Chuang, C.-Y., Sra, S., and Jegelka, S.
\newblock Contrastive learning with hard negative samples.
\newblock In \emph{International Conference on Learning Representations}, 2021.

\bibitem[Sharan et~al.(2014)Sharan, Rosenholtz, and Adelson]{10.1167/14.9.12}
Sharan, L., Rosenholtz, R., and Adelson, E.~H.
\newblock {Accuracy and speed of material categorization in real-world images}.
\newblock \emph{Journal of Vision}, 14\penalty0 (9):\penalty0 12--12, 2014.
\newblock ISSN 1534-7362.
\newblock \doi{10.1167/14.9.12}.

\bibitem[Song et~al.(2019)Song, Kim, and Lee]{song2019selfie}
Song, H., Kim, M., and Lee, J.-G.
\newblock Selfie: Refurbishing unclean samples for robust deep learning.
\newblock In \emph{International conference on machine learning}, pp.\  5907--5915. PMLR, 2019.

\bibitem[van~den Oord et~al.(2019)van~den Oord, Li, and Vinyals]{oord2019}
van~den Oord, A., Li, Y., and Vinyals, O.
\newblock Representation learning with contrastive predictive coding, 2019.

\bibitem[Wang et~al.(2019)Wang, Ustun, and Calmon]{wang2019repairing}
Wang, H., Ustun, B., and Calmon, F.
\newblock Repairing without retraining: Avoiding disparate impact with counterfactual distributions.
\newblock In \emph{International Conference on Machine Learning}, pp.\  6618--6627. PMLR, 2019.

\bibitem[Wang \& Isola(2020)Wang and Isola]{wang2020understanding}
Wang, T. and Isola, P.
\newblock Understanding contrastive representation learning through alignment and uniformity on the hypersphere.
\newblock In \emph{International conference on machine learning}, pp.\  9929--9939. PMLR, 2020.

\bibitem[Warnecke et~al.(2023)Warnecke, Pirch, Wressnegger, and Rieck]{warnecke2021machine}
Warnecke, A., Pirch, L., Wressnegger, C., and Rieck, K.
\newblock Machine unlearning of features and labels.
\newblock \emph{Network and Distributed System Security (NDSS) Symposium}, 2023.

\bibitem[Wu et~al.(2023)Wu, Yang, Qian, Sui, Wang, and He]{wu2023gif}
Wu, J., Yang, Y., Qian, Y., Sui, Y., Wang, X., and He, X.
\newblock Gif: A general graph unlearning strategy via influence function.
\newblock In \emph{Proceedings of the ACM Web Conference 2023}, pp.\  651--661, 2023.

\bibitem[Yao et~al.(2021)Yao, Huang, Hou, Lu, Niu, Xu, Liang, Li, Jiang, and Xu]{yao2021filip}
Yao, L., Huang, R., Hou, L., Lu, G., Niu, M., Xu, H., Liang, X., Li, Z., Jiang, X., and Xu, C.
\newblock Filip: Fine-grained interactive language-image pre-training.
\newblock \emph{arXiv preprint arXiv:2111.07783}, 2021.

\bibitem[Yin et~al.(2023)Yin, Fu, Zhao, Li, Sun, Xu, and Chen]{yin2023survey}
Yin, S., Fu, C., Zhao, S., Li, K., Sun, X., Xu, T., and Chen, E.
\newblock A survey on multimodal large language models.
\newblock \emph{arXiv preprint arXiv:2306.13549}, 2023.

\bibitem[Yoon et~al.(2020)Yoon, Arik, and Pfister]{yoon2020data}
Yoon, J., Arik, S., and Pfister, T.
\newblock Data valuation using reinforcement learning.
\newblock In \emph{International Conference on Machine Learning}, pp.\  10842--10851. PMLR, 2020.

\bibitem[Yuksekgonul et~al.(2023)Yuksekgonul, Bianchi, Kalluri, Jurafsky, and Zou]{yuksekgonul2023when}
Yuksekgonul, M., Bianchi, F., Kalluri, P., Jurafsky, D., and Zou, J.
\newblock When and why vision-language models behave like bags-of-words, and what to do about it?
\newblock In \emph{The Eleventh International Conference on Learning Representations}, 2023.

\end{thebibliography}
\bibliographystyle{icml2025}

%%%%%%%%%%%%%%%%%%%%%%%%%%%%%%%%%%%%%%%%%%%%%%%%%%%%%%%%%%%%%%%%%%%%%%%%%%%%%%%
%%%%%%%%%%%%%%%%%%%%%%%%%%%%%%%%%%%%%%%%%%%%%%%%%%%%%%%%%%%%%%%%%%%%%%%%%%%%%%%
% APPENDIX
%%%%%%%%%%%%%%%%%%%%%%%%%%%%%%%%%%%%%%%%%%%%%%%%%%%%%%%%%%%%%%%%%%%%%%%%%%%%%%%
%%%%%%%%%%%%%%%%%%%%%%%%%%%%%%%%%%%%%%%%%%%%%%%%%%%%%%%%%%%%%%%%%%%%%%%%%%%%%%%
\newpage
\appendix
\onecolumn

\section{Algorithm}

\begin{algorithm}
\caption{ECIF\label{alg:ECIF}}
\begin{algorithmic}[1]
    \STATE {\bf Input:} Training Dataset $\mathcal{D} = \{ (x^T, x^I) \}$, dataset $\mathcal{D}^*$ to be evaluated, the parameters $\hat{\theta}$ which is involved in IF calculation in the model and the regularization parameter $\delta$.
    \STATE Define $S(\cdot, \cdot)$ as the similarity score.
    \STATE Compute the text embedding and image embedding for $\mathcal{D}^*$ as $u$ and $v$.
    \STATE Random divide the training dataset $\mathcal{D}$ into MM batches and obtain the position index of $\mathcal{D}^*$ as $\text{Seg} \triangleq\{(m, E_m)|m\in S\}$.
    \STATE Compute the influence term as positive and negative samples for the $m$-th batch in $S$ by: 
\begin{align*}
&\text{Pos}_m = \sum_{n\in E_m} \left(-\log \frac{e^{S(u_n, v_n)}}{\sum_{j=1}^N e^{S(u_n, v_j)}}-\log \frac{e^{S(v_n, u_n)}}{\sum_{j=1}^N e^{S(v_n, u_j)}}\right)\\
&\text{Neg}_m = \sum_{i\in[N]/E_m}\left(   \frac{\sum_{j=1}^N e^{S(u_i, v_j)}}{\sum_{n\in E_m}e^{S(u_i, v_n)}} + \frac{\sum_{j=1}^N e^{S(u_i, v_j)}}{\sum_{n\in E_m}e^{S(v_i, u_n)}} \right)
\end{align*}
\STATE Compute the sum of the gradient of $\text{Pos}_m$ and $\text{Neg}_m$ as
$$\Tilde{\text{Pos}} = \sum_{m\in S}\nabla_{\theta}\text{Pos}_m, \text{ and }\Tilde{\text{Neg}} = \sum_{m\in S}\nabla_{\theta} \text{Neg}_m.$$
    \STATE Compute the batch embedding for $\mathcal{D}$ as $\{B_m, m\in[M]\}$.
\STATE Compute the inverse Hessian matrix of the loss function with respect to $\hat{\theta}$ as 
$$G = \left[\sum_{m\in [M]} \nabla^2_{\theta}L_{\text{Batch}}(B_m;\hat{\theta})+ \delta \cdot I\right]^{-1}$$
\STATE Compute the $\text{positive-IF}(\mathcal{D}^*, \text{Seg})$ and $\text{negative-IF}(\mathcal{D}^*, \text{Seg})$ as:
\begin{equation*}
    \text{positive-IF}(\mathcal{D}^*, \text{Seg}) = - G\cdot \Tilde{\text{Pos}}, \quad\text{negative-IF}(\mathcal{D}^*, \text{Seg}) = - G\cdot \Tilde{\text{Neg}}
\end{equation*}
%\warn{dimension of |θ|∗1|\theta|*1}
%IFIN(D∗,Seg)\text{IF}_{\text{Pos}}(\mathcal{D}^*, \text{Seg}) and IFEX(D∗,Seg)\text{negative-IF}(\mathcal{D}^*, \text{Seg})
\STATE Obtain the ECIF as
\begin{align}
    \text{ECIF}(\mathcal{D}^*, \mathcal{D}) = & \left({\text{positive-IF}}, {\text{negative-IF}}\right)
\end{align}
\STATE Edit model parameter to unlearn dataset $\mathcal{D}^*$ by
$$\tilde{\theta} = \hat{\theta} -  {\text{positive-IF}} - {\text{negative-IF}}$$
%\warn{dimension of |θ|∗1|\theta|*1}
\STATE {\bf Return: } $\text{ECIF}(\mathcal{D}^*, \mathcal{D})$, Edited parameter $\Tilde{\theta}$.
\end{algorithmic}
\end{algorithm}

\begin{algorithm}
\caption{Task-related Influence Score Based on ECIF\label{alg:task-related}}
\begin{algorithmic}[1]
    \STATE {\bf Input:} Training Dataset $\mathcal{D} = \{ (x^T, x^I) \}$, dataset $\mathcal{D}^*$ to be evaluated,  test dataset $\mathcal{D}'$, the parameters $\hat{\theta}$ which is involved in IF calculation in the model.
    %\STATE {\bf Randomly select several data from the validation dataset to form a test dataset D′\mathcal{D}'},
\STATE Compute the $\text{ECIF}(\mathcal{D}^*, \mathcal{D})= \left(\bar{\text{positive-IF}}, \bar{\text{negative-IF}}\right)$ by algorithm \ref{alg:ECIF}.
\STATE Compute the gradient of the batch loss function of the test data as 
$$C = \sum_{(U,V)\in \mathcal{D}'}\nabla_{\theta} L_{\text{Batch}}(U, V;\hat{\theta})$$
\STATE Compute the task-related influence score as 
\begin{equation*}
    IS = C^{\mathrm{T}}\cdot{\text{positive-IF}} + C^{\mathrm{T}}\cdot{\text{negative-IF}}
\end{equation*}
    \STATE {\bf Return: } Task-related Influence Score $IS$.
\end{algorithmic}
\end{algorithm}

\begin{algorithm}
\caption{Self Influence Score Based on ECIF\label{alg:task-related-self}}
\begin{algorithmic}[1]
    \STATE {\bf Input:} Training Dataset $\mathcal{D} = \{ (x^T, x^I) \}$, dataset $\mathcal{D}^*$ to be evaluated,  test dataset $\mathcal{D}'$, the parameters $\hat{\theta}$ which is involved in IF calculation in the model.
    %\STATE {\bf Randomly select several data from the validation dataset to form a test dataset D′\mathcal{D}'},
\STATE Compute the $\text{ECIF}(\mathcal{D}^*, \mathcal{D})= \left(\bar{\text{positive-IF}}, \bar{\text{negative-IF}}\right)$ by algorithm \ref{alg:ECIF}.
\STATE Compute the gradient of the batch loss function of the test data as 
$$C = \frac{1}{|\mathcal{D}'|}\sum_{(x^T,x^I)\in \mathcal{D}'}\nabla_{\theta} -\log \frac{u^{\mathrm{T}}\cdot v}{\|u\|\cdot\|v\|},$$
where $u$ and $v$ is the embedding for $x^T$ and $x^I$
\STATE Compute the task-related influence score as 
\begin{equation*}
    IS = C^{\mathrm{T}}\cdot{\text{positive-IF}} + C^{\mathrm{T}}\cdot{\text{negative-IF}}
\end{equation*}
    \STATE {\bf Return: } Task-related Influence Score $IS$.
\end{algorithmic}
\end{algorithm}

\begin{algorithm}
\caption{Relative Influence Score Based on ECIF\label{alg:relative-IS}}
\begin{algorithmic}[1]
    \STATE {\bf Input:} Training Dataset $\mathcal{D} = \{ (x^T, x^I) \}$, dataset $\mathcal{D}^*$ to be evaluated,  test dataset $\mathcal{D}'$.
    \STATE Compute the $\text{ECIF}(\mathcal{D}^*, \mathcal{D}) = \left({\text{positive-IF}}, {\text{negative-IF}}\right)$ by algorithm \ref{alg:ECIF}.
\STATE Compute the gradient of the batch loss function of the test data as 
$$C = \sum_{(U,V)\in \mathcal{D}'}\nabla_{\theta} L_{\text{Batch}}(U, V;\hat{\theta})$$
%\warn{LBatchL_{\text{Batch}} is enough,}
\IF{${\text{positive-IF}}$ is parallel to ${\text{negative-IF}}$}
\STATE Compute the relative-IS as
\begin{equation*}
   \text{relative-IS} ={\|{\text{positive-IF}}\|}^{-1}\left|C^{\mathrm{T}} {\text{negative-IF}} \right|
\end{equation*}
\ELSE[${\text{positive-IF}}$ is not parallel to ${\text{negative-IF}}$]
\STATE Define $I=\left[{\text{positive-IF}}, {\text{negative-IF}}\right]$.
\STATE Compute the relative-IS as
\begin{equation*}
\text{relative-IS} = {\|C\|}^{-1} \left|C^{\mathrm{T}}  I \left[I^{\mathrm{T}}\cdot I\right]^{-1}  I^{\mathrm{T}}C\right|
\end{equation*}
\ENDIF 
    \STATE {\bf Return:} $\text{relative-IS}$.
\end{algorithmic}
\end{algorithm}

\begin{algorithm}
\caption{Relative Influence Score Based on ECIF\label{alg:relative-IS-self}}
\begin{algorithmic}[1]
    \STATE {\bf Input:} Training Dataset $\mathcal{D} = \{ (x^T, x^I) \}$, dataset $\mathcal{D}^*$ to be evaluated,  test dataset $\mathcal{D}'$.
    \STATE Compute the $\text{ECIF}(\mathcal{D}^*, \mathcal{D}) = \left({\text{positive-IF}}, {\text{negative-IF}}\right)$ by algorithm \ref{alg:ECIF}.
\STATE Compute the gradient of the batch loss function of the test data as 
$$C = \frac{1}{|\mathcal{D}'|}\sum_{(x^T,x^I)\in \mathcal{D}'}\nabla_{\theta} -\log \frac{u^{\mathrm{T}}\cdot v}{\|u\|\cdot\|v\|},$$
where $u$ and $v$ is the embedding for $x^T$ and $x^I$
%\warn{LBatchL_{\text{Batch}} is enough,}
\IF{${\text{positive-IF}}$ is parallel to ${\text{negative-IF}}$}
\STATE Compute the relative-IS as
\begin{equation*}
   \text{relative-IS} ={\|{\text{positive-IF}}\|}^{-1}\left|C^{\mathrm{T}} {\text{negative-IF}} \right|
\end{equation*}
\ELSE[${\text{positive-IF}}$ is not parallel to ${\text{negative-IF}}$]
\STATE Define $I=\left[{\text{positive-IF}}, {\text{negative-IF}}\right]$.
\STATE Compute the relative-IS as
\begin{equation*}
\text{relative-IS} = {\|C\|}^{-1} \left|C^{\mathrm{T}}  I \left[I^{\mathrm{T}}\cdot I\right]^{-1}  I^{\mathrm{T}}C\right|
\end{equation*}
\ENDIF 
    \STATE {\bf Return:} $\text{relative-IS}$.
\end{algorithmic}
\end{algorithm}
\section{Acceleration for Influence Function} \label{sec:further}
\paragraph{LOGRA.}\label{app:logra}
For one layer, given the input $x_i$, output $x_o$ and the weight $W$, the forward and backward computation can be written as $x_o = W x_i$, $\mathrm{vec}({\mathcal{D}}W) = \sum_{t=1}^T x_{i, t} \otimes \mathcal{D}x_{o,t}$, $\mathcal{D}x_i = W^{\mathrm{T}}\mathcal{D}x_o$, where $T$ denotes for the sequence dimension in language modeling, $\mathcal{D}$ the derivative with respect to  the loss, $\otimes$ the Kronecker product, and $\mathrm{vec}(\cdot)$ the vectorization operation. Observing gradient $\mathrm{vec}({\mathcal{D}}W)$ obtained during backpropagation is structured as a sum of Kronecker products between forward and backward activations, LOGRA imposes an additional Kronecker-product structure on the projection matrix P as follows:
\begin{equation*}
P\mathrm{vec}(\mathcal{D}W) \triangleq (P_i \otimes P_o) \mathrm{vec}(\mathcal{D}W) = \sum_{t=1}^{T} (P_i \otimes P_o) (x_{i,t} \otimes D x_{o,t}) = \sum_{t=1}^{T} P_i x_{i,t} \otimes P_o D x_{o,t}
\end{equation*}
where $P_i$  is the projection matrix on the input and $P_o$ is that on the backward activations, and $P\triangleq P_i\otimes P_o$.

\section{Influence Function in Contrastive Learning}

\subsection{Influence Function for Positive Samples.}
We first consider the influence function for positive samples.

\paragraph{Single Data Pair Version.}\label{posi:single}
To quantify the impact of $x^T$ and $x^I$ as positive samples, we first define $L_{T2I}(u_n, V_m; \theta) + L_{I2T}(v_n, U_m; \theta)$ as $\text{Pos}((x^T,x^I); {\theta})$.
Following the idea of influence function, we can up-weight these two parts by $\epsilon$ and obtain an up-weighted loss function as 
\begin{align*}
    L_{\text{Total}, \epsilon}(\theta) = {\sum_{(U,V)\in \mathcal{B}}}L_{\text{Batch}}(U,V;{\theta}) +\frac{\delta}{2}\|\theta\|^2_2  + \epsilon \cdot \text{Pos}((x^T,x^I); {\theta}). 
\end{align*}
And the parameters are obtained by $\hat{\theta}_{\epsilon} = \argmin_{\theta} L_{\text{Total}, \epsilon}(\theta)$. From this minimizing condition, we have
\begin{equation*}
    \begin{split}
         {\sum_{(U,V)\in \mathcal{B}}} \nabla_{\theta}  L_{\text{Batch}}(U,V;\hat{\theta}_{\epsilon}) + \epsilon \cdot\nabla_{\theta} \text{Pos}((x^T,x^I); \hat{\theta}_{\epsilon}) = 0
    \end{split}
\end{equation*}
Perform a Taylor expand at $\theta = \hat{\theta}$, we have
\begin{equation*}
    \begin{split}
    {\sum_{(U,V)\in \mathcal{B}}} \nabla_{\theta}  L_{\text{Batch}}(U,V;\hat{\theta}) + \epsilon \cdot\nabla_{\theta} \text{Pos}((x^T,x^I); \hat{\theta}) + {\sum_{(U,V)\in \mathcal{B}}} \nabla^2_{\theta}  L_{\text{Batch}}(U,V;\hat{\theta})\cdot \left(\hat{\theta}_{\epsilon}-\hat{\theta}\right)\approx 0
    \end{split}
\end{equation*}
Because $\hat{\theta}$ minimizes ${\sum_{(U,V)\in \mathcal{B}}} L_{\text{Batch}}(U,V;\hat{\theta})$, the first term in the above equation equals $0$.
\begin{equation*}
    \begin{split}
        \text{positive-IF}((x^T, x^I);\hat{\theta})\triangleq \left.\frac{\mathrm{d} {\hat{\theta}_{\epsilon}} }{\mathrm{d}\epsilon}\right |_{\epsilon = 0} = -H_{\hat{\theta}}^{-1} \cdot \nabla_{\theta} \text{Pos}((x^T,x^I); \hat{\theta})
    \end{split}
\end{equation*}
where $H_{\hat{\theta}} = \nabla_{\theta}^2 {\sum_{(U,V)\in \mathcal{B}}}L_{\text{Batch}}(U,V;\hat{\theta})+\delta I$ is the Hessian matrix at $\hat{\theta}$.

\paragraph{Extension to Multiple Samples.}\label{posi:multi}
The influence evaluation described above can be extended to a subset $\mathcal{D}^*\subset \mathcal{D}$. Let set $S$ to index the batches containing data from $\mathcal{D}^*$. For every $m\in S$, define an index set $E_m$ to specify the position of data from $\mathcal{D}^*$ within the $m$-th batch. We encapsulate the assigned results as $\text{Seg} = \{(m, E_m)\vert m\in S\}$. By employing a derivation method similar to that used for a single data point, we can obtain the parameter-related influence function for $\mathcal{D}^*$ by summing the influence as a position sample (\ref{eq:4}) for all samples in $\mathcal{D}^*$.

\begin{proposition}
The influence function for dataset $\mathcal{D}^*$ serving as positive sample (positive-IF) is
    \begin{equation*}
    \text{positive-IF}(\mathcal{D}^*, \text{Seg}; \hat{\theta}) = -H_{\hat{\theta}}^{-1} \cdot \nabla_{\theta} \text{Pos}(\mathcal{D}^*, \text{Seg}; \hat{\theta})
\end{equation*}
where
\begin{equation*}
    \text{Pos}(\mathcal{D}^*, \text{Seg}; \hat{\theta}) = \sum_{m\in S}\sum_{n\in E_m}\left( L_{T2I}(u_n, V_m; \hat{\theta}) + L_{I2T}(v_n, U_m; \hat{\theta})\right)
\end{equation*}
\end{proposition}
\begin{proof}
$\text{Seg} = \{(m, E_m)\vert m\in S\}$, for $m\in S$, $U_m$, $V_m$ are the text and image embedding for the $m$-th batch, respectively. For $n\in E_m$, $u_n$ and $v_n$ are embeddings for a single data pair in $m$-th batch, which is included in the dataset to be evaluated $\mathcal{D}^*$.
Define $\text{Pos}(\mathcal{D}^*, \text{Seg}; {\theta})$ as 
\begin{equation*}
\text{Pos}(\mathcal{D}^*, \text{Seg}; {\theta}) = \sum_{m\in S}\sum_{n\in E_m}\left( L_{T2I}(u_n, V_m;{\theta}) + L_{I2T}(v_n, U_m; {\theta})\right)
\end{equation*}
Following the idea of influence function, we can up-weight these by $\epsilon$ and obtain an up-weighted loss function as 
\begin{align*}
    L_{\text{Total}, \epsilon}(\theta) = {\sum_{(U,V)\in \mathcal{B}}}L_{\text{Batch}}(U,V;{\theta}) +\frac{\delta}{2}\|\theta\|^2_2  + \epsilon \cdot \text{Pos}(\mathcal{D}^*, \text{Seg}; {\theta}). 
\end{align*}
And the parameters are obtained by $\hat{\theta}_{\epsilon} = \argmin_{\theta} L_{\text{Total}, \epsilon}(\theta)$. From this minimizing condition, we have
\begin{equation*}
    \begin{split}
         {\sum_{(U,V)\in \mathcal{B}}} \nabla_{\theta}  L_{\text{Batch}}(U,V;\hat{\theta}_{\epsilon}) + \epsilon \cdot\nabla_{\theta} \text{Pos}(\mathcal{D}^*, \text{Seg}; \hat{\theta}_{\epsilon}) = 0
    \end{split}
\end{equation*}
Perform a Taylor expand at $\theta = \hat{\theta}$, we have
\begin{equation*}
    \begin{split}
    {\sum_{(U,V)\in \mathcal{B}}} \nabla_{\theta}  L_{\text{Batch}}(U,V;\hat{\theta}) + \epsilon \cdot\nabla_{\theta} \text{Pos}(\mathcal{D}^*, \text{Seg}; \hat{\theta}) + {\sum_{(U,V)\in \mathcal{B}}} \nabla^2_{\theta}  L_{\text{Batch}}(U,V;\hat{\theta})\cdot \left(\hat{\theta}_{\epsilon}-\hat{\theta}\right)\approx 0
    \end{split}
\end{equation*}
Because $\hat{\theta}$ minimizes ${\sum_{(U,V)\in \mathcal{B}}} L_{\text{Batch}}(U,V;\hat{\theta})$, the first term in the above equation equals $0$.
\begin{equation*}
    \begin{split}
        \text{positive-IF}(\mathcal{D}^*, \text{Seg}; \hat{\theta})\triangleq \left.\frac{\mathrm{d} {\hat{\theta}_{\epsilon}} }{\mathrm{d}\epsilon}\right |_{\epsilon = 0} = -H_{\hat{\theta}}^{-1} \cdot \nabla_{\theta} \text{Pos}(\mathcal{D}^*, \text{Seg}; \hat{\theta})
    \end{split}
\end{equation*}
where $H_{\hat{\theta}} = \nabla_{\theta}^2 {\sum_{(U,V)\in \mathcal{B}}}L_{\text{Batch}}(U,V;\hat{\theta})+\delta I$ is the Hessian matrix at $\hat{\theta}$.
\end{proof}

\subsection{Influence Function for Negative Samples.}\label{app:nega-loss}
Then, we come to derive the influence function for the negative sample. 

In this part, we will illustrate how we give an approximation function for the loss function in which the influence as a negative sample of the data we are considering is removed. With the help of Taylor expansion, this influence is separated into a single term in this approximation function, and we can achieve this by removing this term from the original loss function.

After removing the impact of $(x^T, x^I)$ as a negative sample from the $m$-th batch, the loss function corresponding to this batch should become:
\begin{equation}\label{app:nega-removed}
        L_{\text{T2I}, \text{-neg}}^m((x^T, x^I), S; \theta) = \sum_{\substack{k\in [B]\\k\neq n}} \frac{e^{S_{k,k}}} {\sum_{\substack{j\in [B]\\j\neq n}} e^{S_{k,j}}} + \text{Pos}((x^T,x^I); {\theta}). 
\end{equation}
Mathematically, let $E_n$ be an $B\times B$ matrix such that its $n$-th column and the $n$-th row comprises ones, while all other entries are zero. 
We add the matrix $\log\zeta\times E_n$ to the similarity matrix. Then ${S_{*,n}}$ becomes ${S_{*,n}} + \log \zeta$. The loss function based on the revised similarity matrix becomes: 
\begin{equation*}
\begin{split}
        L_{\text{T2I}, \lambda}^m ((x^T, x^I), S; \theta) =&  \sum_{\substack{k\in [B]\\k\neq n}}-\log \frac{e^{S_{k,k}}} {\sum_{\substack{j\in [B]\\j\neq n}} e^{S_{k,j}}+  e^{\log\zeta} \cdot e^{S_{k,n}}} + \text{Pos}\left((x^T, x^I);{\theta}\right)\\
        =&  \sum_{\substack{k\in [B]\\k\neq n}}-\log \frac{e^{S_{k,k}}} {\sum_{\substack{j\in [B]\\j\neq n}} e^{S_{k,j}}+\zeta\cdot e^{S_{k,n}}} + \text{Pos}\left((x^T,x^I);{\theta}\right)\\
 =&  \sum_{\substack{k\in [B]\\k\neq n}}-\log \frac{e^{S_{k,k}}} {\sum_{\substack{j\in [B]}} e^{S_{k,j}}+  \left(\zeta-1\right) \cdot e^{S_{k,n}}} + \text{Pos}\left((x^T,x^I);{\theta}\right)\\
\end{split}
\end{equation*}
We can easily see that as $\zeta$ approaches $0$, the loss function $L_{\text{T2I}, \zeta}^m$ converges to $L_{\text{T2I, -neg}}^m$ in (\ref{app:nega-removed}). When $\zeta=1$, the loss function equals the original one. To separate the change in the $\zeta$ approaching $0$ from $1$ process, we perform a Taylor expansion at $\zeta=0$ and drop the $O(\left(\zeta-1\right)^2)$ term, then $L_{\text{T2I}, \zeta}^m$ becomes
\begin{align*}
    & \sum_{\substack{k\in [B]\\k\neq n}}-\log \frac{e^{S_{k,k}}} {\sum_{\substack{j\in [B]}} e^{S_{k,j}}} + (\zeta-1)\cdot\sum_{\substack{k\in [B]\\k\neq n}}\left(\frac{\sum_{\substack{j\in [B]}} e^{S_{k,j}}}{e^{S_{k,n}}}\right) + O((\zeta-1)^2) + \text{Pos}((x^T, x^I);{\theta}).
\end{align*}
And by setting $\zeta = 0$, the loss function $L_{\text{T2I}, 0}^m$ indicates that the influence of $(x^T,x^I)$ when it serves as the negative sample is fully removed from the training process.
\begin{align*}
    L_{\text{T2I}, 0}^m = & \sum_{\substack{k\in [B]\\k\neq n}}-\log \frac{e^{S_{k,k}}} {\sum_{\substack{j\in [B]}} e^{S_{k,j}}} + (0-1)\cdot\sum_{\substack{k\in [B]\\k\neq n}}\left(\frac{\sum_{\substack{j\in [B]}} e^{S_{k,j}}}{e^{S_{k,n}}}\right)  + \text{Pos}((x^T, x^I); {\theta})\\
    = & L_{\text{T2I}}^m (S; \theta) - \sum_{\substack{k\in [B]\\k\neq n}}\left(\frac{\sum_{\substack{j\in [B]}} e^{S_{k,j}}}{e^{S_{k,n}}}\right). 
\end{align*}

\paragraph{Single Data Pair Version.}\label{app:nega-single}
From above discussion, to quantify the impact of $x^T$ and $x^I$ as negative samples, we first define $\text{Neg}\left((x^T,x^I);\theta\right)$ as 
\begin{equation*}
  \text{Neg}\left((x^T,x^I);\theta\right)=  \sum_{\substack{k\in [B]\\k\neq n}}\left(\frac{\sum_{\substack{j\in [B]}} e^{S_{k,j}}}{e^{S_{k,n}}} +\frac{\sum_{\substack{j\in [B]}} e^{S_{j,k}}}{e^{S_{n,k}}}\right), 
\end{equation*}
Down-weighting the influence as a negative sample by $\zeta$ from $1$ to $0$, this influence in the loss function is then approximately eliminated. In this process, the loss function becomes
\begin{align*}
    L_{\text{Total}, \zeta}(\theta) = {\sum_{(U,V)\in \mathcal{B}}}L_{\text{Batch}}(U,V;{\theta}) +\frac{\delta}{2}\|\theta\|^2_2  + \left(\zeta -1 \right) \cdot \text{Neg}((x^T,x^I); {\theta}). 
\end{align*}
And the parameters are obtained by $\hat{\theta}_{\zeta} = \argmin_{\theta} L_{\text{Total}, \zeta}(\theta)$. From this minimizing condition, we have
\begin{equation*}
    \begin{split}
         {\sum_{(U,V)\in \mathcal{B}}} \nabla_{\theta}  L_{\text{Batch}}(U,V;\hat{\theta}_{\zeta}) + (\zeta-1) \cdot\nabla_{\theta} \text{Neg}((x^T,x^I); \hat{\theta}_{\zeta}) = 0
    \end{split}
\end{equation*}
Perform a Taylor expand at $\theta = \hat{\theta}$, we have
\begin{equation*}
    \begin{split}
    {\sum_{(U,V)\in \mathcal{B}}} \nabla_{\theta}  L_{\text{Batch}}(U,V;\hat{\theta}) + (\zeta-1) \cdot\nabla_{\theta} \text{Neg}((x^T,x^I); \hat{\theta}) + {\sum_{(U,V)\in \mathcal{B}}} \nabla^2_{\theta}  L_{\text{Batch}}(U,V;\hat{\theta})\cdot \left(\hat{\theta}_{\zeta}-\hat{\theta}\right)\approx 0
    \end{split}
\end{equation*}
Because $\hat{\theta}$ minimizes ${\sum_{(U,V)\in \mathcal{B}}} L_{\text{Batch}}(U,V;\hat{\theta})$, the first term in the above equation equals $0$. Then 
\begin{equation*}
    \begin{split}
       \hat{\theta}_{\zeta}-\hat{\theta}= -(\zeta-1) \cdot H_{\hat{\theta}}^{-1} \cdot \nabla_{\theta} \text{Neg}((x^T,x^I); \hat{\theta})
    \end{split}
\end{equation*}
where $H_{\hat{\theta}} = \nabla_{\theta}^2 {\sum_{(U,V)\in \mathcal{B}}}L_{\text{Batch}}(U,V;\hat{\theta})+\delta I$ is the Hessian matrix at $\hat{\theta}$.
\begin{equation*}
    \begin{split}
        \text{negative-IF}((x^T, x^I);\hat{\theta})\triangleq \left.\frac{\mathrm{d} {\hat{\theta}_{\zeta}} }{\mathrm{d}\zeta}\right |_{\zeta = 0} = -H_{\hat{\theta}}^{-1} \cdot \nabla_{\theta} \text{Neg}((x^T,x^I); \hat{\theta})
    \end{split}
\end{equation*}

\paragraph{Extension to Multiple Samples.}\label{app: nega-multi} Then, we extend the above influence evaluation to a subset $\mathcal{D}^*\subset \mathcal{D}$. Let set $S$ to index the batches containing data from $\mathcal{D}^*$. For every $m\in S$, define an index set $E_m$ to specify the position of data from $\mathcal{D}^*$ within the $m$-th batch. We encapsulate the assigned results as $\text{Seg} = \{(m, E_m)\vert m\in S\}$. By employing a derivation method similar to that used for a single data point, we can obtain the parameter-related influence function for $\mathcal{D}^*$.

\begin{proposition}\label{app:nega-if}
The influence function for dataset $\mathcal{D}^*$ serving as negative sample (negative-IF) is
\begin{equation*}
    \text{Neg}(\mathcal{D}^*,\text{Seg};\theta) =  \sum_{m\in S}\sum_{\substack{k\in [B]/E_m}}\left(\frac{\sum_{\substack{j\in [B]}} e^{S_{k,j}}}{\sum_{n\in E_m}e^{S_{k,n}}} +\frac{\sum_{\substack{j\in [B]}} e^{S_{j,k}}}{\sum_{n\in E_m}e^{S_{n, k}}}\right)
\end{equation*}
And 
\begin{equation*}
\text{negative-IF}(\mathcal{D}^*,\text{Seg};\hat{\theta}) = -H_{\hat{\theta}}^{-1} \cdot \nabla_{\theta}  \text{Neg}(\mathcal{D}^*,\text{Seg};\hat{\theta}).    
\end{equation*}
\end{proposition}
\begin{proof}
$\text{Seg} = \{(m, E_m)\vert m\in S\}$, for $m\in S$, $U_m$, $V_m$ are the text and image embedding for the $m$-th batch, respectively. For $n\in E_m$, $u_n$ and $v_n$ are embeddings for a single data pair in $m$-th batch, which is included in the dataset to be evaluated $\mathcal{D}^*$.

{\bf Step 1}. Noting the data in $\mathcal{D}^*$ may come from different batches and multiple data from one batch, then we firstly derive the loss function approximation with separated negative sample influence removed.

For the $m$-th batch, $m\in S$, after removing the impact of the data indexed by $E_m$ as a negative sample, the loss function corresponding to this batch should become:
\begin{equation}
        L_{\text{T2I}, \text{ -neg}}^m((x^T, x^I), S; \theta) = \sum_{\substack{k\in [B]\\k\notin E_m}} \frac{e^{S_{k,k}}} {\sum_{\substack{j\in [B]\\j\notin E_m}} e^{S_{k,j}}} + \text{Pos}((x^T,x^I); {\theta}). 
\end{equation}
Then, for $n\in E_m$, let $E_n$ be an $B\times B$ matrix such that its $n$-th column and the $n$-th row comprises ones, while all other entries are zero. 
We add the matrix $\log\zeta\times E_n$ to the similarity matrix. Then, the loss function based on the revised similarity matrix becomes: 
\begin{equation*}
    L_{\text{T2I}, \lambda}^m ((x^T, x^I), S; \theta) =  \sum_{\substack{k\in [B]\\k\notin E_m}}-\log \frac{e^{S_{k,k}}} {\sum_{\substack{j\in [B]}} e^{S_{k,j}}+  \left(\zeta-1\right) \cdot \sum_{n\in E_m}e^{S_{k,n}}} + \text{Pos}((x^T,x^I); {\theta}).
\end{equation*}
We can easily see that as $\zeta$ approaches $0$, the loss function $L_{\text{T2I}, \zeta}^m$ converges to $L_{\text{T2I, -neg}}^m$ in (\ref{app:nega-removed}). When $\zeta=1$, the loss function equals the original one. To separate the change in the $\zeta$ approaching $0$ from $1$ process, we perform a Taylor expansion at $\zeta=0$ and drop the $O(\left(\zeta-1\right)^2)$ term, then $L_{\text{T2I}, \zeta}^m$ becomes
\begin{align*}
    & \sum_{\substack{k\in [B]\\k\notin E_m}}-\log \frac{e^{S_{k,k}}} {\sum_{\substack{j\in [B]}} e^{S_{k,j}}} + (\zeta-1)\cdot\sum_{\substack{k\in [B]\\k\notin E_m}}\left(\frac{\sum_{\substack{j\in [B]}} e^{S_{k,j}}}{\sum_{n\in E_m}e^{S_{k,n}}}\right) + O((\zeta-1)^2) + \text{Pos}((x^T,x^I); {\theta}).
\end{align*}
And by setting $\zeta = 0$, the loss function $L_{\text{T2I}, 0}^m$ indicates that the influence of $(x^T,x^I)$ when it serves as the negative sample is fully removed from the training process.
\begin{align*}
    L_{\text{T2I}, 0}^m = &     \sum_{\substack{k\in [B]\\k\notin E_m}}-\log \frac{e^{S_{k,k}}} {\sum_{\substack{j\in [B]}} e^{S_{k,j}}} + (0-1)\cdot\sum_{\substack{k\in [B]\\k\notin E_m}}\left(\frac{\sum_{\substack{j\in [B]}} e^{S_{k,j}}}{\sum_{n\in E_m}e^{S_{k,n}}}\right) + \text{Pos}((x^T,x^I); {\theta})\\
    = & L_{\text{T2I}}^m (S; \theta) - \sum_{\substack{k\in [B]\\k\notin E_m}}\left(\frac{\sum_{\substack{j\in [B]}} e^{S_{k,j}}}{\sum_{n\in E_m}e^{S_{k,n}}}\right) 
\end{align*}
By down-weighting the influence of $\mathcal{D}^*$ as negative samples by $\zeta$, the total loss function becomes
\begin{equation*}
    \begin{split}
         L_{\text{Total}, \zeta}(\theta) = {\sum_{(U,V)\in \mathcal{B}}}L_{\text{Batch}}(U,V;{\theta}) +\frac{\delta}{2}\|\theta\|^2_2  + \left(\zeta -1 \right) \cdot \sum_{m\in S} \sum_{k\in [B]/E_m}\left(\frac{\sum_{\substack{j\in [B]}} e^{S_{k,j}}}{\sum_{n\in E_m}e^{S_{k,n}}}\right) 
    \end{split}
\end{equation*}
Then denote $\text{Neg}(\mathcal{D}^*,\text{Seg};\theta)$ as 
    \begin{equation*}
    \text{Neg}(\mathcal{D}^*,\text{Seg};\theta) =  \sum_{m\in S}\sum_{\substack{k\in [B]/E_m}}\left(\frac{\sum_{\substack{j\in [B]}} e^{S_{k,j}}}{\sum_{n\in E_m}e^{S_{k,n}}} +\frac{\sum_{\substack{j\in [B]}} e^{S_{j,k}}}{\sum_{n\in E_m}e^{S_{n, k}}}\right)
\end{equation*}
And the loss function with the negative-sample influence of $\mathcal{D}^*$ explicitly removed is
\begin{equation*}
    \begin{split}
         L_{\text{Total}, 0}(\theta) = {\sum_{(U,V)\in \mathcal{B}}}L_{\text{Batch}}(U,V;{\theta}) +\frac{\delta}{2}\|\theta\|^2_2  - \text{Neg}(\mathcal{D}^*,\text{Seg};\theta) 
    \end{split}
\end{equation*}

%%%%%%%%%%%%
{\bf Step 2}. The parameters are obtained by $\hat{\theta}_{\zeta} = \argmin_{\theta} L_{\text{Total}, \zeta}(\theta)$. From this minimizing condition, we have
\begin{equation*}
    \begin{split}
         {\sum_{(U,V)\in \mathcal{B}}} \nabla_{\theta}  L_{\text{Batch}}(U,V;\hat{\theta}_{\zeta}) + (\zeta -1 )\cdot\nabla_{\theta} \text{Neg}(\mathcal{D}^*, \text{Seg}; \hat{\theta}) = 0
    \end{split}
\end{equation*}
Perform a Taylor expand at $\theta = \hat{\theta}$, we have
\begin{equation*}
    \begin{split}
    {\sum_{(U,V)\in \mathcal{B}}} \nabla_{\theta}  L_{\text{Batch}}(U,V;\hat{\theta}) + (\zeta -1 ) \cdot\nabla_{\theta} \text{Neg}(\mathcal{D}^*, \text{Seg}; \hat{\theta}) + {\sum_{(U,V)\in \mathcal{B}}} \nabla^2_{\theta}  L_{\text{Batch}}(U,V;\hat{\theta})\cdot \left(\hat{\theta}_{\zeta}-\hat{\theta}\right)\approx 0
    \end{split}
\end{equation*}
Because $\hat{\theta}$ minimizes ${\sum_{(U,V)\in \mathcal{B}}} L_{\text{Batch}}(U,V;\hat{\theta})$, the first term in the above equation equals $0$.
\begin{equation*}
    \begin{split}
        \text{negative-IF}(\mathcal{D}^*, \text{Seg}; \hat{\theta})\triangleq \left.\frac{\mathrm{d} {\hat{\theta}_{\zeta}} }{\mathrm{d}\zeta}\right |_{\zeta = 0} = -H_{\hat{\theta}}^{-1} \cdot \nabla_{\theta} \text{Neg}(\mathcal{D}^*, \text{Seg}; \hat{\theta})
    \end{split}
\end{equation*}
where $H_{\hat{\theta}} = \nabla_{\theta}^2 {\sum_{(U,V)\in \mathcal{B}}} L_{\text{Batch}}(U,V;\hat{\theta})+\delta I$ is the Hessian matrix at $\hat{\theta}$.
\end{proof}

\section{Approximation Error Bound}\label{app:error_bound}
In the previous discussion, we have established that when applying the influence function method to contrastive learning, it is impractical to design a sample-specific up-weighting scheme that approximates the corresponding loss function resulting from the removal of a single pair in the batch without affecting the remaining data. Therefore, based on the previous derivation, we provide an estimation function $L^-$ for this loss function. Consider the dataset $\mathcal{D}^*$, define
\begin{equation*}
    L'(\mathcal{D}^*, \text{Seg}; \theta)\triangleq\text{Pos}(\mathcal{D}^*, \text{Seg}; \theta) + \cdot\text{Neg}(\mathcal{D}^*, \text{Seg}; \theta),
\end{equation*}
Then the loss function with the influence of $\mathcal{D}^*$ removed becomes
\begin{equation}\label{eq:approxiated_loss}
    \begin{split}
       L^-(\mathcal{B},\mathcal{D}^*, \text{Seg}; \theta) = L_{Total}(\mathcal{B}; \theta) -  L'(\mathcal{D}^*, \text{Seg}; \theta).
    \end{split}
\end{equation}
Equation \eqref{eq:approxiated_loss} is based on the assumption that the influence of data acting as positive and negative samples on model parameters can be linearly superimposed, and we can leverage ECIF to edit the model based on the following corollary. This approach enables us to achieve the unlearning or updating of specific data without the need to remove data and retrain the model. 

Assume $\hat{\theta} = \argmin L_{Total}$ is the original model parameter, and $\hat{\theta}(-\mathcal{D}^*)$ is the minimizer of $L^-$, which is obtained from retraining. Denote $\theta_{if}(-\mathcal{D}^*)$ as the updated model with the influence of $\mathcal{D}^*$ removed and is obtained by the ECIF method, which is an estimation for $\hat{\theta}(-\mathcal{D}^*)$.
Because we concentrate on $\mathcal{D}^*$, we omit the $\text{Seg}$ in the above definitions for short.

In this part, we will study the error between the estimated influence given by the ECIF method and retraining. We use the parameter changes as the evaluation metric:
\begin{equation}\label{app:theorem_total}
    \left|\left(\theta_{if}(-\mathcal{D}^*)-\hat{\theta}\right) - \left( \hat{\theta}(-\mathcal{D}^*) - \hat{\theta}\right)\right| =  \left|\theta_{if}(-\mathcal{D}^*) -  \hat{\theta}(-\mathcal{D}^*)\right|
\end{equation}
Before our main theorem of the upper bound for equation (\ref{app:theorem_total}), we need to prove corollaries and make some assumptions.
\begin{proposition}
Assume that influence as positive sample and as negative sample can be linearly superposed. Then when the influence of dataset $\mathcal{D}^*$ as positive sample is up-weighted by $\epsilon$ and that as negative sample is up-weighted by $\zeta$, then the loss function become
\begin{equation*}
   L^-(\mathcal{D}^*, \text{Seg}; \theta; \epsilon, \zeta)\triangleq L_{Total}(\mathcal{B}; \theta) + \epsilon \cdot \nabla_{\theta}\text{Pos}(\mathcal{D}^*,\text{Seg};\hat{\theta}) + \left(\zeta -1\right)\cdot\text{Neg}(\mathcal{D}^*,\text{Seg};\hat{\theta})
\end{equation*}
And corresponding parameters ${\theta}_{\epsilon, \zeta}$ are defined as
\begin{equation*}
    \hat{\theta}_{\epsilon, \zeta}(-\mathcal{D}^*) = \argmin_{\theta}L^-(\mathcal{D}^*, \text{Seg}; \theta; \epsilon, \zeta)
\end{equation*}
 The approximation of $\hat{\theta}_{\epsilon, \zeta}(-\mathcal{D}^*)$ is derived as
 \begin{equation}
     \hat{\theta}_{\epsilon, \zeta}(\mathcal{D}^*)\approx {\theta}_{\epsilon, \zeta}(\mathcal{D}^*)  = \hat{\theta}-H_{\hat{\theta}}^{-1} \cdot \left( \frac{\sqrt{2}}{2}\cdot \nabla_{\theta} \text{Pos}(\mathcal{D}^*, \text{Seg}; \hat{\theta}) +  \frac{\sqrt{2}}{2}\cdot \nabla_{\theta} \text{Neg}(\mathcal{D}^*, \text{Seg}; \hat{\theta}) \right)
 \end{equation}
\end{proposition}
\begin{property}
Assume that influence as positive sample and as negative sample can be linearly superposed. Then when the influence of dataset $\mathcal{D}^*$ as positive sample is up-weighted by $\epsilon$ and that as negative sample is up-weighted by $\zeta$, then the loss function become
\begin{equation*}
   L^-(\mathcal{D}^*, \text{Seg}; \theta; \epsilon, \zeta)\triangleq L_{Total}(\mathcal{B}; \theta) + \epsilon \cdot \nabla_{\theta}\text{Pos}(\mathcal{D}^*,\text{Seg};\hat{\theta}) + \left(\zeta -1\right)\cdot\text{Neg}(\mathcal{D}^*,\text{Seg};\hat{\theta})
\end{equation*}
And corresponding parameters ${\theta}_{\epsilon, \zeta}$ are defined as
\begin{equation*}
    \hat{\theta}_{\epsilon, \zeta}(-\mathcal{D}^*) = \argmin_{\theta}L^-(\mathcal{D}^*, \text{Seg}; \theta; \epsilon, \zeta)
\end{equation*}
 The approximation of $\hat{\theta}_{\epsilon, \zeta}(-\mathcal{D}^*)$ is derived as
 \begin{equation}\label{app:edit_para}
\begin{split}
\hat{\theta}_{\epsilon, \zeta}(-\mathcal{D}^*) \approx & {\theta}_{\epsilon, \zeta}(-\mathcal{D}^*)\\
\triangleq&\hat{\theta} -H^{-1}_{\hat{\theta}} \cdot\left(\epsilon \cdot \nabla_{\theta}\text{Pos}(\mathcal{D}^*,\text{Seg};\hat{\theta}) + \left(\zeta -1\right)\cdot\nabla_{\theta}\text{Neg}(\mathcal{D}^*,\text{Seg};\hat{\theta})\right)
\end{split}
\end{equation}
\end{property}

\begin{proof}
    Assume that influence as positive sample and as negative sample can be linearly superposed. Then when the influence of dataset $\mathcal{D}^*$ as positive sample is up-weighted by $\epsilon$ and that as negative sample is up-weighted by $\zeta$, then the loss function become
\begin{equation*}
   L^-(\mathcal{D}^*, \text{Seg}; \theta; \epsilon, \zeta)\triangleq L_{Total}(\mathcal{B}; \theta) + \epsilon \cdot \text{Pos}(\mathcal{D}^*,\text{Seg};\hat{\theta}) + \left(\zeta -1\right)\cdot\text{Neg}(\mathcal{D}^*,\text{Seg};\hat{\theta})
\end{equation*}
And corresponding parameters ${\theta}_{\epsilon, \zeta}$ are defined as
\begin{equation*}
    \hat{\theta}_{\epsilon, \zeta}(-\mathcal{D}^*) = \argmin_{\theta}L^-(\mathcal{D}^*, \text{Seg}; \theta; \epsilon, \zeta)
\end{equation*}
Then, from the minimizing condition, 
\begin{equation*}
    \nabla_{\theta} L_{Total}(\mathcal{B}; \hat{\theta}_{\epsilon, \zeta}) + \epsilon \cdot \nabla_{\theta}\text{Pos}(\mathcal{D}^*,\text{Seg};\hat{\theta}_{\epsilon, \zeta}) + \left(\zeta -1\right)\cdot\nabla_{\theta}\text{Neg}(\mathcal{D}^*,\text{Seg};\hat{\theta}_{\epsilon, \zeta}) = 0,
\end{equation*}
where $\hat{\theta}_{\epsilon, \zeta}(-\mathcal{D}^*)$ is written as $\hat{\theta}_{\epsilon, \zeta}$ for short.
Perform a Taylor expansion around ${\theta}= \hat{\theta}$, then we have
\begin{equation*}
\begin{split}
     \nabla_{\theta} L_{Total}(\mathcal{B}; \hat{\theta}) + \epsilon \cdot \nabla_{\theta}\text{Pos}(\mathcal{D}^*,\text{Seg};\hat{\theta}) + \left(\zeta -1\right)\cdot\nabla_{\theta}\text{Neg}(\mathcal{D}^*,\text{Seg};\hat{\theta})& \\
     + \nabla^2_{\theta}L_{Total}(\mathcal{B}; \hat{\theta}) \cdot \left(\hat{\theta}_{\epsilon, \zeta} - \hat{\theta}\right) &= 0.
\end{split}
\end{equation*}
Because $\hat{\theta}$ minimizes $L_{Total}(\mathcal{B}; {\theta})$, the first term in above equation equals $0$. Then we have
\begin{equation*}
\begin{split}
\hat{\theta}_{\epsilon, \zeta} &\approx \hat{\theta} -H^{-1}_{\hat{\theta}} \cdot\left(\epsilon \cdot \nabla_{\theta}\text{Pos}(\mathcal{D}^*,\text{Seg};\hat{\theta}) + \left(\zeta -1\right)\cdot\nabla_{\theta}\text{Neg}(\mathcal{D}^*,\text{Seg};\hat{\theta})\right)\\
 & = \hat{\theta} - \epsilon \cdot H^{-1}_{\hat{\theta}} \cdot \nabla_{\theta}\text{Pos}(\mathcal{D}^*,\text{Seg};\hat{\theta}) - \left(\zeta -1\right)\cdot H^{-1}_{\hat{\theta}} \cdot\nabla_{\theta}\text{Neg}(\mathcal{D}^*,\text{Seg};\hat{\theta})\\
 & = \hat{\theta} + \epsilon \cdot  \text{positive-IF}(\mathcal{D}^*, \text{Seg}; \hat{\theta}) + (\zeta-1) \cdot  \text{negative-IF}(\mathcal{D}^*, \text{Seg}; \hat{\theta})
\end{split}
\end{equation*}
where $H_{\hat{\theta}} = \nabla_{\theta}^2 {\sum_{(U,V)\in \mathcal{B}}} L_{\text{Batch}}(U,V;\hat{\theta})+\delta I$. When $\epsilon =-1$, $\zeta = 0$, $\hat{\theta}_{-1, 0}$ estimates the parameters obtained by retraining after $\mathcal{D}^*$ removed.

% \warn{wait for change}
% Given a direction α\alpha, 
% the directional derivative of ˆθϵ,ζ\hat{\theta}_{\epsilon, \zeta} is defined as
% \begin{equation*}
%         D_{\alpha}  \hat{\theta}_{\epsilon, \zeta} 
%         \approx -H_{\hat{\theta}}^{-1} \cdot \left( \cos{\alpha} \cdot \nabla_{\theta} \text{Pos}(\mathcal{D}^*, \text{Seg}; \hat{\theta}) + \sin{\alpha} \cdot \nabla_{\theta} \text{Neg}(\mathcal{D}^*, \text{Seg}; \hat{\theta}) \right)
% \end{equation*}
% We choose α\alpha to be 4545 degrees by default and take a one-step approximation along this direction, then we have
% \begin{equation*}
%     \begin{split}
%                  &\hat{\theta}(-\mathcal{D}^*)\approx {\theta}_{\epsilon, \zeta}(\mathcal{D}^*) \\
%          =& \hat{\theta}-H_{\hat{\theta}}^{-1} \cdot \left( \frac{\sqrt{2}}{2} \cdot \nabla_{\theta} \text{Pos}(\mathcal{D}^*, \text{Seg}; \hat{\theta}) + \frac{\sqrt{2}}{2} \cdot \nabla_{\theta} \text{Neg}(\mathcal{D}^*, \text{Seg}; \hat{\theta}) \right)
%     \end{split}
% \end{equation*}
\end{proof}

\begin{assumption}The loss $L_{\text{Batch}}(U, V, \theta)$ is convex and twice-differentiable in $\theta$, with positive regularization $\delta > 0$. There exists $C_H \in \mathbb{R}$ such that
$$\| \nabla^2_{\theta} L_{\text{Batch}}(U, V; \theta_1) - \nabla^2_{\theta} L_{\text{Batch}}(U, V; \theta_2)\|_{2} \leq C_H \| \theta_1 - \theta_2 \|_2$$
for all $(U, V) \in \mathcal{B} $ and $\theta_1, \theta_2 \in \Theta$. 
\end{assumption}

\begin{assumption}The function $L'((x^T, x^I); \theta)$:
\begin{equation*}
    L'((x^T, x^I); \theta) =  \text{Pos}((x^T, x^I); \theta) +  \text{Neg}((x^T, x^I); \theta)
\end{equation*}
is convex and twice-differentiable in $\theta$, with some positive regularization. There exists $C'_H \in \mathbb{R}$ such that
$$\| \nabla^2_{\theta}  L'((x^T, x^I); \theta_1) - \nabla^2_{\theta}  L'((x^T, x^I); \theta_2)\|_{2} \leq C'_H \| \theta_1 - \theta_2 \|_2$$
for all $(x^T, x^I) \in \mathcal{D}^* $ and $\theta_1, \theta_2 \in \Theta$. 
\end{assumption}
\begin{corollary}
    \begin{equation*}
        \|\nabla^2_{\theta}  L^-(\mathcal{D}^*, \text{Seg}; \theta_1) - \nabla^2_{\theta}  L^-(\mathcal{D}^*, \text{Seg}; \theta_2)\|_2\leq  
     \left(|\mathcal{B}|\cdot C_H + |\mathcal{D}^*|\cdot C_H'||\right)\|\theta_1-\theta_2\| 
    \end{equation*}
    Define $C_H^- \triangleq|\mathcal{B}|\cdot C_H + |\mathcal{D}^*|\cdot C_H'$
\end{corollary}

\begin{definition}
Define $|\mathcal{D}|$ as the number of pairs 
\begin{equation*}
        C'_{L} = \max_{(x^T,x^I)\in\mathcal{B}} \left\| \nabla_{\theta} L'((x^T, x^I); \hat{\theta})\right\|_2,
\end{equation*}
\begin{equation*}
    \sigma'_{\text{min}} = \text{smallest singular value of } \nabla^2_{\theta}  L^-(\mathcal{D}^*, \text{Seg}; \hat{\theta}),
\end{equation*}
\begin{equation*}
    \sigma_{\text{min}} = \text{smallest singular value of } \nabla^2_{\theta}  L_{\text{Total}}(\mathcal{B}; \hat{\theta}),
\end{equation*}
\end{definition}
Based on above corollaries and assumptions, we derive the following theorem.
\begin{theorem}
    We obtain the error between the actual influence and our predicted influence as follows:
\begin{equation*}
    \begin{split}
         &\left\|\hat{\theta}(-\mathcal{D}^*) - {\theta}_{if}(-\mathcal{D}^*)\right\|\\
         \leq & \frac{C'_HC_H^- |\mathcal{D}^*|^2  {C'_{L}}^2}{2 (\sigma'_{\text{min}} + \delta)^3} + \left|\frac{2\delta+\sigma_{\text{min}}+\sigma'_{\text{min}}}{\left(\delta+ \sigma'_{\text{min}}\right)\cdot\left(\delta+ \sigma_{\text{min}}\right)}\right| \cdot C_L'|\mathcal{D}^*|
    \end{split}
\end{equation*}
\end{theorem}
\begin{proof}
We will use the one-step Newton approximation as an intermediate step. Define $\Delta\theta_{Nt}(-\mathcal{D}^*)$ as
\begin{equation*}
    \Delta\theta_{Nt}(-\mathcal{D}^*)\triangleq H_{\delta}^{-1}\cdot \nabla_{\theta}L'(\mathcal{D}^*, \text{Seg}; \hat{\theta}),
\end{equation*}
 where $H_{\delta} = \delta \cdot I + \nabla_{\theta}^2L^-(\mathcal{D}^*, \text{Seg};\hat{\theta})$ is the regularized empirical Hessian at $\hat{\theta}$ but reweighed after removing the influence of $\mathcal{D}^*$. Then the one-step Newton approximation for $\hat{\theta}(-\mathcal{D}^*)$ is defined as $\theta_{Nt}(-\mathcal{D}^*) \triangleq \Delta\theta_{Nt}(-\mathcal{D}^*) + \hat{\theta}$.

In the following, we will separate the error between $\theta_{if}(-\mathcal{D}^*)$ and $\hat{\theta}(-\mathcal{D}^*)$ into the following two parts:
\begin{equation*}
    \hat{\theta}(-\mathcal{D}^*) - \theta_{if}(-\mathcal{D}^*) = \underbrace{\hat{\theta}(-\mathcal{D}^*) - \theta_{Nt}(-\mathcal{D}^*)}_{\text{Err}_{\text{Nt, act}}(-\mathcal{D}^*)} + \underbrace{\left(\theta_{Nt}(-\mathcal{D}^*)-\hat{\theta}\right) - \left(\theta_{if}(-\mathcal{D}^*) - \hat{\theta}\right)}_{\text{Err}_{\text{Nt, if}}(-\mathcal{D}^*)}
\end{equation*}

Firstly, in {\bf Step $1$}, we will derive the bound for Newton-actual error ${\text{Err}_{\text{Nt, act}}(-\mathcal{D}^*)}$.
Since $L^-(\theta)$ is strongly convex with parameter $\sigma'_{\text{min}} + \delta$ and minimized by 
$\hat{\theta}(-\mathcal{D}^*)$, we can bound the distance
$\left\|\hat{\theta}(-\mathcal{D}^*) - {\theta}_{Nt}(-\mathcal{D}^*)\right\|_2$ in terms of the norm of the gradient at ${\theta}_{Nt}$:
\begin{equation}\label{bound:1}
    \left\|\hat{\theta}(-\mathcal{D}^*) - {\theta}_{Nt}(-\mathcal{D}^*)\right\|_2 \leq \frac{2}{\sigma'_{\text{min}} + \delta} \left\|\nabla_{\theta} L^- \left({\theta}_{Nt}(-\mathcal{D}^*)\right)\right\|_2
\end{equation}
Therefore, the problem reduces to bounding $\left\|\nabla_{\theta} L^- \left({\theta}_{Nt}(-\mathcal{D}^*)\right)\right\|_2$.
Noting that $\nabla_{\theta}L'(\hat{\theta}) = -\nabla_{\theta}L^-$. This is because $\hat{\theta}$ minimizes $L^- + L'$, that is, $$\nabla_{\theta}L^-(\hat{\theta}) + \nabla_{\theta}L'(\hat{\theta}) = 0.$$ Recall that $\Delta\theta_{Nt}= H_{\delta}^{-1}\cdot \nabla_{\theta}L'(\mathcal{D}^*, \text{Seg}; \hat{\theta}) = -H_{\delta}^{-1}\cdot \nabla_{\theta}L^-(\mathcal{D}^*, \text{Seg}; \hat{\theta})$.
Given the above conditions, we can have this bound for $\text{Err}_{\text{Nt, act}}(-\mathcal{D}^*)$.
\begin{equation}\label{bound:2}
    \begin{split}
            &\left\|\nabla_{\theta} L^- \left({\theta}_{Nt}(-\mathcal{D}^*)\right)\right\|_2\\
    = & \left\|\nabla_{\theta} L^- \left(\hat{\theta} + \Delta{\theta}_{Nt}(-\mathcal{D}^*)\right)\right\|_2\\
    = & \left\|\nabla_{\theta} L^- \left(\hat{\theta} + \Delta \theta_{N_t}(-\mathcal{D}^*)\right) - \nabla_{\theta} L^- \left(\hat{\theta} \right) -  \nabla_{\theta}^2 L^- \left(\hat{\theta}\right) \cdot  \Delta \theta_{N_t}(-\mathcal{D}^*)\right\|_2\\
     = & \left\|\int_0^1 \left(\nabla_{\theta}^2 L^- \left(\hat{\theta} + t\cdot \Delta \theta_{Nt}(-\mathcal{D}^*)\right) - \nabla_{\theta}^2 L^- \left(\hat{\theta}\right)\right) \Delta \theta_{Nt}(-\mathcal{D}^*) \, dt\right\|_2\\
    \leq & \frac{C_H^-}{2} \left\|\Delta \theta_{Nt}(-\mathcal{D}^{*})\right\|_2^2 =  \frac{C_H^-}{2} \left\|\left[\nabla_{\theta}^2 L^-(\hat{\theta})\right]^{-1} \nabla_{\theta} L^-(\hat{\theta})\right\|_2^2\\
    \leq & \frac{C_{H}^{-}}{2 (\sigma'_{\text{min}} + \delta)^2} \left\|\nabla_{\theta} L^-(\hat{\theta})\right\|_2^2 = \frac{C_H^-}{2 (\sigma'_{\text{min}} + \delta)^2} \left\|\nabla_{\theta} L'(\hat{\theta})\right\|_2^2\\
    \leq &\frac{C_{H}^{-} {\|\mathcal{D}^*\|}^2  {C'_{L}}^2}{2 (\sigma'_{\text{min}} + \delta)^2}.
    \end{split}
\end{equation}

Now we come to {\bf Step $2$} to bound ${\text{Err}_{\text{Nt, if}}(-\mathcal{D}^*)}$, and we will bound the difference in parameter change between Newton and our ECIF method.
\begin{align*}
        &\left\|{\left(\theta_{Nt}(-\mathcal{D}^*)-\hat{\theta}\right) - \left(\theta_{if}(-\mathcal{D}^*) - \hat{\theta}\right)}\right\| \\
        =& \left\| \left[\left({\delta \cdot I+ \nabla_{\theta}^2 L^- \left(\hat{\theta}\right)}\right)^{-1} + \left({\delta \cdot I+ \nabla_{\theta}^2 L_{\text{Total}} \left(\hat{\theta}\right)}\right)^{-1}\right]\cdot \nabla_{\theta}L'(\mathcal{D}^*, \text{Seg}; \hat{\theta})\right\|
\end{align*}
For simplification, we use matrix $A$, $B$ for the following substitutions:
\begin{align*}
    A = {\delta \cdot I+ \nabla_{\theta}^2 L^- \left(\hat{\theta}\right)}\\
    B = {\delta \cdot I+ \nabla_{\theta}^2 L_{\text{Total}} \left(\hat{\theta}\right)}
\end{align*}
And $A$ and $B$ are positive definite matrices with the following properties
\begin{align*}
\delta + \sigma'_{\text{min}} \prec A \prec \delta +   \sigma'_{\text{max}}\\
\delta + \sigma_{\text{min}} \prec B \prec \delta +   \sigma_{\text{max}}\\
\end{align*}
Therefore, we have
\begin{equation}\label{bound:3}
    \begin{split}
             &\left\|{\left(\theta_{Nt}(-\mathcal{D}^*)-\hat{\theta}\right) - \left(\theta_{if}(-\mathcal{D}^*) - \hat{\theta}\right)}\right\| \\
     =&\left\|\left(A^{-1}+B^{-1}\right)\cdot \nabla_{\theta}L^-(\mathcal{D}^*, \text{Seg}; \hat{\theta})\right\| \\
     \leq & \left\|A^{-1}+B^{-1}\right\|\cdot \left\|\nabla_{\theta}L^-(\mathcal{D}^*, \text{Seg}; \hat{\theta})\right\|\\
     \leq & \left|\frac{2\delta+\sigma_{\text{min}}+\sigma'_{\text{min}}}{\left(\delta+ \sigma'_{\text{min}}\right)\cdot\left(\delta+ \sigma_{\text{min}}\right)}\right|\cdot\left\|\nabla_{\theta}L^-(\mathcal{D}^*, \text{Seg}; \hat{\theta})\right\|\\
     \leq& \left|\frac{2\delta+\sigma_{\text{min}}+\sigma'_{\text{min}}}{\left(\delta+ \sigma'_{\text{min}}\right)\cdot\left(\delta+ \sigma_{\text{min}}\right)}\right| \cdot C_L'|\mathcal{D}^*|
    \end{split}
\end{equation}
By combining the conclusions from Step I and Step II in Equations \ref{bound:1}, \ref{bound:2} and \ref{bound:3}, we obtain the error between the actual influence and our predicted influence as follows:
\begin{equation*}
    \begin{split}
         &\left\|\hat{\theta}(-\mathcal{D}^*) - \theta_{if}(-\mathcal{D}^*)\right\|\\
         \leq & \frac{C'_HC_H^- |\mathcal{D}^*|^2  {C'_{L}}^2}{2 (\sigma'_{\text{min}} + \delta)^3} + \left|\frac{2\delta+\sigma_{\text{min}}+\sigma'_{\text{min}}}{\left(\delta+ \sigma'_{\text{min}}\right)\cdot\left(\delta+ \sigma_{\text{min}}\right)}\right| \cdot C_L'|\mathcal{D}^*|.
    \end{split}
\end{equation*}
It is notable that such error bound is small when the number of removal samples $|\mathcal{D}^*|$ is fixed as in practice $\delta=O(|\mathcal{B}|)$.  
\end{proof}

\section{Applications of ECIF}
\subsection{Task-related IS}
\begin{property}
Considering a specific set $\mathcal{D}'$ with text and image embeddings $U'$ and $V'$, and a dataset $D^*$ to be removed, then we have 
%. Then its Task-related Influence Score related to D∗\mathcal{D}^*  with Seg={(m,Em)|m∈S}\text{Seg} = \{(m, E_m)\vert m\in S\} is defined as
\begin{align}
& L_{\text{Batch}}(U', V'; \hat{\theta}(-D^*))-L_{\text{Batch}}(U', V'; \hat{\theta})\approx    \nabla L_{\text{Batch}}(U{'},V{'};\hat{\theta})^{\mathrm{T}} (\hat{\theta}(-D^*)- \hat{\theta})\notag \\
  & = - \nabla L_{\text{Batch}}(U{'},V{'};\hat{\theta})^{\mathrm{T}} \cdot \left( \text{positive-IF}(\mathcal{D}^*,\text{Seg}; \hat{\theta})+\text{negative-IF}(\mathcal{D}^*,\text{Seg}; \hat{\theta})\right).
\end{align}
where $\hat{\theta}(-D^*)$ is the optimal model for the loss eliminating $D^*$, $  \text{positive-IF}(\mathcal{D}^*;\text{Seg};\hat{\theta})$ and $\text{negative-IF}(\mathcal{D}^*,\text{Seg};\hat{\theta})$ are obtained from Proposition \ref{pro:ECIF} for $D^*$. 
\end{property} 
\begin{proof}
    \begin{align*}
    &\text{IS}(\mathcal{D}', \mathcal{D}^*; \text{Seg})
    \triangleq -\left.\frac{\mathrm{d} L_{\text{Batch}}(U{'},V{'};{\theta}_{\epsilon, \zeta = 0}) }{\mathrm{d} \epsilon}\right|_{\epsilon = 0} - \left.\frac{\mathrm{d} L_{\text{Batch}}(U{'},V{'};{\theta}_{\epsilon = 0, \zeta}) }{\mathrm{d} \zeta}\right|_{\zeta = 0}\\
     \approx&- \nabla L_{\text{Batch}}(U{'},V{'};\hat{\theta})^{\mathrm{T}} \cdot \left( \text{positive-IF}(\mathcal{D}^*,\text{Seg}; \hat{\theta})+\text{negative-IF}(\mathcal{D}^*,\text{Seg}; \hat{\theta})\right)
\end{align*}
\end{proof}

\subsection{Relative Influence Score}\label{app:relative-IS}
\begin{proposition}Define $I = [\text{positive-IF}(x;\hat{\theta}), \text{negative-IF}(x;\hat{\theta})]$. If the $2\times 2$ matrix $I^{\mathrm{T}}\cdot I$ is irreversible, then the optimization problem
  \begin{align}\label{app:optimization_inequl}
     & \arg\max_{x \in \mathcal{D}}  \max_{\epsilon, \zeta} \left|  L_{\text{Batch}}(U{'},V{'};\hat{\theta} + \Delta\hat{\theta}_{\epsilon, \zeta} (x))   -L_{\text{Batch}}(U{'},V{'};\hat{\theta})\right|   \quad\text{s.t.} \left\| \Delta \hat{\theta}_{\epsilon, \zeta} (x)\right\|^2 \leq \delta^2
\end{align}
is equivalent to
\begin{equation*}
\arg\max_{x \in \mathcal{D}}  {\|\text{negative-IF}(x;\hat{\theta})\|}_2^{-1}\left| \nabla L_{\text{Batch}}(U{'},V{'};\hat{\theta})^{\mathrm{T}}\cdot  \text{negative-IF}(x;\hat{\theta}) \right|.
\end{equation*}
Else,  $I^{\mathrm{T}}\cdot I$ is reversible, then the initial problem is equivalent to
\begin{equation*}
   \arg\max_{x \in \mathcal{D}} {\|\nabla L_{\text{Batch}}(U{'},V{'};\hat{\theta})\|}_2^{-1} \left|\nabla L_{\text{Batch}}(U{'},V{'};\hat{\theta})^{\mathrm{T}}\cdot  I \cdot\left[I^{\mathrm{T}}\cdot I\right]^{-1}  \cdot I^{\mathrm{T}}\cdot\nabla L_{\text{Batch}}(U{'},V{'};\hat{\theta})\right|.
\end{equation*}
\end{proposition}

\begin{proof}
From (\ref{app:edit_para}), we have
\begin{align}\label{loss:difference:proof}
    &\left|L_{\text{Batch}}(U{'},V{'};\hat{\theta} + \Delta\hat{\theta}_{\epsilon, \zeta} (x))   -L_{\text{Batch}}(U{'},V{'};\hat{\theta})\right|\\
        \approx &\left|\nabla L((U', V'); \hat{\theta})^{\mathrm{T}} \cdot\Delta\hat{\theta}_{\epsilon, \zeta} (x))\right|\\
    \approx &\left|\nabla L((U', V'); \hat{\theta})^{\mathrm{T}} \cdot \left(\epsilon \cdot \text{positive-IF}((x^T,x^I);\hat{\theta}) + \left(\zeta -1\right)\text{negative-IF}((x^T,x^I);\hat{\theta})\right)\right|\label{app:relative_version}
\end{align}
And still from (\ref{app:edit_para}), the constraint in parameter changes can be written as
\begin{align}\label{app:loss:difference}
    &\left\| \Delta \hat{\theta}_{\epsilon, \zeta} (x)\right\|\\
    =& \|\epsilon \cdot \text{positive-IF}((x^T,x^I);\hat{\theta}) + \left(\zeta -1\right)\text{negative-IF}((x^T,x^I);\hat{\theta})\|\leq \delta
\end{align}
We can regard (\ref{app:relative_version}) as the inner product between vector $u\triangleq\nabla L((U', V'); \hat{\theta})$ and vector $v\triangleq\epsilon \cdot \text{positive-IF} + \left(\zeta -1\right)\text{negative-IF}$. 

If $\text{positive-IF}$ is not parallel to $\text{negative-IF}$, then the constraint in equation (\ref{app:loss:difference}) becomes that vector $v$ is chosen from a ball of radius $\delta$. Otherwise, the constraint is equivalent to a constraint on the norm of a vector that is parallel to $\text{positive-IF}$ or $\text{negative-IF}$. Therefore, we will proceed with a classification discussion based on whether $\text{positive-IF}$ and $\text{negative-IF}$ are parallel.

Firstly, we consider the $\not\parallel$ case. As is well known, the inner product of vectors reaches its extreme when the two vectors are parallel. We can choose $\epsilon$ and $\zeta$ freely to make vectors $v\parallel u$.
Assume that there exists $c\in\mathbb{R}$ s.t.
\begin{equation*}
    \left[\text{positive-IF}, \text{negative-IF} \right]\cdot\begin{bmatrix} \epsilon \\ \zeta-1 \end{bmatrix} = c \cdot \nabla L((U', V'); \hat{\theta})
\end{equation*}
Denote $\left[\text{positive-IF}, \text{negative-IF} \right]$ as $I$
\begin{align*}
    \left[\text{positive-IF},\text{negative-IF} \right]\cdot\begin{bmatrix} \epsilon \\ \zeta-1 \end{bmatrix} &= c \cdot \nabla L((U', V'); \hat{\theta})\\
    \begin{bmatrix} \text{positive-IF}^{\mathrm{T}}\\ \text{negative-IF}^{\mathrm{T}} \end{bmatrix}\cdot \left[\text{positive-IF}, \text{negative-IF} \right]\cdot\begin{bmatrix} \epsilon \\ \zeta-1 \end{bmatrix}  &= c \cdot \begin{bmatrix} \text{positive-IF}^{\mathrm{T}}\\ \text{negative-IF}^{\mathrm{T}} \end{bmatrix}\cdot \nabla L((U', V'); \hat{\theta})\\
    I^{\mathrm{T}}\cdot I\cdot \begin{bmatrix} \epsilon \\ \zeta-1 \end{bmatrix}  &= c\cdot I^{\mathrm{T}} \cdot  \nabla L((U', V'); \hat{\theta})\\
    \begin{bmatrix} \epsilon \\ \zeta-1 \end{bmatrix} &= c\cdot\left[I^{\mathrm{T}}\cdot I\right]^{-1} \cdot I^{\mathrm{T}}\cdot\nabla L((U', V'); \hat{\theta})
\end{align*}
Noting that $ I^{\mathrm{T}}\cdot I$ is invertible matrix as long as $\text{positive-IF}$, $\text{negative-IF}$ are not parallel.
Considering the constraints of the length of vector $v$, then 
\begin{equation*}
    \|c \cdot \nabla L((U', V'); \hat{\theta})\| \leq \delta
\end{equation*}
We can make vector $v$ reach its largest norm with setting $c$ to an appropriate number:
\begin{equation*}
    c  = \frac{\delta}{\|\nabla L((U', V'); \hat{\theta})\|}
\end{equation*}
Finally, we obtain the expression of vector $2$ that maximizes expression (\ref{loss:difference:proof})
\begin{equation*}
     \left[\text{positive-IF}, \text{negative-IF} \right]\cdot\begin{bmatrix} \epsilon \\ \zeta-1 \end{bmatrix} = c \cdot I \cdot\left[I^{\mathrm{T}}\cdot I\right]^{-1} \cdot I^{\mathrm{T}}\cdot\nabla L((U', V'); \hat{\theta})
\end{equation*}
Then we have
\begin{align*}
    &\left| L((U', V'); \theta_{\epsilon, \zeta}(x^T, x^I)) - L((U', V'); \hat{\theta}) \right| \\
    =& \left|\nabla L((U', V'); \hat{\theta})^{\mathrm{T}} \cdot  \left(\left[\text{positive-IF}, \text{negative-IF} \right]\cdot\begin{bmatrix} \epsilon \\ \zeta-1 \end{bmatrix}\right)\right|\\
    =& c\cdot \left|\nabla L((U', V'); \hat{\theta})^{\mathrm{T}} \cdot I \cdot\left[I^{\mathrm{T}}\cdot I\right]^{-1} \cdot I^{\mathrm{T}}\cdot\nabla L((U', V'); \hat{\theta})\right|
\end{align*}
where $I =\left[\text{positive-IF}, \text{negative-IF} \right]$. 
%Since cc does not depend on the data we evaluate, we can drop it from the expression.
\begin{align*}
\arg\max_{(x^T, x^I)\in \mathcal{D}} \frac{\delta}{\|\nabla L((U', V'); \hat{\theta})\|}\cdot \left|\nabla L((U', V'); \hat{\theta})^{\mathrm{T}} \cdot I \cdot\left[I^{\mathrm{T}}\cdot I\right]^{-1} \cdot I^{\mathrm{T}}\cdot\nabla L((U', V'); \hat{\theta})\right|
\end{align*}
where $I=\left[\text{positive-IF}((x^T,x^I); \hat{\theta} ) ,\text{negative-IF}((x^T,x^I); \hat{\theta})\right]$.

If $\text{positive-IF}$, $\text{negative-IF}$ are not parallel, the optimization problem in form (\ref{app:optimization_inequl}) is equivalent to
\begin{align*}
\arg\max_{x\in \mathcal{D}} \frac{\delta}{\|\nabla L((U', V'); \hat{\theta})\|} \left|\nabla L((U', V'); \hat{\theta})^{\mathrm{T}}  I \left[I^{\mathrm{T}}\cdot I\right]^{-1}  I^{\mathrm{T}}\nabla L((U', V'); \hat{\theta})\right|.
\end{align*}
Because $\delta$ is independent of data, we can drop it and write the above equation as
\begin{align*}
\arg\max_{x\in \mathcal{D}} {\|\nabla L((U', V'); \hat{\theta})\|}^{-1} \left|\nabla L((U', V'); \hat{\theta})^{\mathrm{T}}  I \left[I^{\mathrm{T}}\cdot I\right]^{-1}  I^{\mathrm{T}}\nabla L((U', V'); \hat{\theta})\right|.
\end{align*}

Then, we come to the second case where $\text{positive-IF}\parallel \text{negative-IF}$. We can define a
\begin{align}\label{app:loss:difference_parallel}
    &\left\| \Delta \hat{\theta}_{\epsilon, \zeta} (x)\right\|\\
    =& \|\epsilon \cdot \text{positive-IF}((x^T,x^I);\hat{\theta}) + \left(\zeta -1\right)\text{negative-IF}((x^T,x^I);\hat{\theta})\|\\
    \triangleq& \|\alpha(\epsilon, \zeta)\cdot \text{positive-IF}((x^T,x^I);\hat{\theta})\|\leq \delta
\end{align}
And the constraint is imposed on $\alpha$ by
\begin{equation*}
    \alpha(\epsilon, \zeta) \leq \frac{\delta}{\|\text{positive-IF}((x^T,x^I);\hat{\theta})\|}
\end{equation*}
Therefore, equation (\ref{loss:difference:proof}) is equivalent to
\begin{align*}
&\max_{\epsilon, \zeta}\left|\nabla L((U', V'); \hat{\theta})^{\mathrm{T}} \cdot \left(\alpha(\epsilon, \zeta) \cdot \text{positive-IF}((x^T,x^I);\hat{\theta}\right)\right|\\
= &\max_{\epsilon, \zeta} \alpha(\epsilon, \zeta) \cdot \left|\nabla L((U', V'); \hat{\theta})^{\mathrm{T}} \cdot \left(\text{positive-IF}((x^T,x^I);\hat{\theta}\right)\right|\\
=& \frac{\delta}{\|\text{positive-IF}((x^T,x^I);\hat{\theta}))\|} \cdot \left|\nabla L((U', V'); \hat{\theta})^{\mathrm{T}} \cdot \text{positive-IF}((x^T,x^I);\hat{\theta})\right|\\
 =&\frac{\delta}{\|\text{negative-IF}((x^T,x^I); \hat{\theta}) \|}\cdot\left|\nabla L((U', V'); \hat{\theta})^{\mathrm{T}} \cdot  \text{negative-IF}((x^T,x^I); \hat{\theta}) \right|
\end{align*}
Because $\delta$ is independent of data, we can drop it and write the above equation as
\begin{align*}
    & {\|\text{positive-IF}((x^T,x^I);\hat{\theta}))\|}^{-1} \cdot \left|\nabla L((U', V'); \hat{\theta})^{\mathrm{T}} \cdot \text{positive-IF}((x^T,x^I);\hat{\theta})\right|\\
 =&{\|\text{negative-IF}((x^T,x^I); \hat{\theta}) \|}^{-1}\cdot\left|\nabla L((U', V'); \hat{\theta})^{\mathrm{T}} \cdot  \text{negative-IF}((x^T,x^I); \hat{\theta}) \right|.
\end{align*}

% \begin{align*}
%     &\argmax_{\epsilon, \lambda}\left|\nabla L((U', V'); \hat{\theta})^{\mathrm{T}} \cdot \left(\epsilon \cdot \text{positive-IF}((x^T,x^I), \text{Seg}) + \left(\lambda -1\right)\text{negative-IF}((x^T,x^I), \text{Seg})\right)\right|\\ 
%     =&\frac{\delta}{\|\text{positive-IF}((x^T,x^I), \text{Seg}) \|}\cdot\left|\nabla L((U', V'); \hat{\theta})^{\mathrm{T}} \cdot  \text{positive-IF}((x^T,x^I), \text{Seg}) \right|\\
%     =&\frac{\delta}{\|\text{negative-IF}((x^T,x^I), \text{Seg}) \|}\cdot\left|\nabla L((U', V'); \hat{\theta})^{\mathrm{T}} \cdot  \text{negative-IF}((x^T,x^I), \text{Seg}) \right|
% \end{align*}
\end{proof}

\section{Additional Experimental Results}
\subsection{Details of Experiment Settings}\label{app:exp_detail}
\noindent {\bf Datasets.}
We employ three datasets for our utility and efficiency evaluation tasks, as well as for the misprediction traceback experiments: \textit{FGVC-Aircraft dataset}~\citep{maji2013finegrainedvisualclassificationaircraft}, \textit{Food101 dataset}~\citep{bossard2014food}, \textit{Flowers102 dataset}~\citep{4756141}. The FGVC-Aircraft dataset comprises 10,000 images of airplanes, each annotated with the model and bounding box of the dominant aircraft depicted. The Food-101 dataset, publicly available for food image recognition, includes 101 food categories, with each category containing 1,000 images. The images feature food photographs captured from various angles and under different lighting conditions. The Flowers-102 dataset consists of 102 classes of flowers native to the United Kingdom, with each class containing between 40 and 258 images. 
We use \textit{Cifar-10 dataset}~\citep{Krizhevsky2009Learning} for the misalignment detection tasks. 

\noindent {\bf Implementation Details.} Our experiments utilized an Nvidia V100-32G GPU and 10 CPU cores with 64 GB memory. For all the following tasks, we employ the CLIP model 'ViT-B/16' and use LoRA few-shot learning. 

For utility evaluation, when testing our method on a random sample-removing task, 10%10\% samples are randomly removed. For valuable (harmful) samples, we remove 10%10\% of the valuable (harmful) data identified by ECIF. Each removal is repeated for 33 times with different seeds.

For the experiment of \textit{Identifying influential data for fine-tuning}, we first calculate the task-related IS for every individual sample and collect valuable data with positive IS, then choose to remove 00-30%30\% of these. We conduct the experiments 33 times for each removal with different seeds. The experiment setting for \textit{harmful data removal} is similar. Differently, we select harmful data with negative IS. The experiments are conducted on Food101, Flowers102, FGVC-Aircraft, and DTD datasets, and the remove ratio ranges 00 to 90%90\%. For each removal, we conducted the experiments 3 times with different seeds. 

The \textit{multiple samples removal} experiments are conducted on Food101, Flowers102, FGVC-Aircraft, and DTD datasets, with removal ratios from $1\%$ to $7\%$, respectively. 

For the \textit{misprediction trace back} task, we conduct experiments on Food101, Flowers102, FGVC-Aircraft, and DTD datasets. We first choose a mispredicted test sample as the target in algorithm \ref{alg:task-related-self}, then calculate the relative IS for each individual sample in the training dataset. Noting the relative IS is always positive. We visualize training samples with top-$10$ relative IS.

For the \textit{misalignment detection} tasks, Cifar-10 and imagenette (smaller version of imageNet) datasets are used. We also applied standard data augmentation techniques on the training set,i.e., random cropping and random flipping. The model is optimized with Adam with weight decay (5e-1), and $\beta$ is set to 0.9. A dropout ratio of 0.25 is used. The training iterations are set to 30, with a learning rate of 2e-4 and a batch size of 16. The rank of the low-rank matrices of LoRA is set to 2. We trained the model on a poisoned version of the dataset (20\% / 30\% of the data samples are mislabeled). Then, we compute the influence score IS of all the training samples on the mispredicted test samples. At the end, we visualize the training samples that have the highest positive IS score.

\subsection{Baseline Method Results}\label{app:baseline}
We conducted experiments comparing ECIF with other data evaluation methods (IF-EKFAC~\citep{grosse2023studyinglargelanguagemodel}, TARK~\citep{park2023trakattributingmodelbehavior}, and TracIN~\citep{NEURIPS2020_e6385d39}) by removing the least contributive 10\% of training data and retraining the model As shown in figure \ref{app:baselines}, ECIF demonstrated superior performance while maintaining high computational efficiency. This represents a significant improvement over traditional methods like TracIn. 

\begin{figure}
    \centering
    \includegraphics[width=0.6\linewidth]{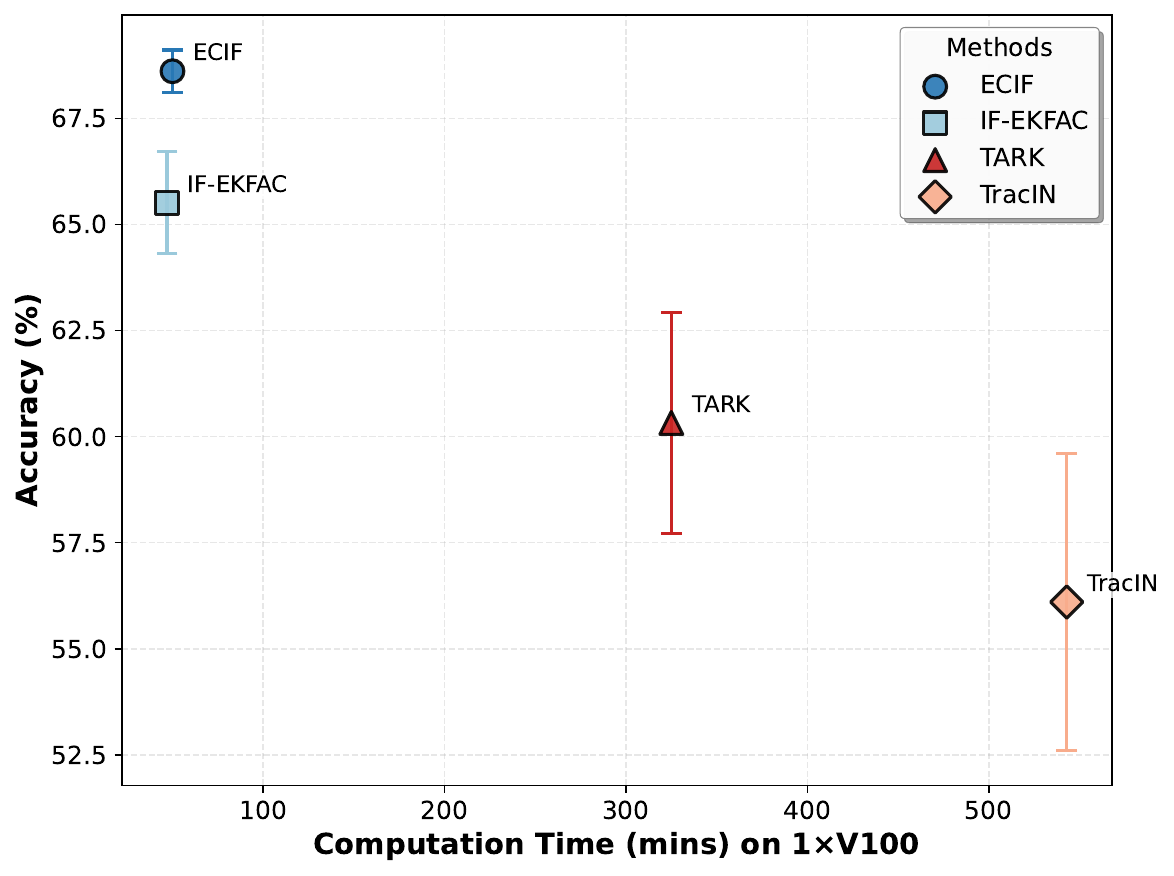}
    \caption{We compare the IF-EKFAC, TARK, TracIN with our ECIF on Flower102 dataset.}
    \label{app:baselines}
\end{figure}

% \subsection{Timing Method in Table ?????????\ref{tab:results}}
% \input{image/tab:total_time}
% We present another version of the runtime calculation method for Table ?????????\ref{tab:results}. Note that the timing
% method in Table ?????????\ref{tab:results} is based on the average time taken for each data point used during the update
% process. In this version, we present another timing method, where RT represents the total runtime for the updating.

\subsection{Extending Utility and Efficiency Evaluation to Larger Dataset}\label{app:exp_larger_ds}
We use \textit{Cifar-100 dataset}~\citep{krizhevsky2009learningML} for our utility and efficiency evaluation tasks.
\begin{table}[t]
\centering
\caption{Comparison of different removal and update strategies on CIFAR-100.}
\label{tab:cifar100}
\begin{tabular}{lccc}
\toprule
\textbf{Removal Type} & \textbf{Method}           & \textbf{Accuracy (\%)}      & \textbf{Time (s)}         \\
\midrule
Random                & Retrain                   & $73.50 \pm 0.35$           & $12.56 \pm 0.37$         \\
                      & IF Update                 & $73.00 \pm 0.20$           & $7.40 \pm 0.11$          \\
\midrule
Positive              & Retrain                   & $73.50 \pm 0.41$           & $8.02 \pm 0.15$          \\
                      & IF Update                 & $72.92 \pm 0.31$           & $2.70 \pm 0.17$          \\
\midrule
Negative              & Retrain                   & $72.83 \pm 0.12$           & $7.92 \pm 0.01$          \\
                      & IF Update                 & $73.00 \pm 0.20$           & $2.26 \pm 0.19$          \\
\bottomrule
\end{tabular}
\end{table}

Results in table 4 \ref{tab:cifar100} highlight the performance of various removal and update strategies on CIFAR-100. The results demonstrate the effectiveness of IF Update (influence function) compared to traditional retraining in terms of both accuracy and computational efficiency. Across all removal types (Random, Positive, and Negative), IF Update consistently achieves comparable or higher accuracy while significantly reducing runtime.

For Random data removal, IF Update improves accuracy from 66.67\% $\pm$ 2.36\% (Retrain) to 73.33\% $\pm$ 1.18\% and nearly halves the runtime, decreasing from 6.87 ± 0.19 seconds to 3.85 ± 0.13 seconds. Similarly, under Positive removal, IF Update achieves an accuracy boost from 68.33\% $\pm$ 1.18\% (Retrain) to 72.50\% $\pm$ 2.04\%, with runtime reduced from 3.01 ± 0.05 seconds to 0.97 $\pm$ 0.09 seconds. Lastly, for Negative removal, while the retrained model yields a slightly higher accuracy (73.33\% $\pm$ 2.36\%) compared to IF Update (70.83\% $\pm$ 1.18\%), IF Update achieves comparable performance with a runtime of 3.49 $\pm$ 4.16 seconds, closely matching retraining (3.00 $\pm$ 0.04 seconds). These results validate the utility of IF Update as a computationally efficient alternative to retraining, achieving near-equivalent or superior accuracy across varied removal scenarios on CIFAR-100.

\subsection{Extending Harmful Data Removal to Real-World Noisy Dataset}
We use \textit{ANIMAL-10N dataset}~\citep{song2019selfie} for our harmful data removal tasks. It's a real world noisy dataset, containing five pairs of "confusing" animals: {(cat, lynx), (jaguar, cheetah), (wolf, coyote), (chimpanzee, orangutan), (hamster, guinea pig)}, where two animals in each pair look very similar. Overall, the proportion of incorrect labels was 6.44\%.

\begin{figure}
    \centering
    \includegraphics[width=0.6\linewidth]{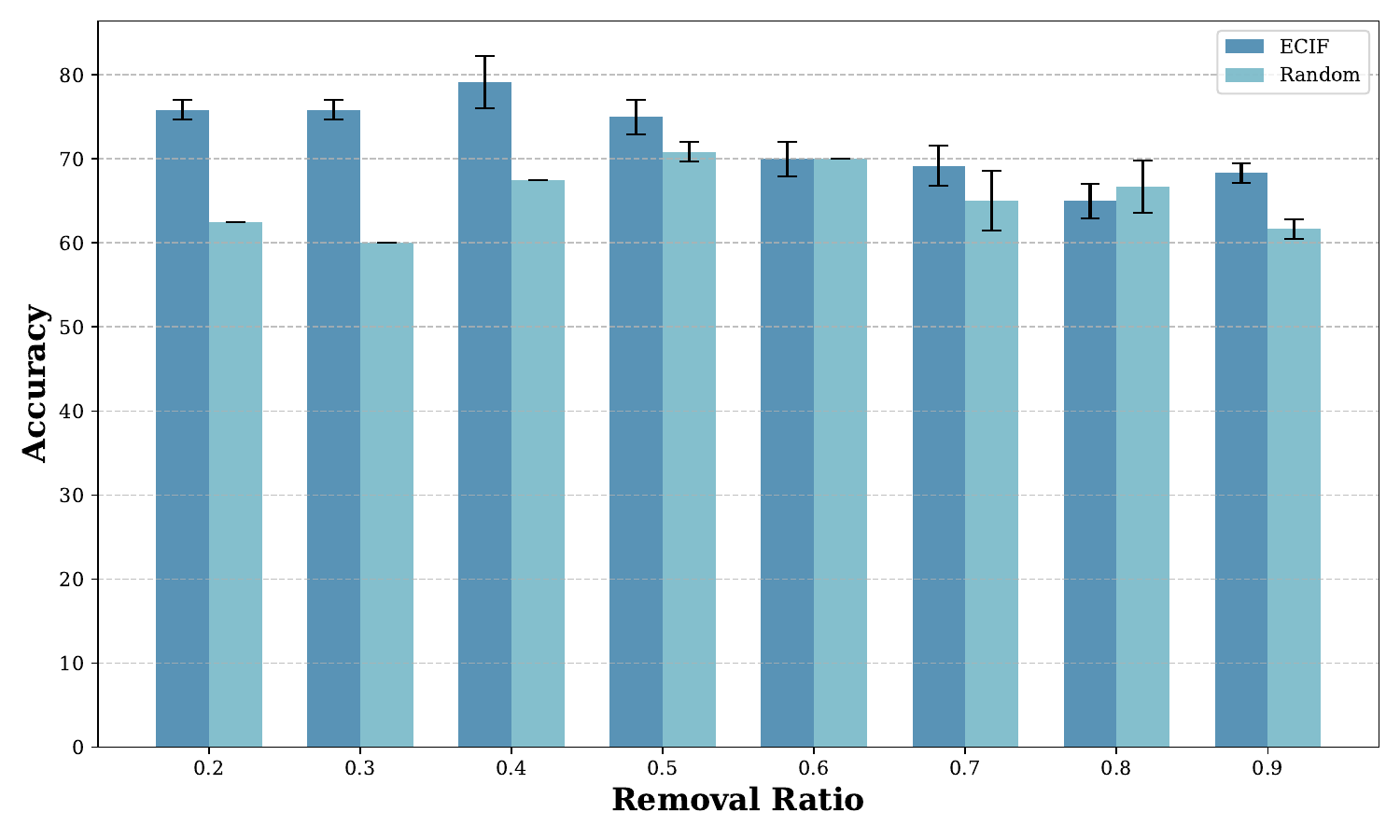}
    \caption{Harmful Data Removal on ANIMAL-10N}
    \label{exp:noisy_removal}
\end{figure}

The harmful removal task on ANIMAL-10N is presented in table \ref{exp:noisy_removal}.We observe that when a portion of harmful data is removed, the ECIF method significantly outperforms random removal, particularly when the removal proportion is small. Specifically, when less than 40\% of the data is removed, ECIF achieves an accuracy improvement of over 10\% compared to random removal. This demonstrates the capability of ECIF to accurately identify harmful samples, thereby substantially enhancing the model's performance.

To provide a method as the reference, we adopt CLIPScore~\citep{hessel2021clipscore}, a basic data evaluation method, as the baseline for MLLM. This method is model-independent and is limited to evaluating data quality rather than assessing the contribution of the data to the model.
\begin{figure}[ht]
\centering
    \subfigure[Harmful Data Removal]{
        \label{fig:app_positive}
        \includegraphics[width=0.45\linewidth]{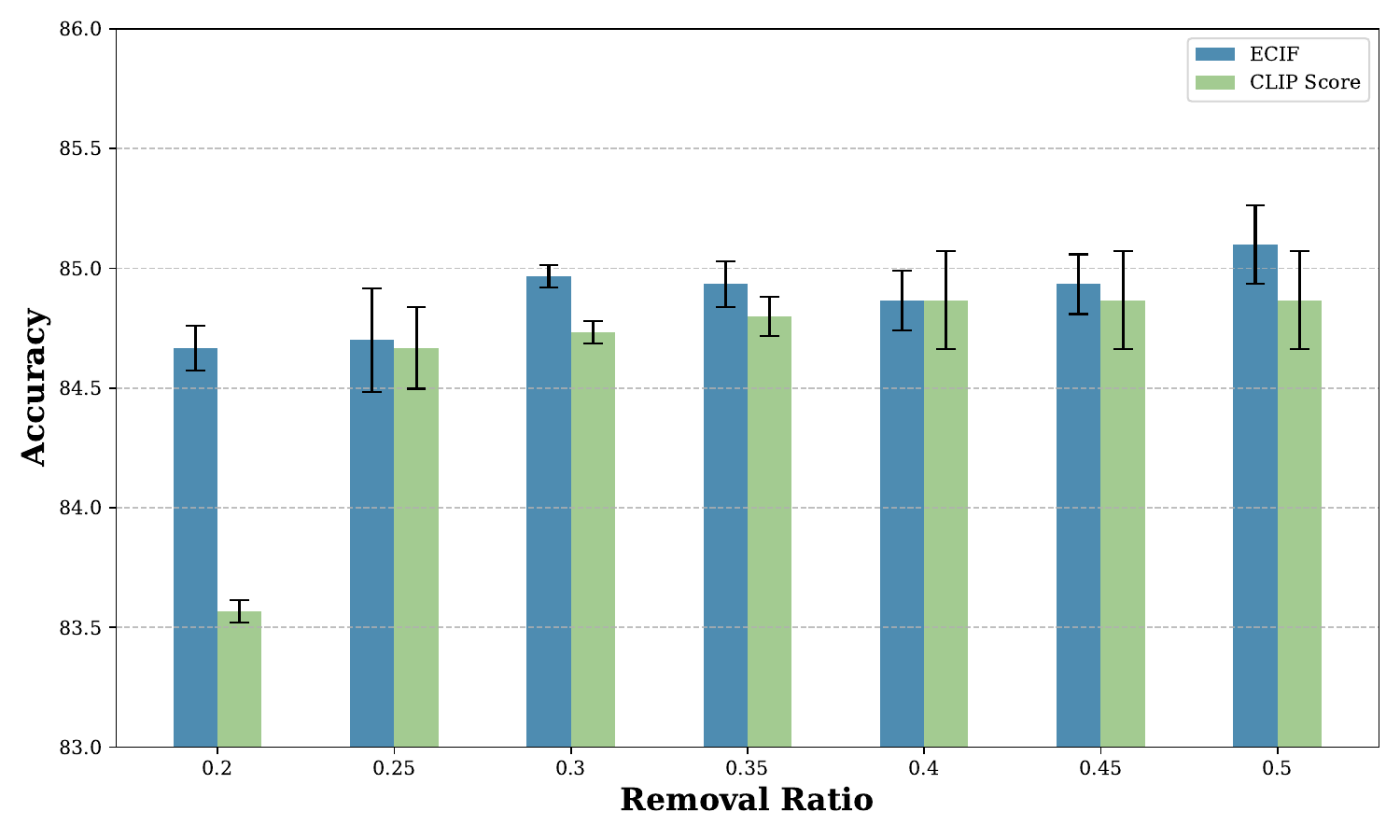}
    }
    \subfigure[Valuable Data Removal]{
        \label{fig:app_negative}
        \includegraphics[width=0.45\linewidth]{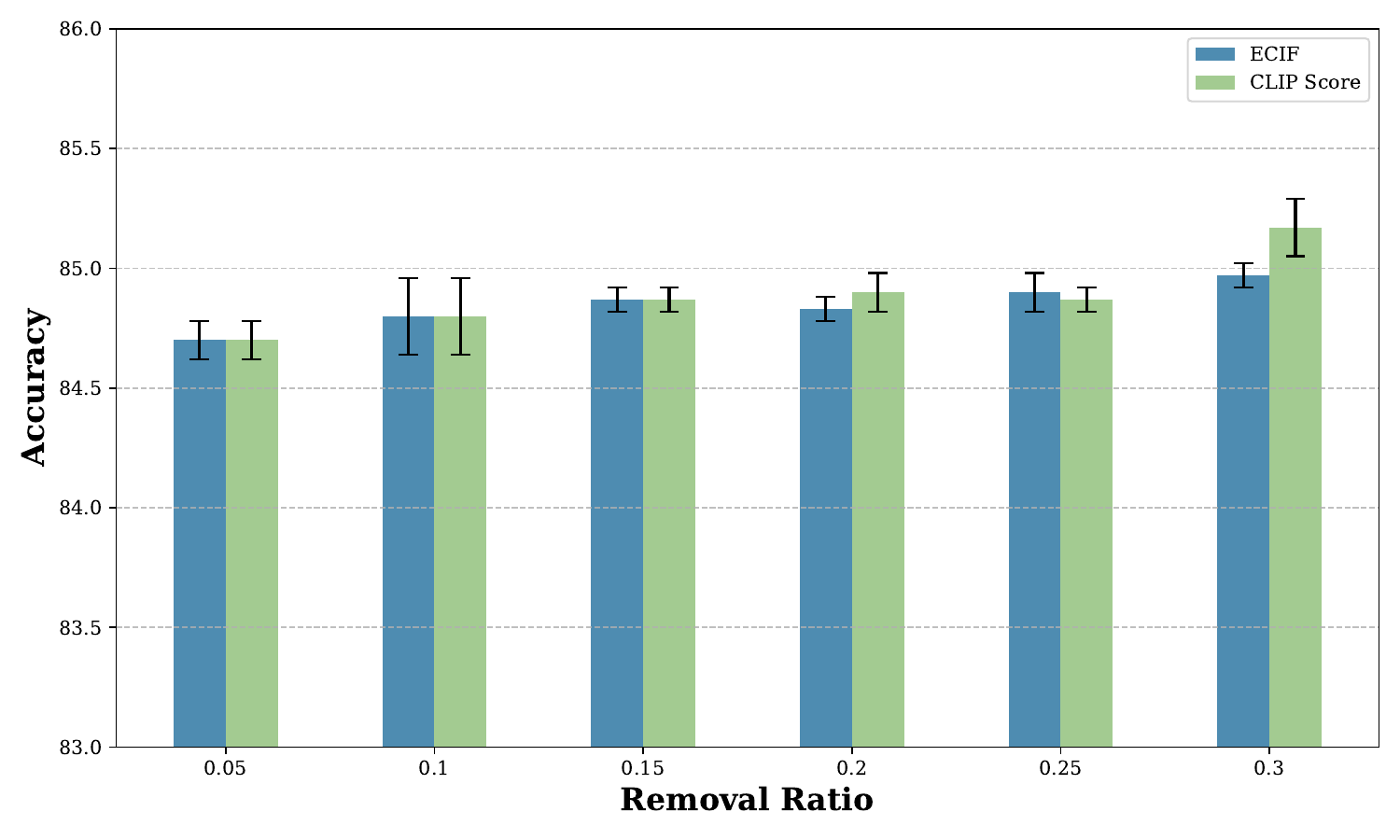}
    }
\caption{Comparison between different methods for data removal on Food101}
\label{fig:app_data_removal}
\end{figure}

In the task of harmful data removal, the ECIF method demonstrates significantly better performance compared to the CLIP Score. For valuable data removal, ECIF performs slightly better than the CLIP Score. This superiority is primarily attributed to ECIF's ability to attribute data based on the relationship between the model and the data, whereas CLIP Score is solely used to evaluate data quality without considering the model's involvement.

\subsection{Evaluating Multiple Samples} \label{sec:exp_multi}

To comprehensively evaluate the data removal capabilities of ECIF in various scenarios, we conducted experiments on the performance when multiple samples need to be removed. Specifically,  we consider the different ratios of samples ($1$-$7\%$) for removal.  As shown in Figure \ref{fig:hyperpara}, we can see the accuracy difference between these two methods is very small (less than 1.5\%) in most cases, except the case of $2\%$ for Food101. These results show the utility of ECIF compared to the ground truth. 
Note that in Table \ref{tab:results}, we have shown that the speed of ECIF is more than two times faster than that of retraining. Thus, ECIF is an editing method that achieves a trade-off between speed and effectiveness.  

% \begin{figure}[ht]
%     \centering
% \includegraphics[width=0.5\linewidth]{figs/ablation/ablation.pdf}
% \caption{Impact of remove ratio on Food101 dataset. \warn{Can be different dataset or different algorithms} \label{fig:hyperpara}}
% \vspace{-7pt}
% \end{figure}
\begin{figure}[ht]
    \centering
\begin{tabular}{ccc}
\includegraphics[width=0.3\linewidth]{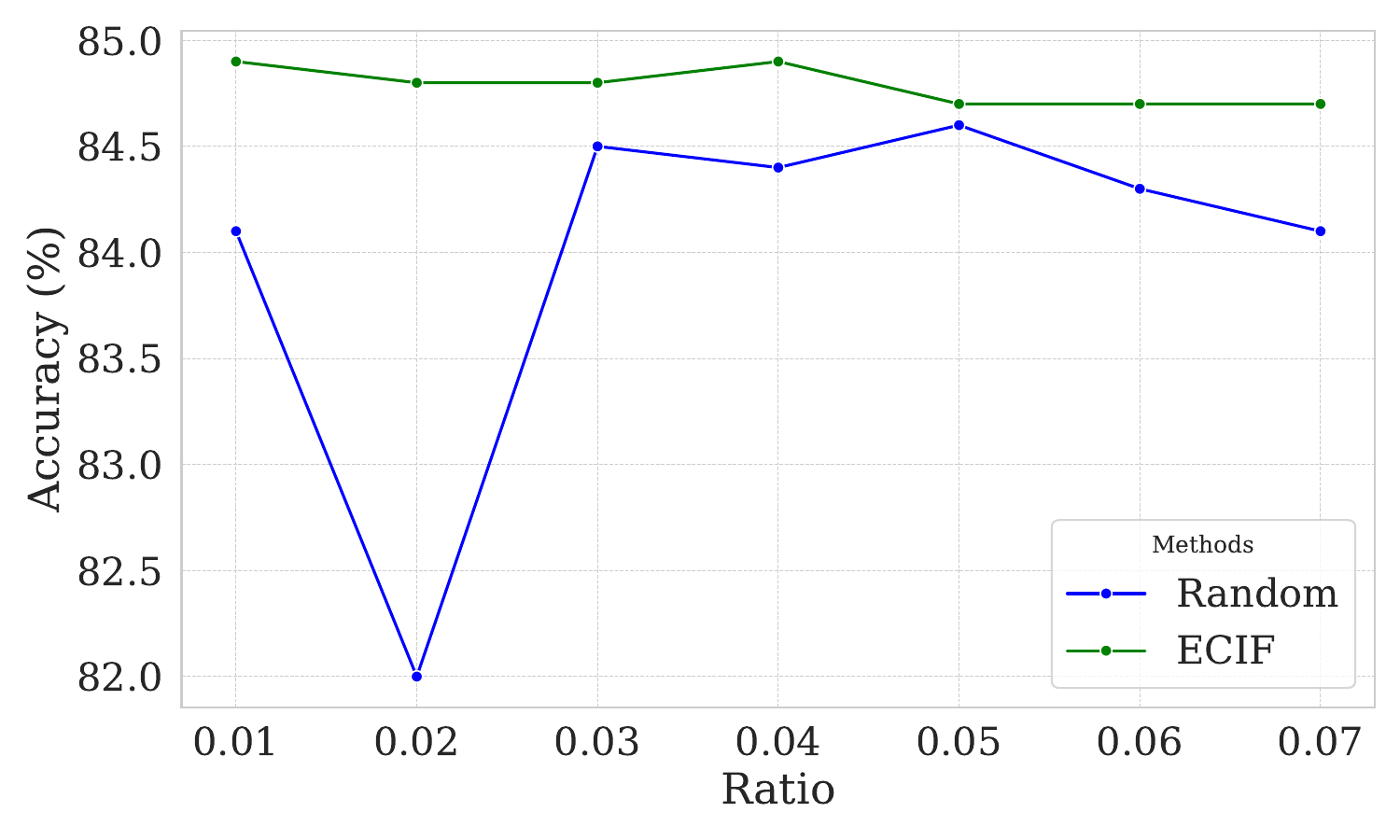}     & 
\includegraphics[width=0.3\linewidth]{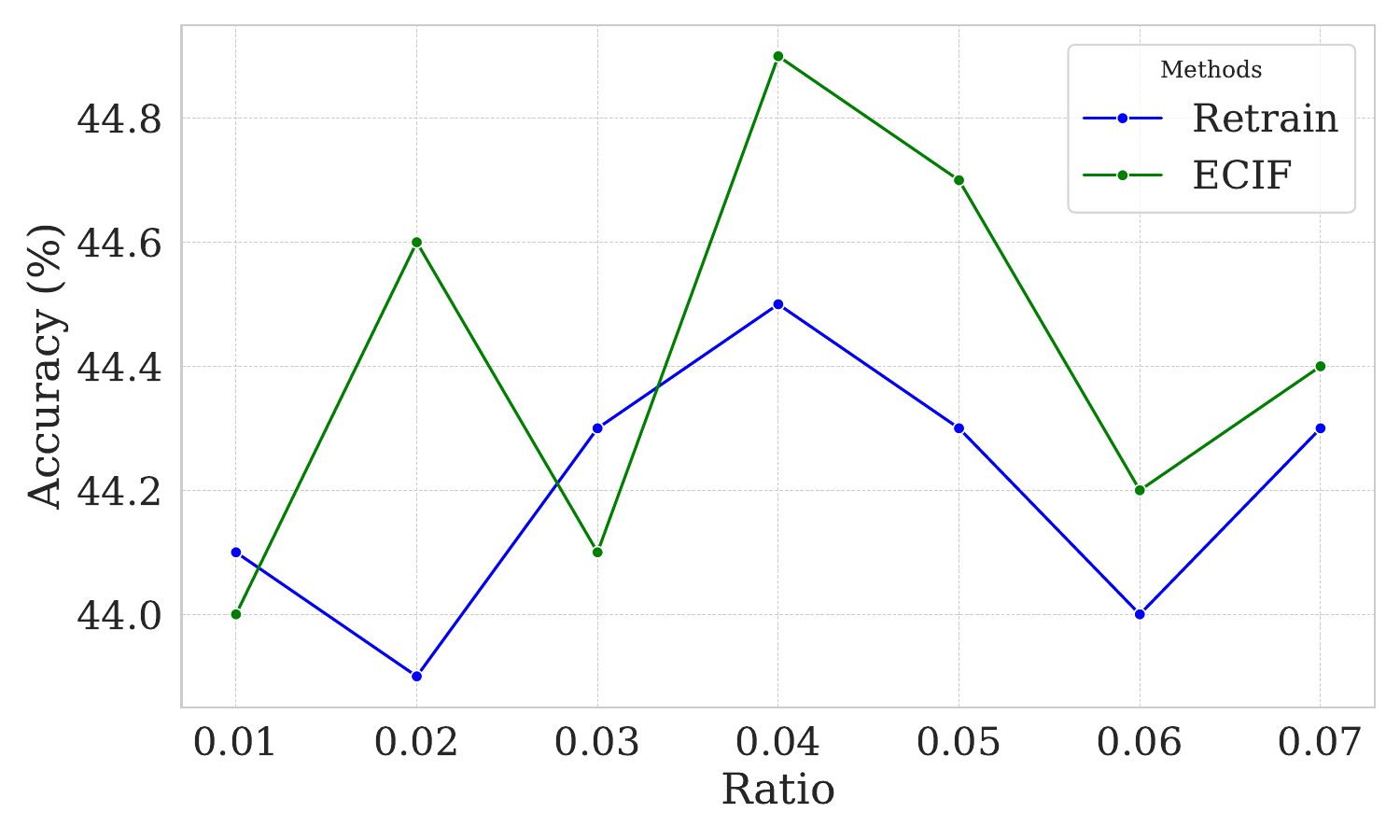} & 
\includegraphics[width=0.3\linewidth]{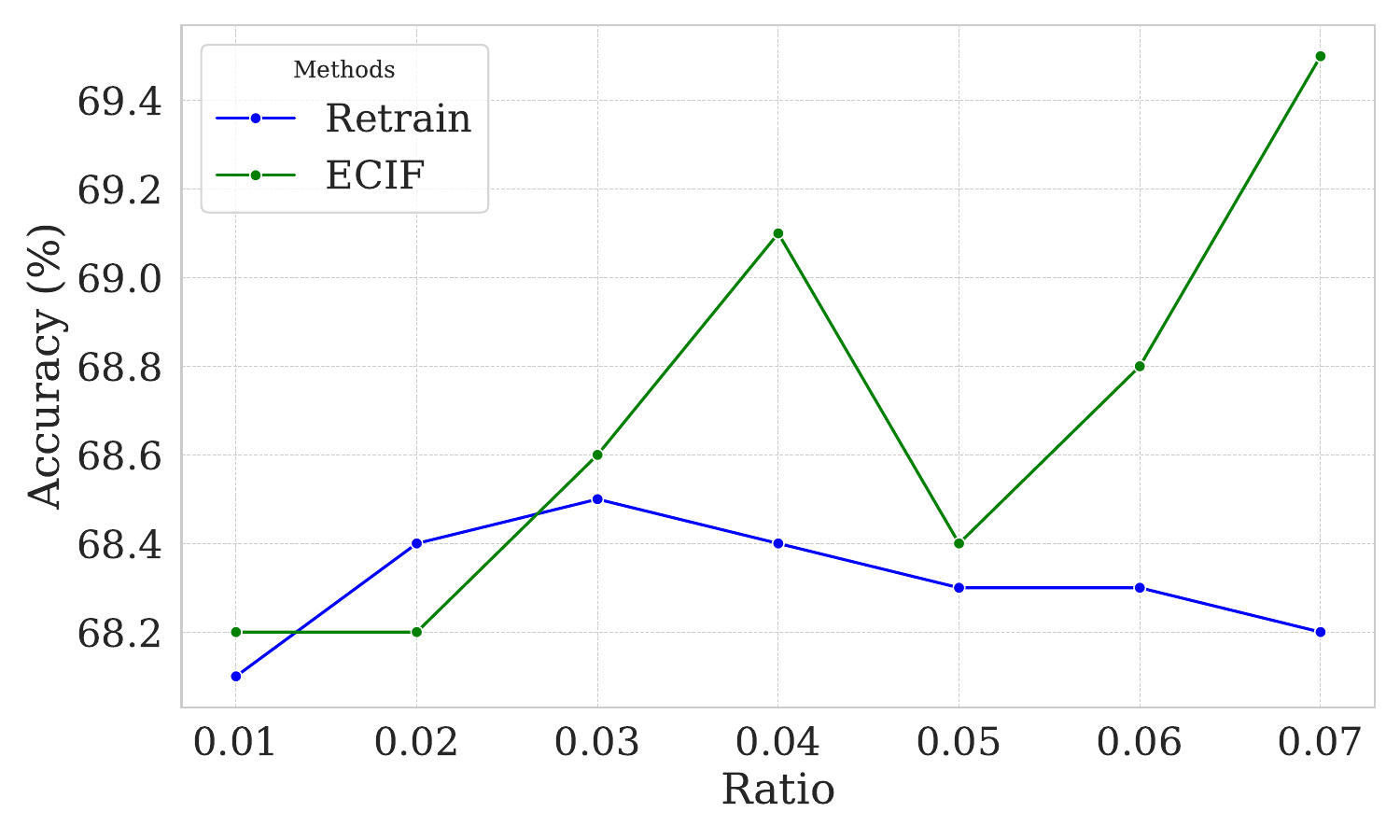} 
\end{tabular}
\caption{Impact of Remove Ratio on Food101, DTD and Flower102 datasets. \label{fig:hyperpara}}
\vspace{-7pt}
\end{figure}
\subsection{Additional Results for ENHANCING FINE-TUNING VIA TASK-RELATED INFLUENCE Score}\label{sec:taskIS_more}

We demonstrate our additional results of using task-related IS to identify harmful data on Flower102 in Figure \ref{exp:taskIS_more}.

%\subsection{Misalignment Detection Settings}
%\textbf{Khouloud:} \textbf{misalignment detection}: We randomly mislabeled 20\% and 30\% of the cifar-10 training dataset samples. LoRA few-shot adaptation with clip model 'ViT-B/16' is used. The num-shots per class is set to 16. The number of iterations is set to 30. the learning rate is set to 2e−42e-4. the batch size is set 16.

\begin{figure}
    \centering
    \includegraphics[width=0.6\linewidth]{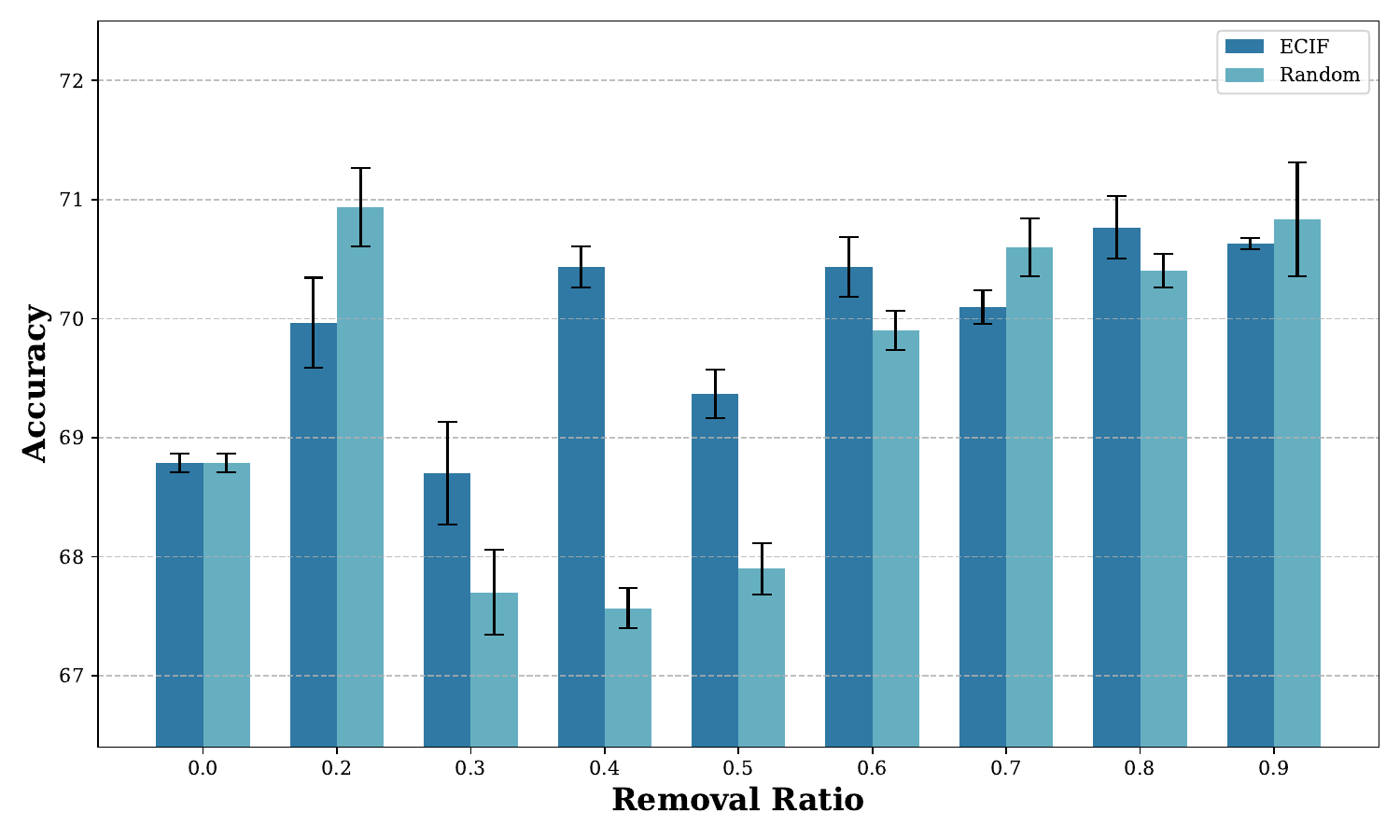}
    \caption{Harmful Data Removal on Flower102}
    \label{exp:taskIS_more}
\end{figure}

\iffalse

\begin{figure}[ht]
\centering
\begin{subfigure}{0.45\linewidth}
    \centering
    \includegraphics[width=\linewidth]{figs/harmful/dtd_1.pdf}
    \caption{Harmful Data Removal on DTD}
\end{subfigure}
\hfill
\begin{subfigure}{0.45\linewidth}
    \centering
    \includegraphics[width=\linewidth]{figs/harmful/flowers_1.pdf}
    \caption{Harmful Data Removal on Flower102}
\end{subfigure}
\hfill
\begin{subfigure}{0.45\linewidth}
    \centering
    \includegraphics[width=\linewidth]{figs/harmful/fgvc_1.pdf}
    \caption{Harmful Data Removal on FGVC}
\end{subfigure}
\hfill
\begin{subfigure}{0.45\linewidth}
    \centering
    \includegraphics[width=\linewidth]{figs/counterfact/cf_eval_dtd.pdf}
    \caption{Valuable Data Removal on DTD}
\end{subfigure}
\hfill
\begin{subfigure}{0.45\linewidth}
    \centering
    \includegraphics[width=\linewidth]{figs/counterfact/cf_eval_fgvc.pdf}
    \caption{Valuable Data Removal on FGVC}
\end{subfigure}
\hfill
\begin{subfigure}{0.45\linewidth}
    \centering
    \includegraphics[width=\linewidth]{figs/counterfact/cf_eval_oxford_flowers.pdf}
    \caption{Valuable Data Removal on Flower102}
\end{subfigure}
\caption{Accuracy after removing influential data recognized by task-related IS on different datasets. \label{exp:taskIS_more}}
\end{figure}
\fi

\subsection{Additional Visualization of Misprediction Trace back}\label{sec:exp_add_pred}
We demonstrate our additional visualization results of the mispredicted data tracing in Table \ref{tab:trace_vis1}-\ref{tab:trace_vis3} and Figure \ref{fig:combined_images_part0} -\ref{fig:combined_images_part2}.

\subsection{Additional Visualization of Misalignment Data Detection}\label{sec:exp_add_mis}
We demonstrate our additional results of the Visualization of the misalignment data detection in Figure \ref{app:misalignment_vis} - \ref{fig:misalign_30}.

% \begin{figure*}[t]
% \centering
% \includegraphics[width=0.65\linewidth]{figs/misalign/n20/imagenett.png}
%     \caption{misalignment detection by visualizing the training samples that have the 10 highest IS scores on cifar-10 test set. 20\% of the training samples were mislabeled.}
%   \label{misalignment_vis}
% \end{figure*}

\begin{figure}
    \centering
    \includegraphics[width=0.5\linewidth]{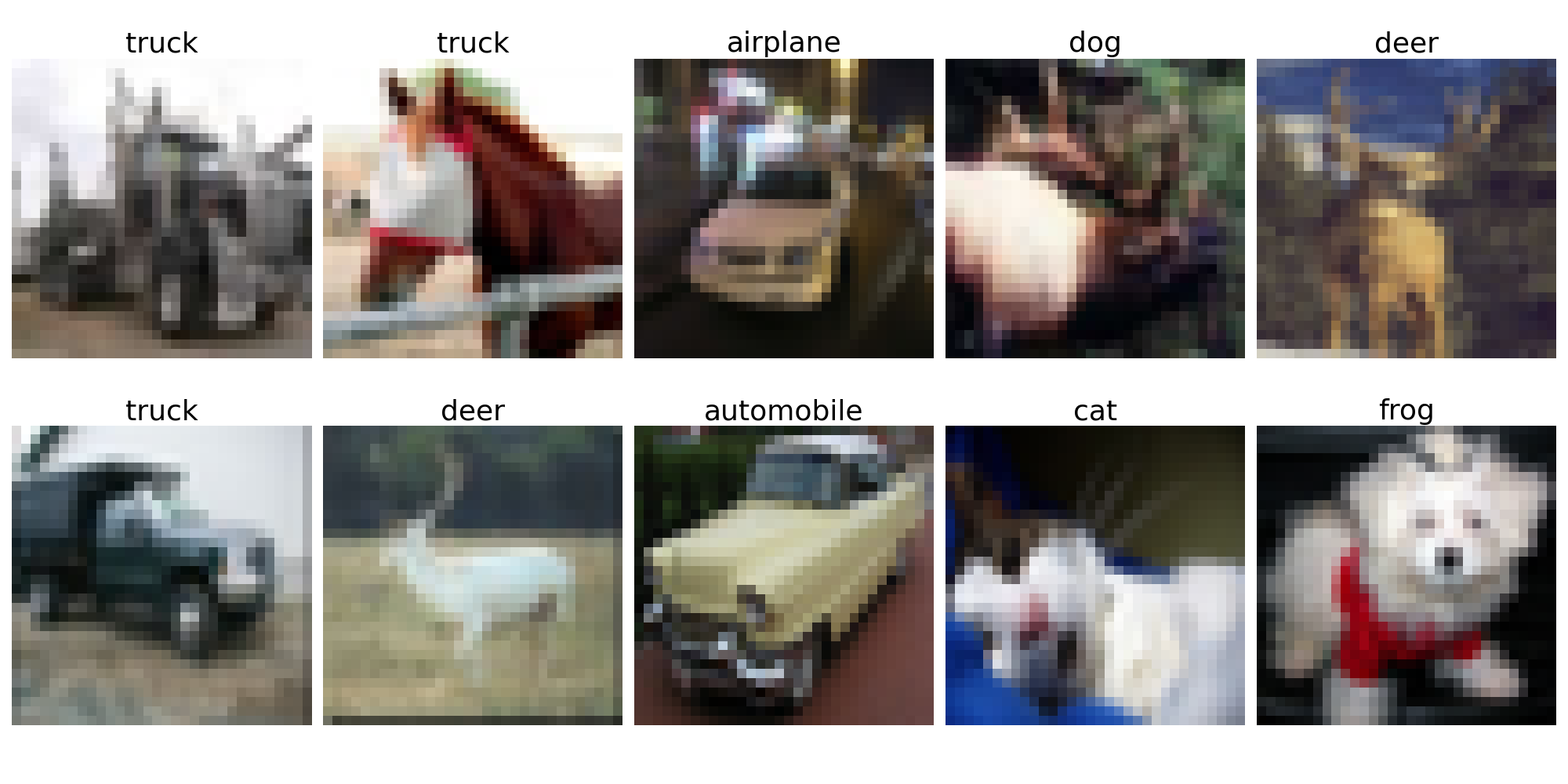}
    \caption{Top-10 misaligned sample pairs in the 20\% mislabeled training data.}
    \label{app:misalignment_vis}
\end{figure}
\begin{figure*}[t]
\centering
\includegraphics[width=0.65\linewidth]{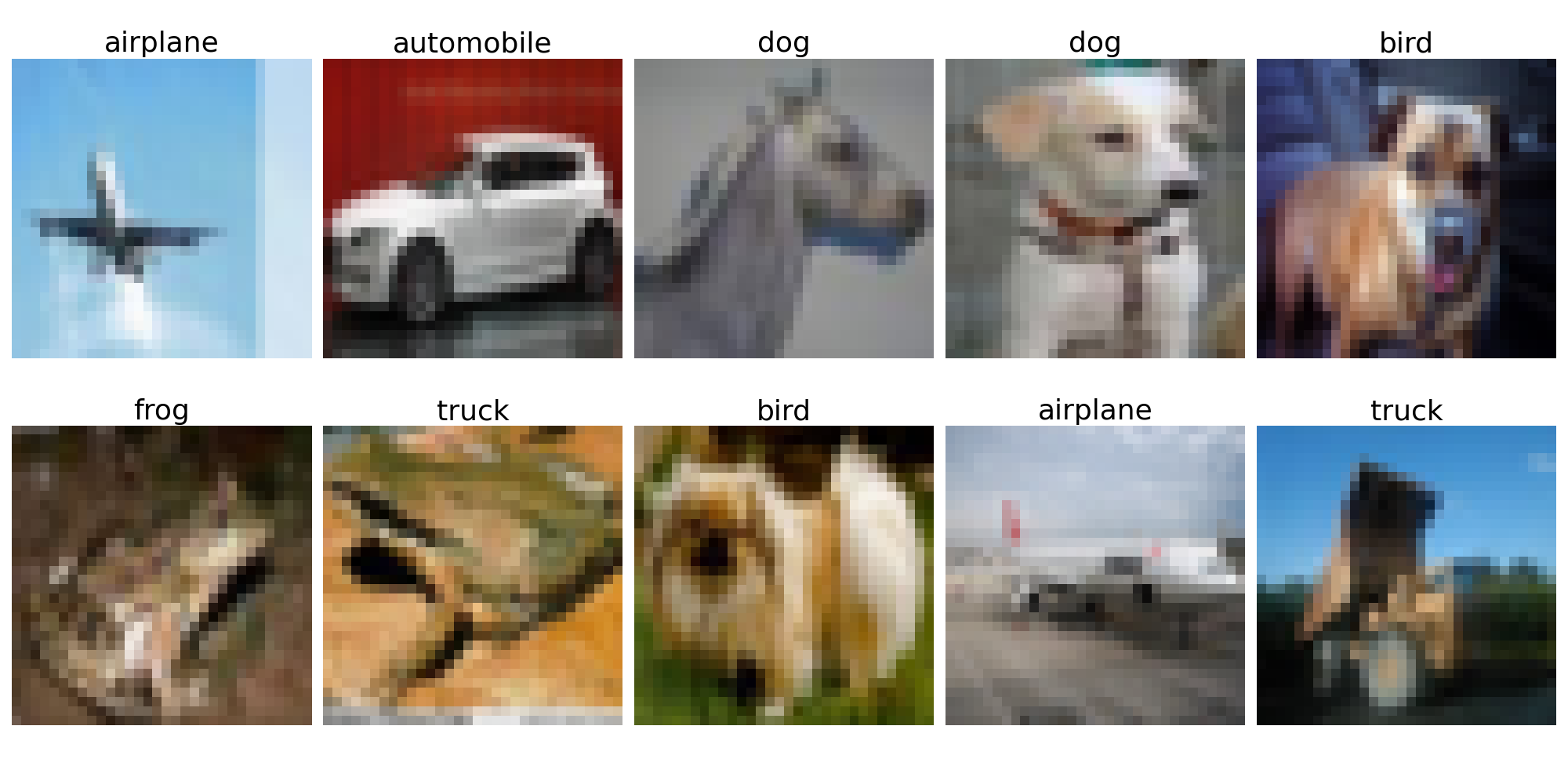}
\caption{Visualization results for misalignment detection. 30\% of the training samples were mislabeled. The figure shows the training samples that have the top-10 highest IS scores on the cifar-10 test set. }
\label{fig:misalign_30}
\end{figure*}

\begin{figure}[ht]
    \centering
    \begin{tabular}{c|c} % Two columns
        \includegraphics[width=0.45\linewidth]{figs/trace/fgvc/737-900.png} & 
        \includegraphics[width=0.45\linewidth]{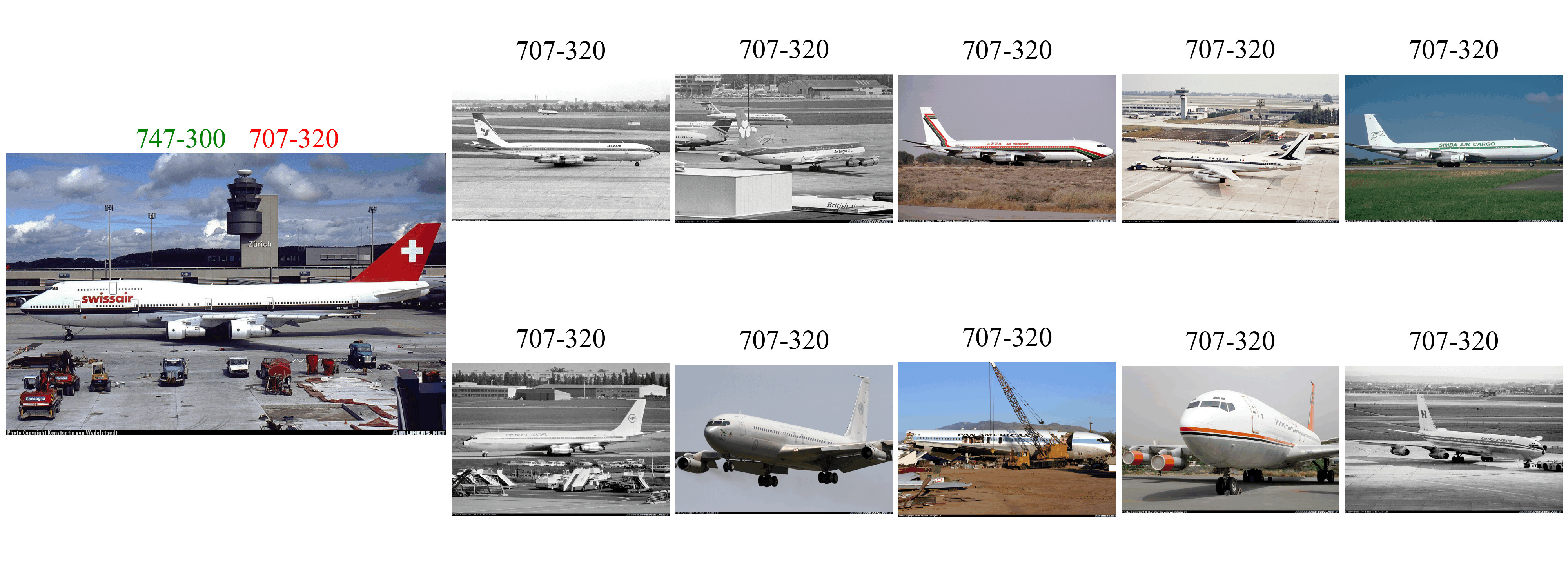} \\ \hline
        \includegraphics[width=0.45\linewidth]{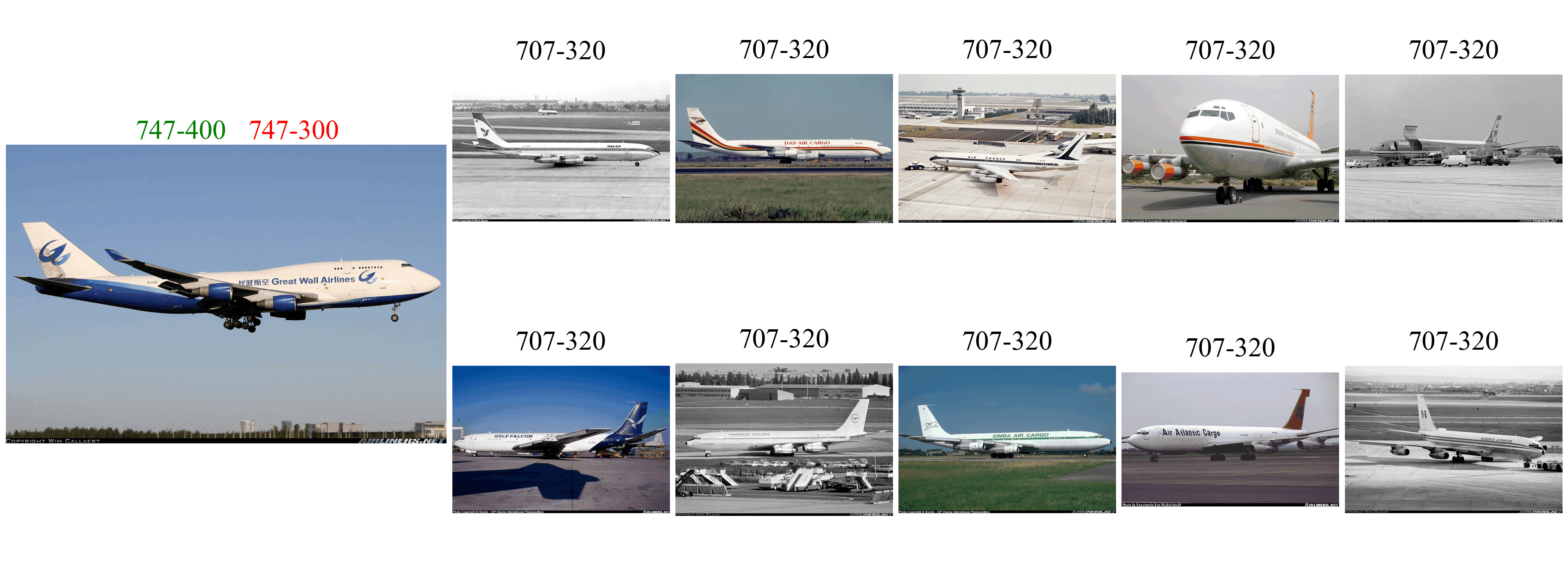} & 
        \includegraphics[width=0.45\linewidth]{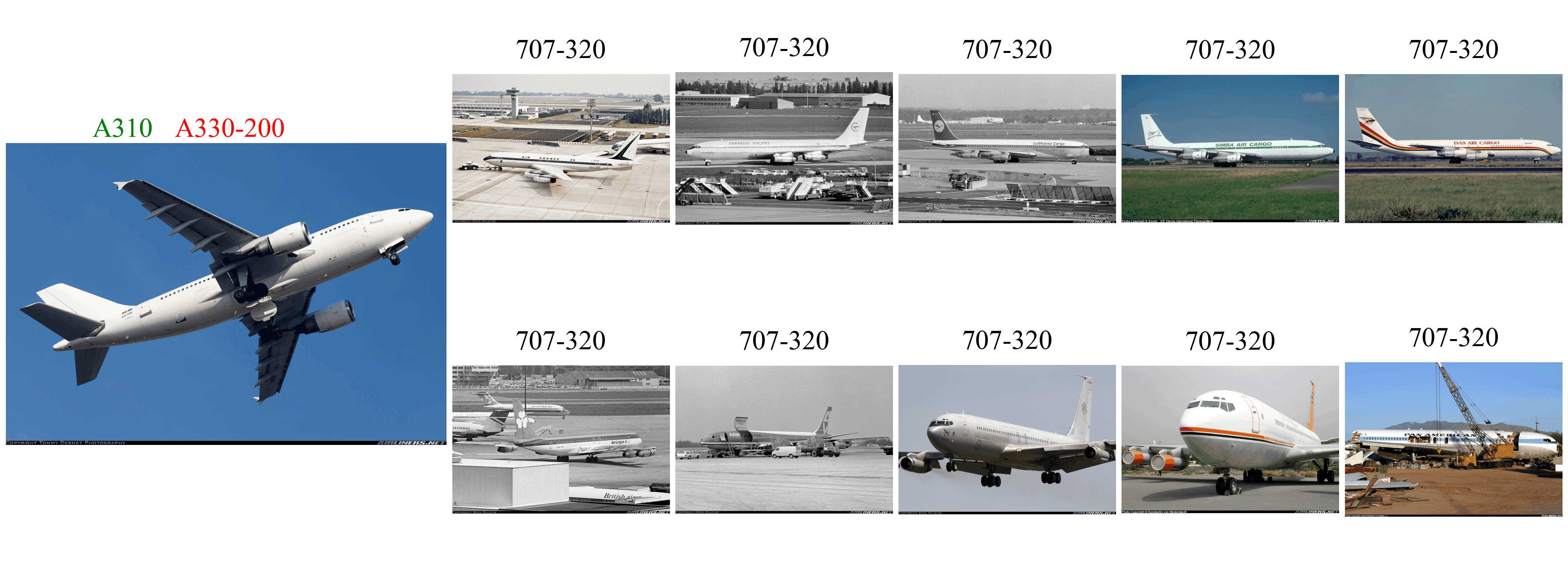} \\ 
        \end{tabular}
    \caption{Top-10 related test data tracing of mispredicted data on FGVC-Aircraft with 30\%  noise data.}
    \label{fig:combined_images_part0}
    \vspace{-7pt}
\end{figure}

\begin{figure}[ht]
    \centering
    \begin{tabular}{c|c} % Two columns
        \includegraphics[width=0.45\linewidth]{figs/trace/food101/seaweed_salad.png} & 
        \includegraphics[width=0.45\linewidth]{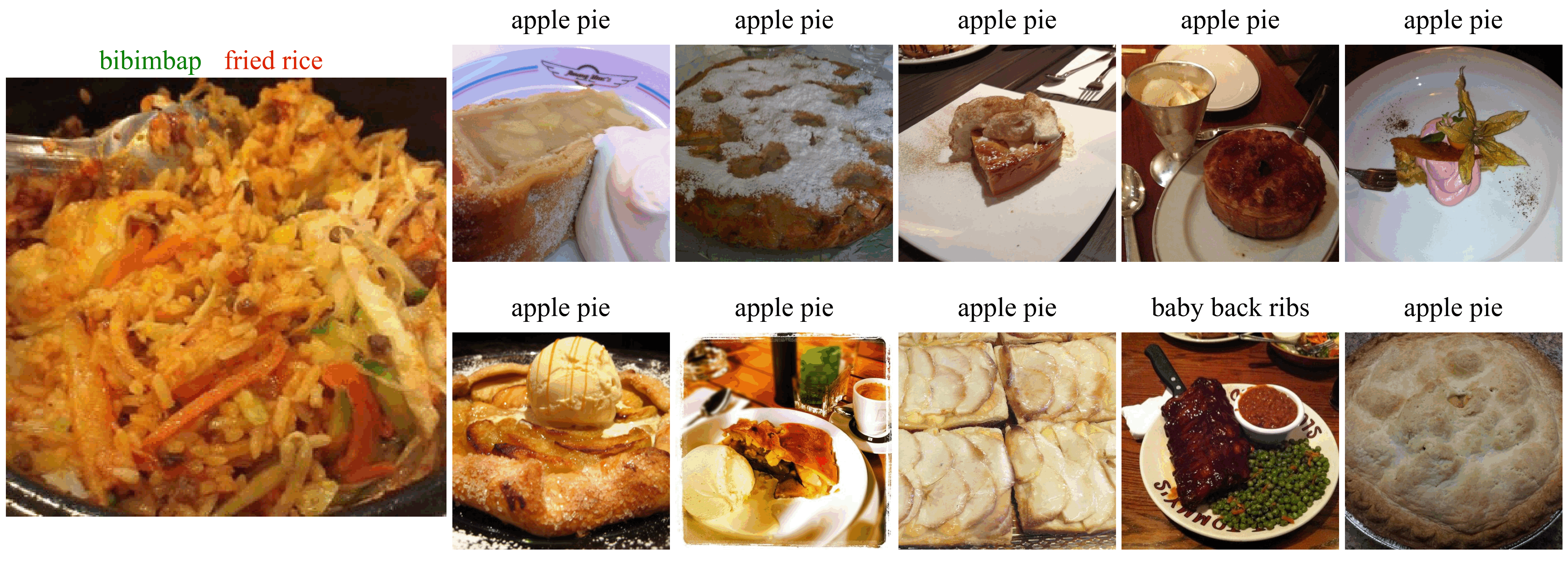} \\ \hline
        \includegraphics[width=0.45\linewidth]{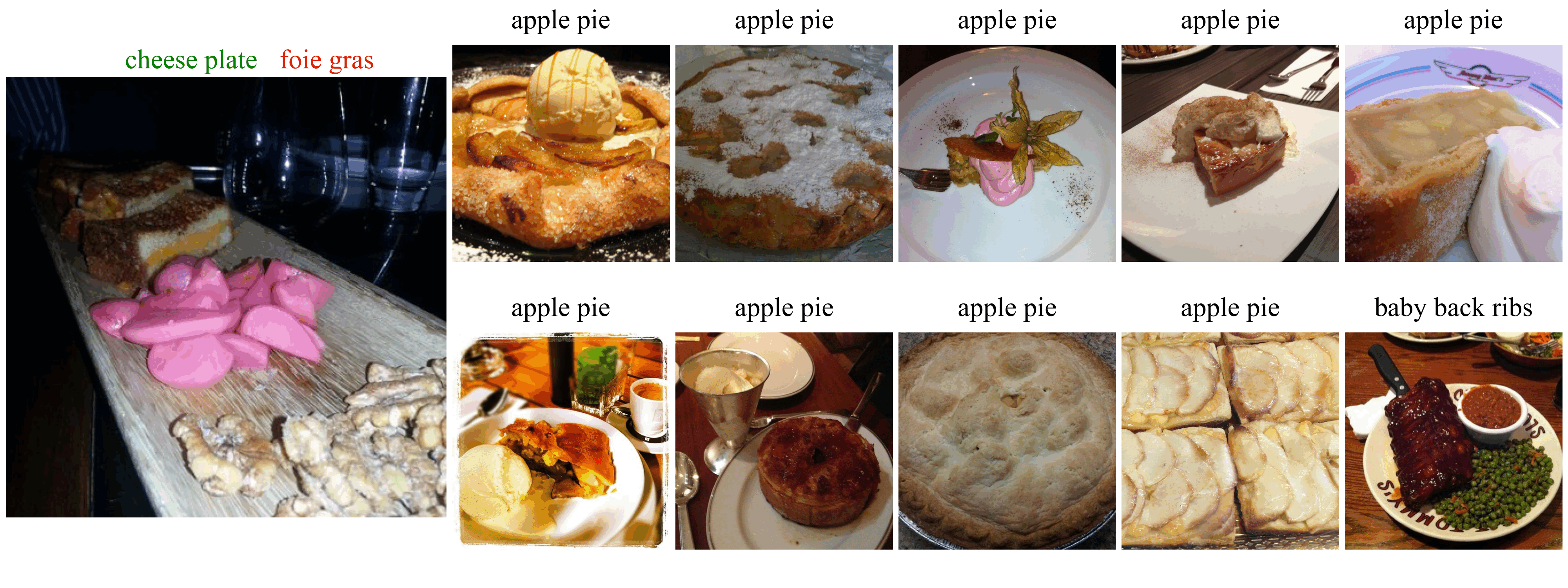} & 
        \includegraphics[width=0.45\linewidth]{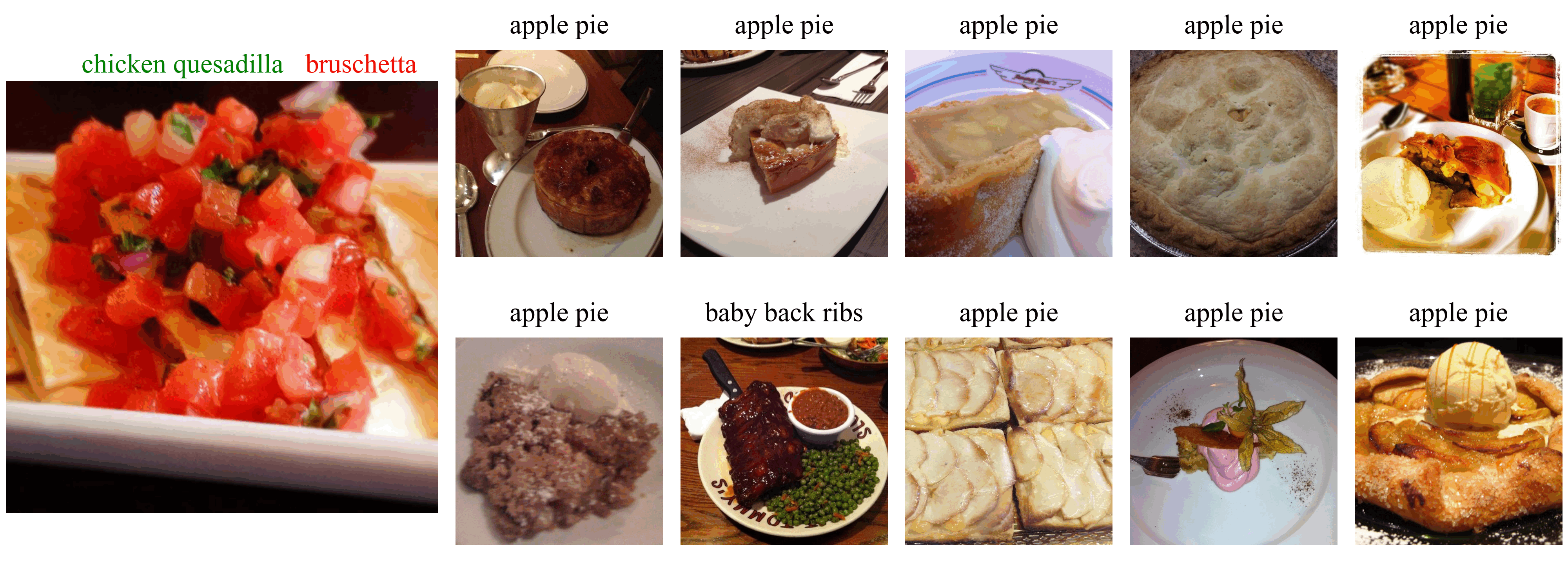} \\ \hline
        \includegraphics[width=0.45\linewidth]{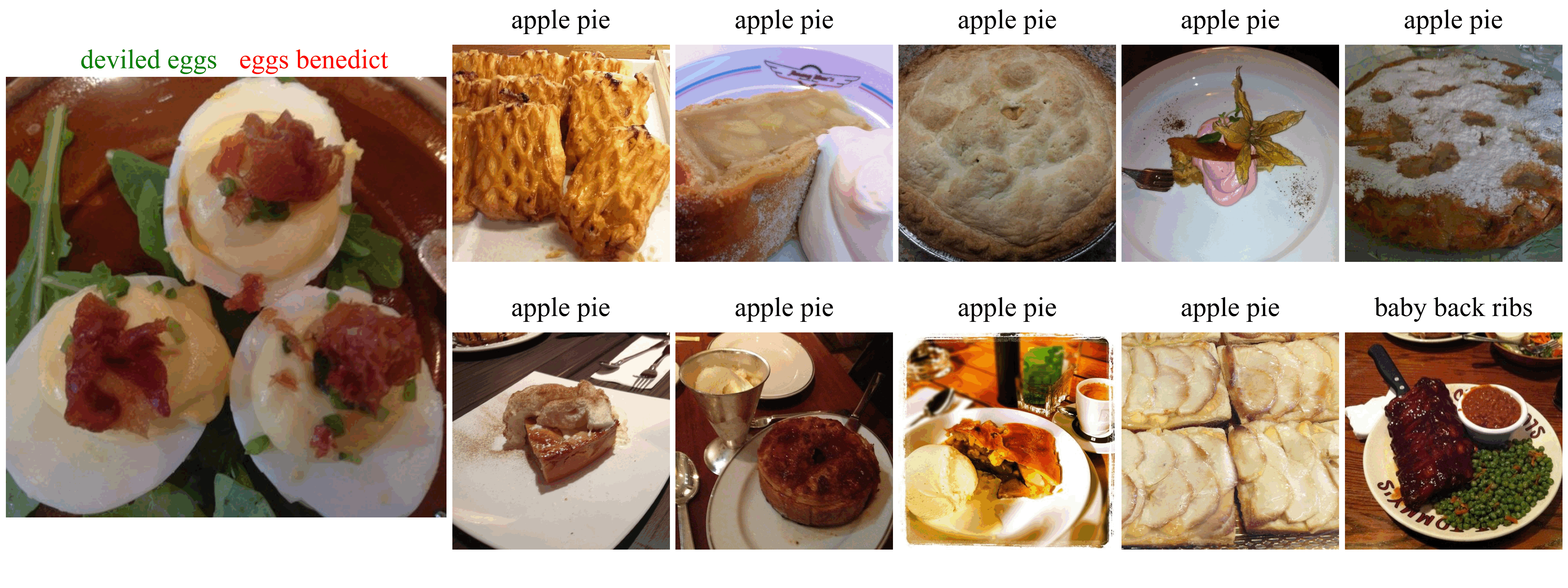} & 
        \includegraphics[width=0.45\linewidth]{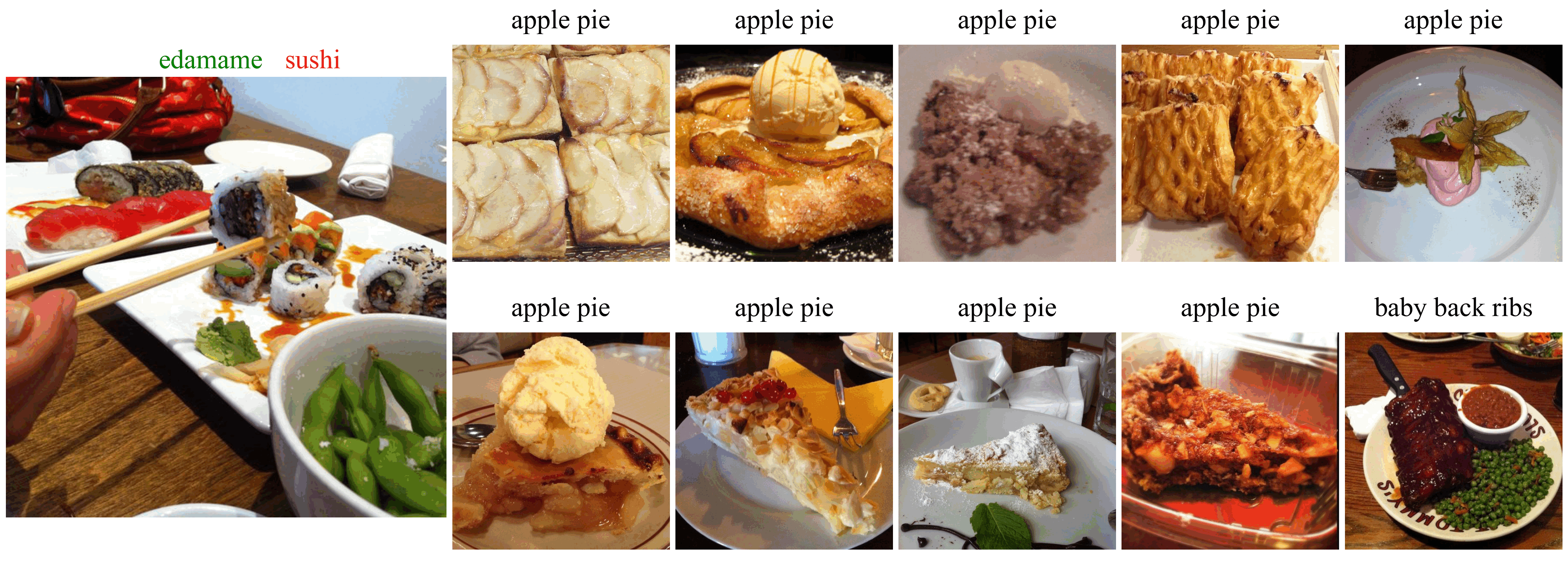} \\
    \end{tabular}
    \caption{Top-10 related test data tracing of mispredicted data on Food-101 with 30\%  noise data.}
    \label{fig:combined_images_part1}
    \vspace{-7pt}
\end{figure}

\begin{figure}[ht]
    \centering
    \begin{tabular}{c|c}
        \includegraphics[width=0.45\linewidth]{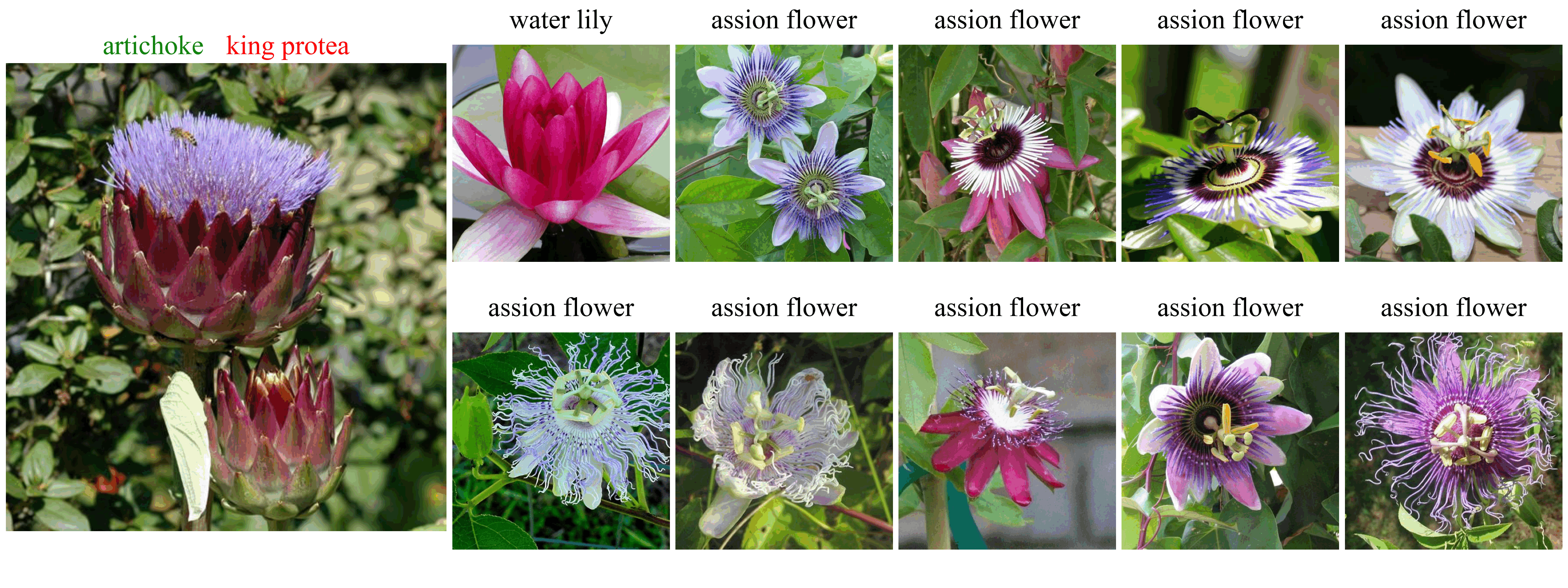} & 
        \includegraphics[width=0.45\linewidth]{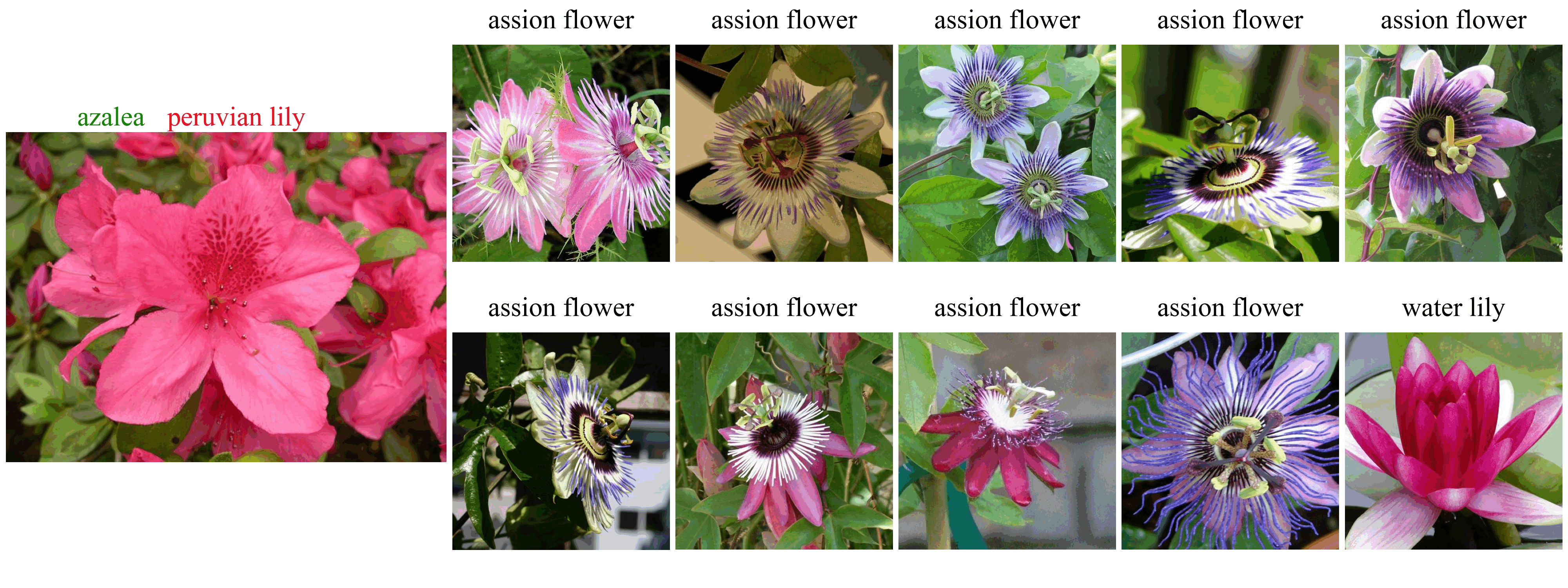} \\ \hline
        \includegraphics[width=0.45\linewidth]{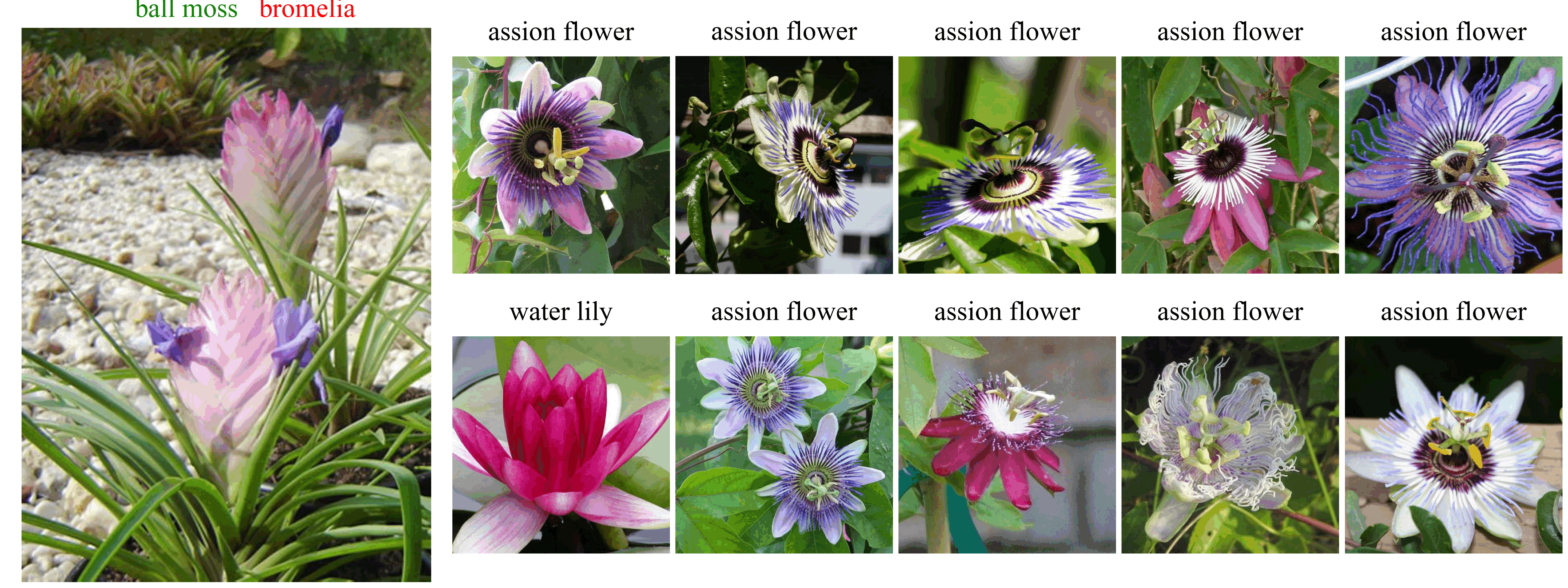} & 
        \includegraphics[width=0.45\linewidth]{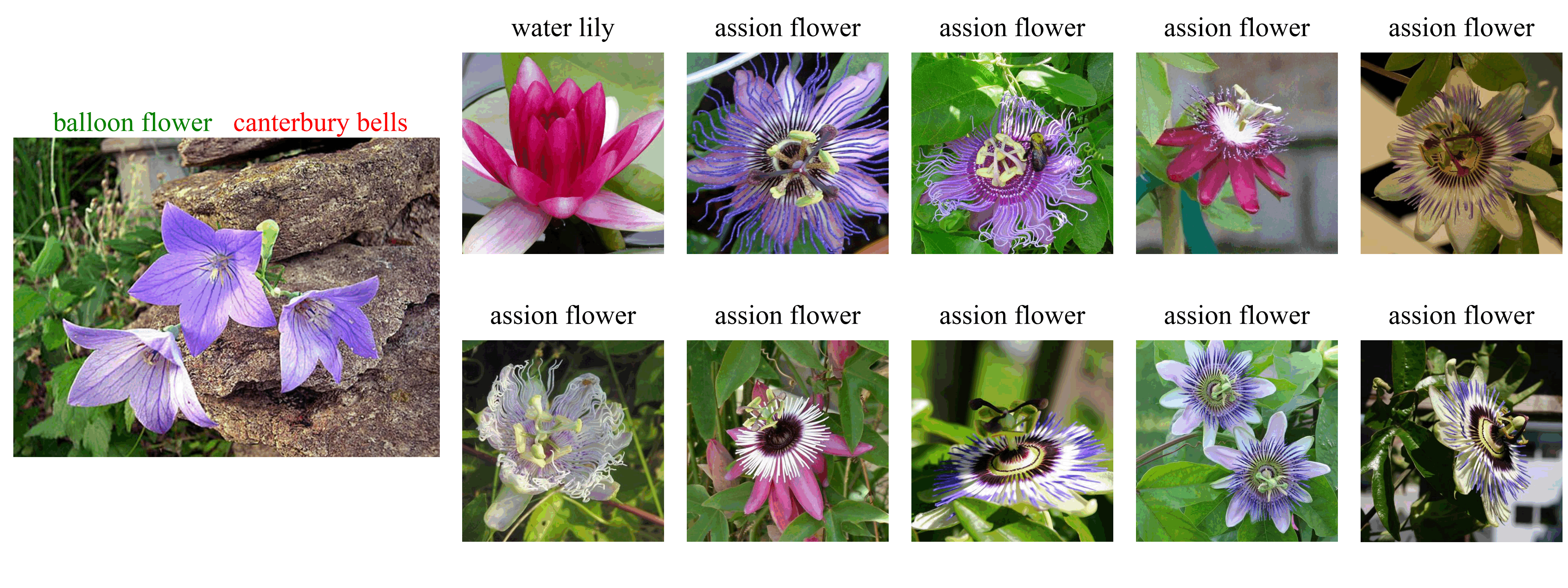} \\ \hline
        \includegraphics[width=0.45\linewidth]{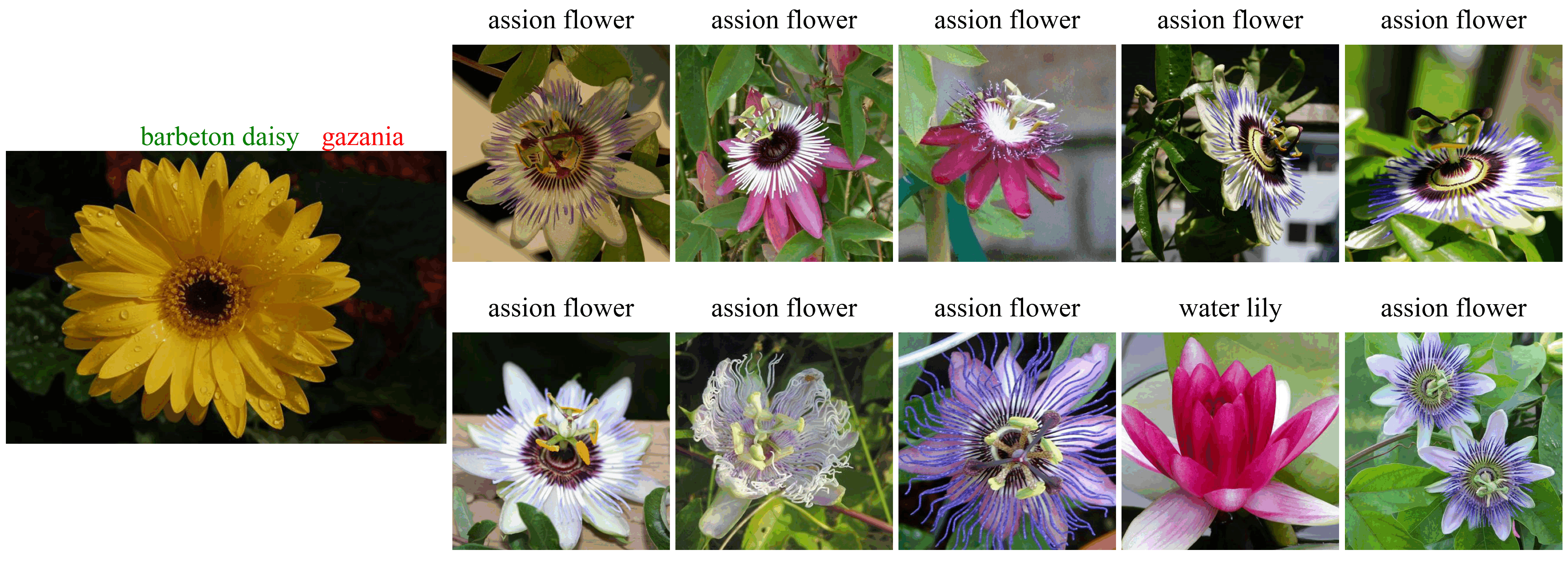} & 
        \includegraphics[width=0.45\linewidth]{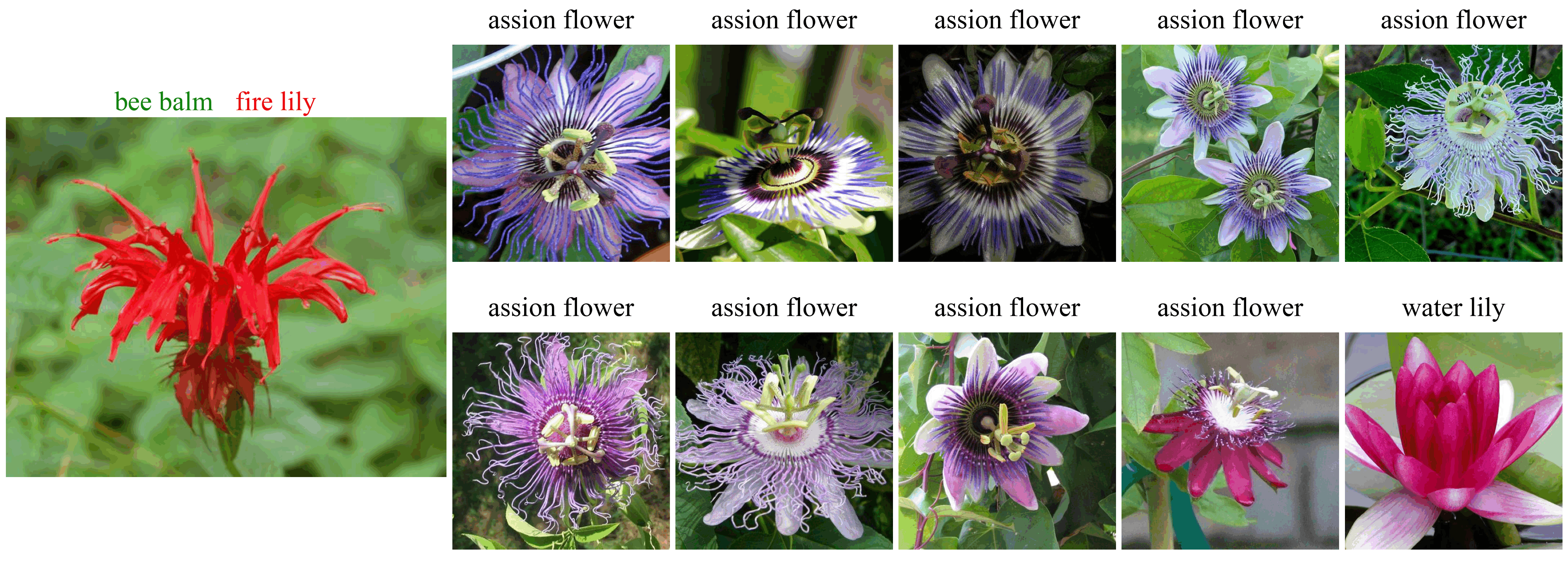} \\ \hline
        \includegraphics[width=0.45\linewidth]{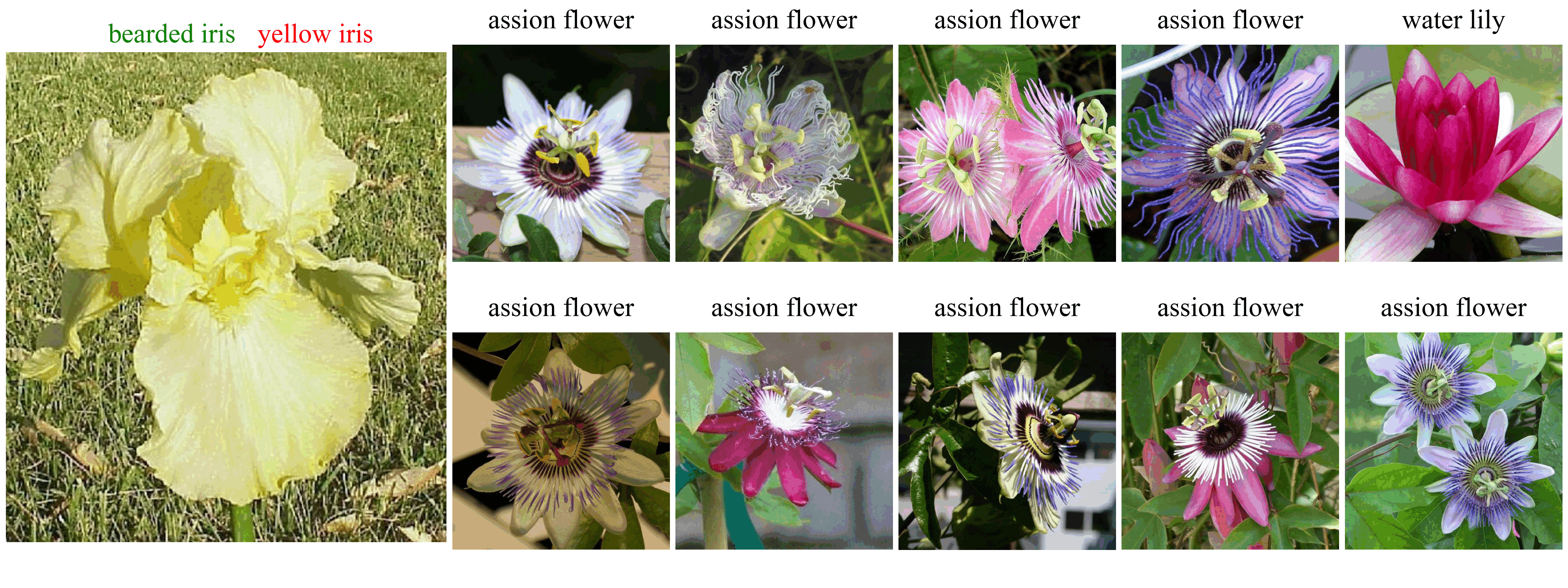} & 
        \includegraphics[width=0.45\linewidth]{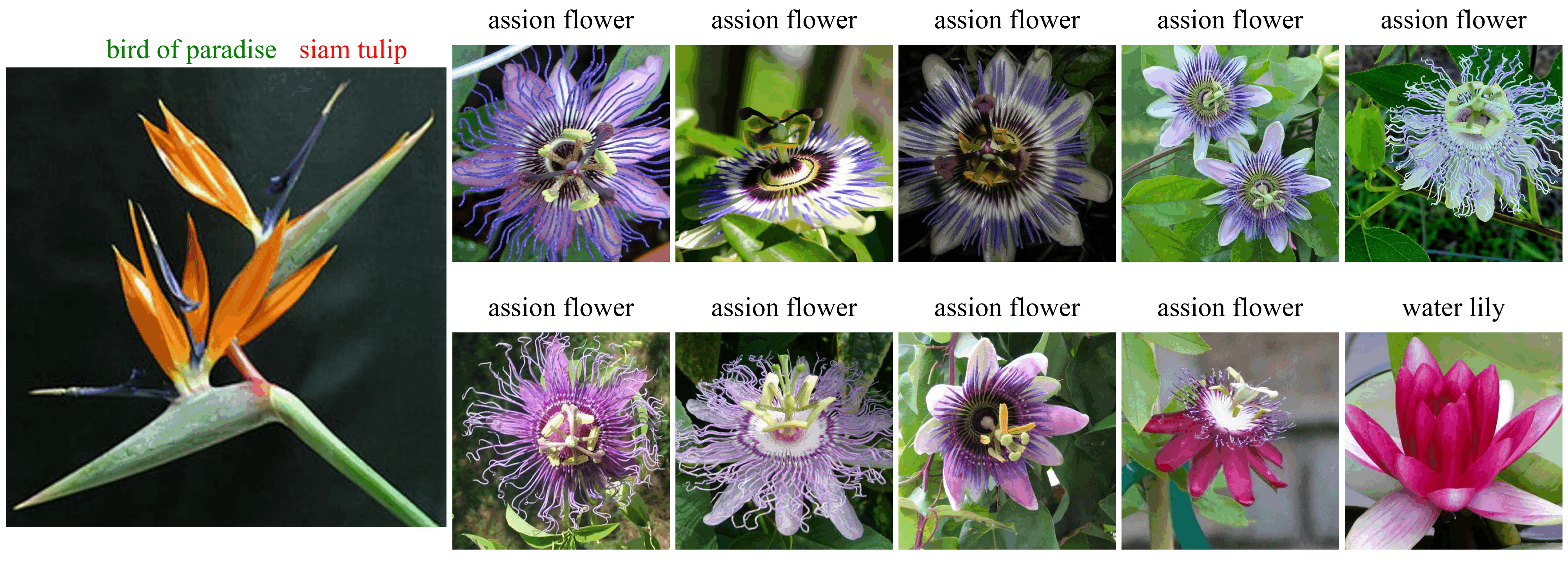}\\
    \end{tabular}
    \caption{Top-10 related test data tracing of mispredicted data on Flowers-102  with 30\%  noise data.   \label{fig:combined_images_part2}}
    \vspace{-7pt}
\end{figure}

\iffalse 
\begin{figure}[ht]
    \centering
    \begin{tabular}{c|c} % Two columns
        \includegraphics[width=0.45\linewidth]{figs/trace/dtd/mispredicted-dtd-banded to polka-dotted.png} & 
        \includegraphics[width=0.45\linewidth]{figs/trace/dtd/mispredicted-dtd-blotchy to paisley.png} \\ \hline
        \includegraphics[width=0.45\linewidth]{figs/trace/dtd/mispredicted-dtd-braided to knitted.png} & 
        \includegraphics[width=0.45\linewidth]{figs/trace/dtd/mispredicted-dtd-bubbly to porous.png} \\ \hline
        \includegraphics[width=0.45\linewidth]{figs/trace/dtd/mispredicted-dtd-cobwebbed to meshed.png} & 
        \includegraphics[width=0.45\linewidth]{figs/trace/dtd/mispredicted-dtd-cracked to grooved.png} \\ \hline
        \includegraphics[width=0.45\linewidth]{figs/trace/dtd/mispredicted-dtd-crosshatched to meshed.png} & 
        \includegraphics[width=0.45\linewidth]{figs/trace/dtd/mispredicted-dtd-dotted to polka-dotted.png} \\ \hline
        \end{tabular}
    \caption{Part 1: Combined images showing the impact of different algorithms or datasets.\di{change caption}}
    \label{fig:combined_images_part3}
    \vspace{-7pt}
\end{figure}
\fi

\begin{table}[htbp]
    \centering
    \begin{tabular}{c|c}
    \includegraphics[width=0.45\textwidth]{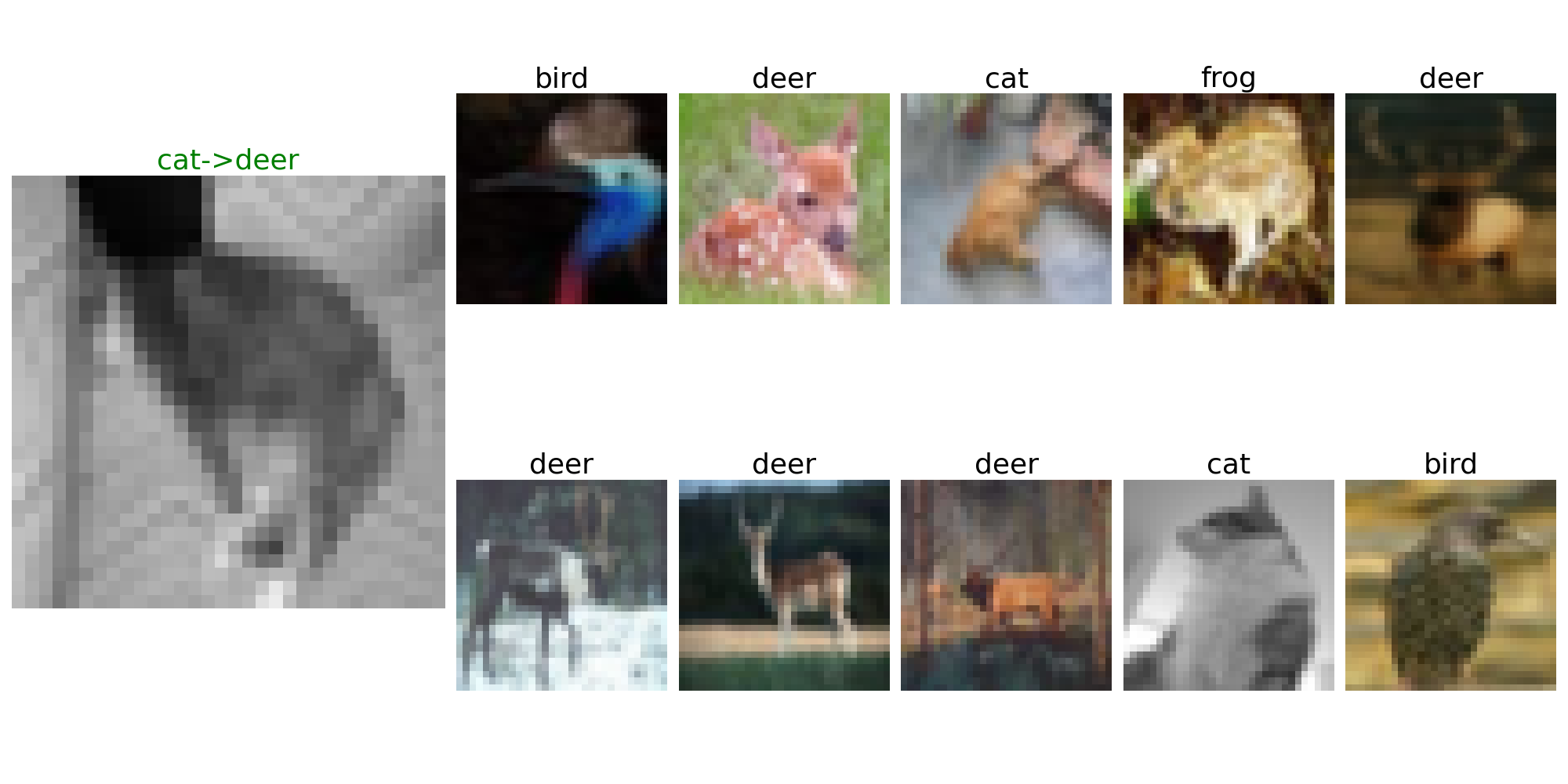}    & \includegraphics[width=0.45\textwidth]{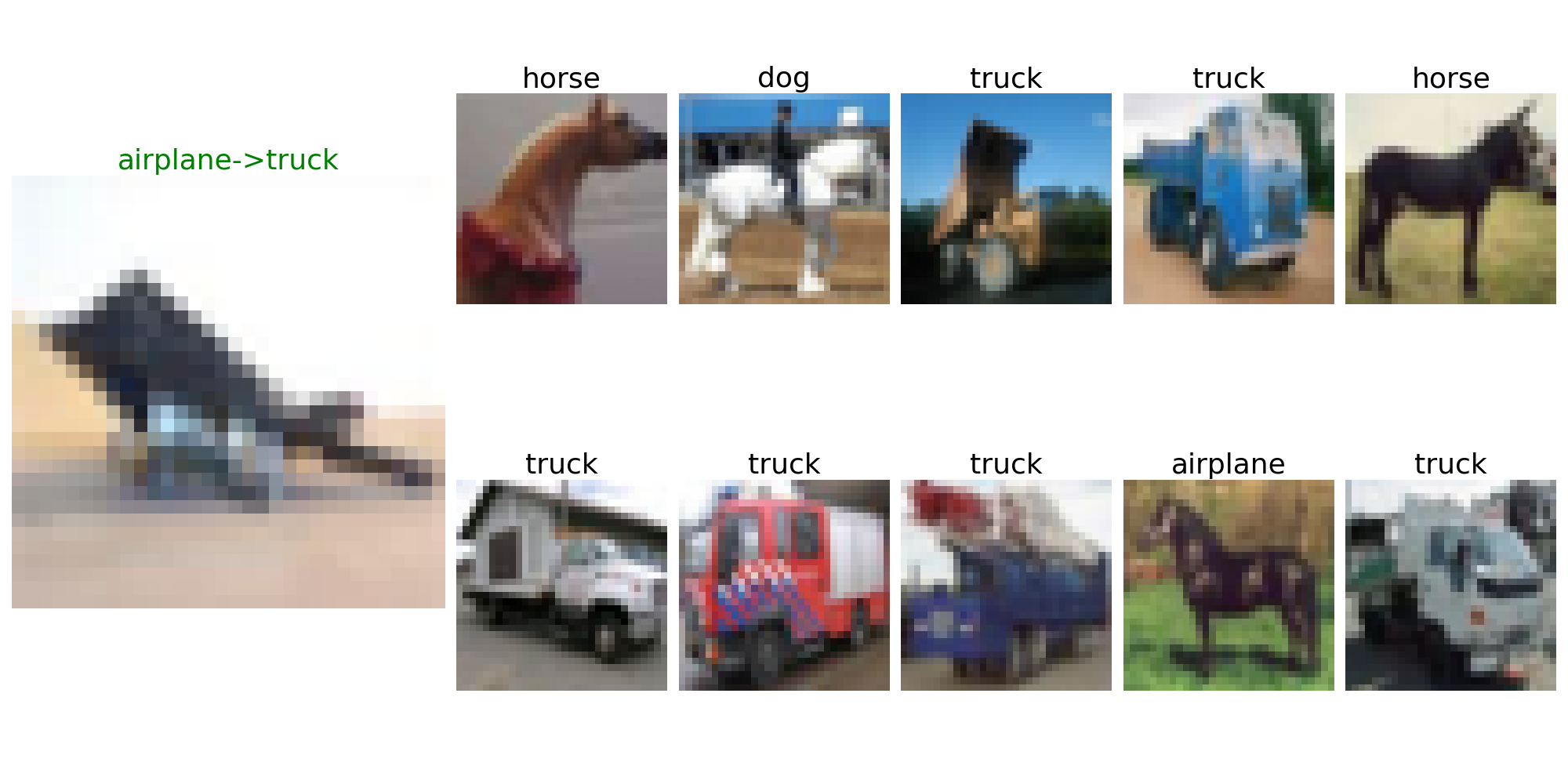} 
       \\ \hline 
    \includegraphics[width=0.45\textwidth]{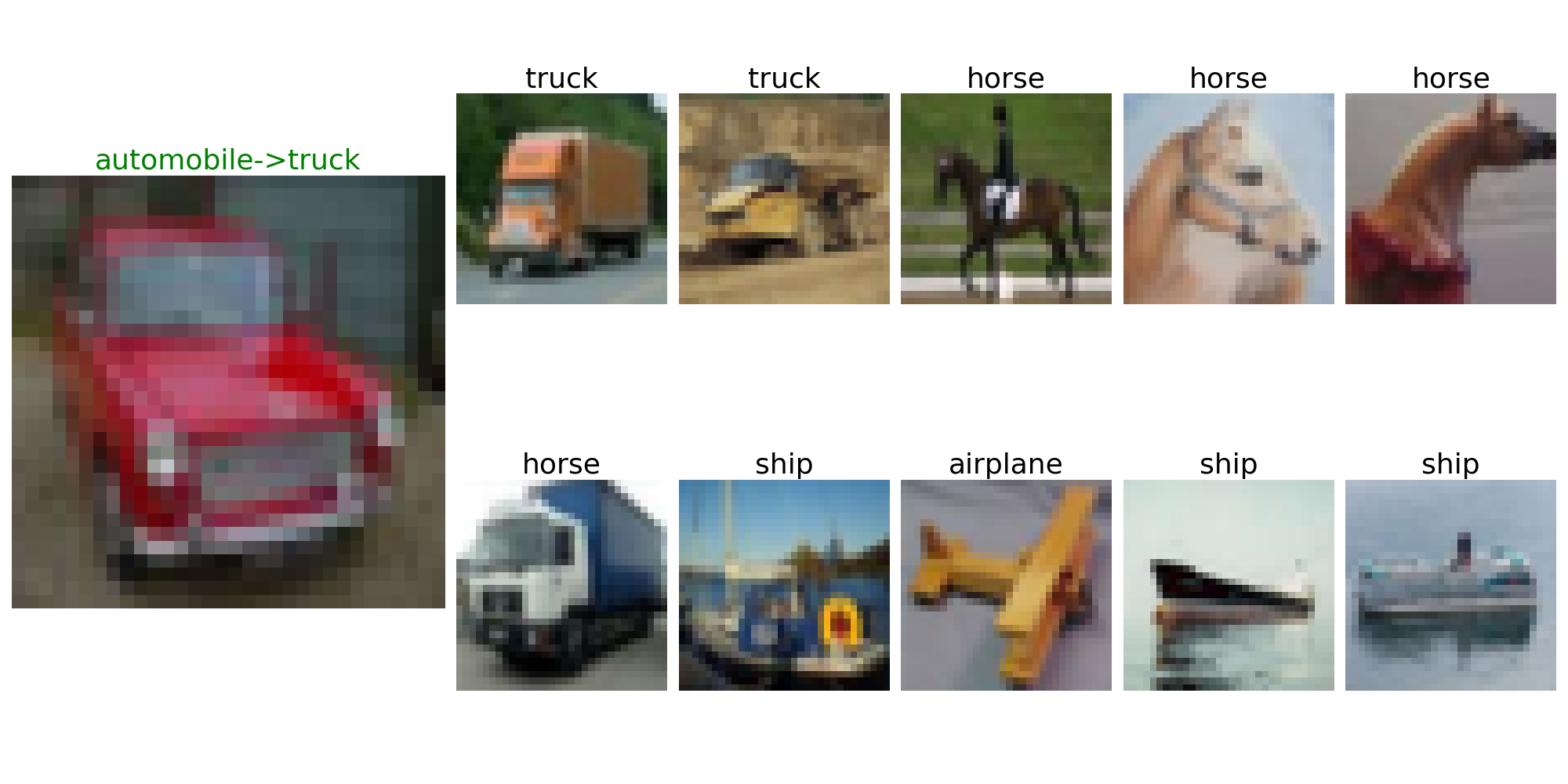}   &\includegraphics[width=0.45\textwidth]{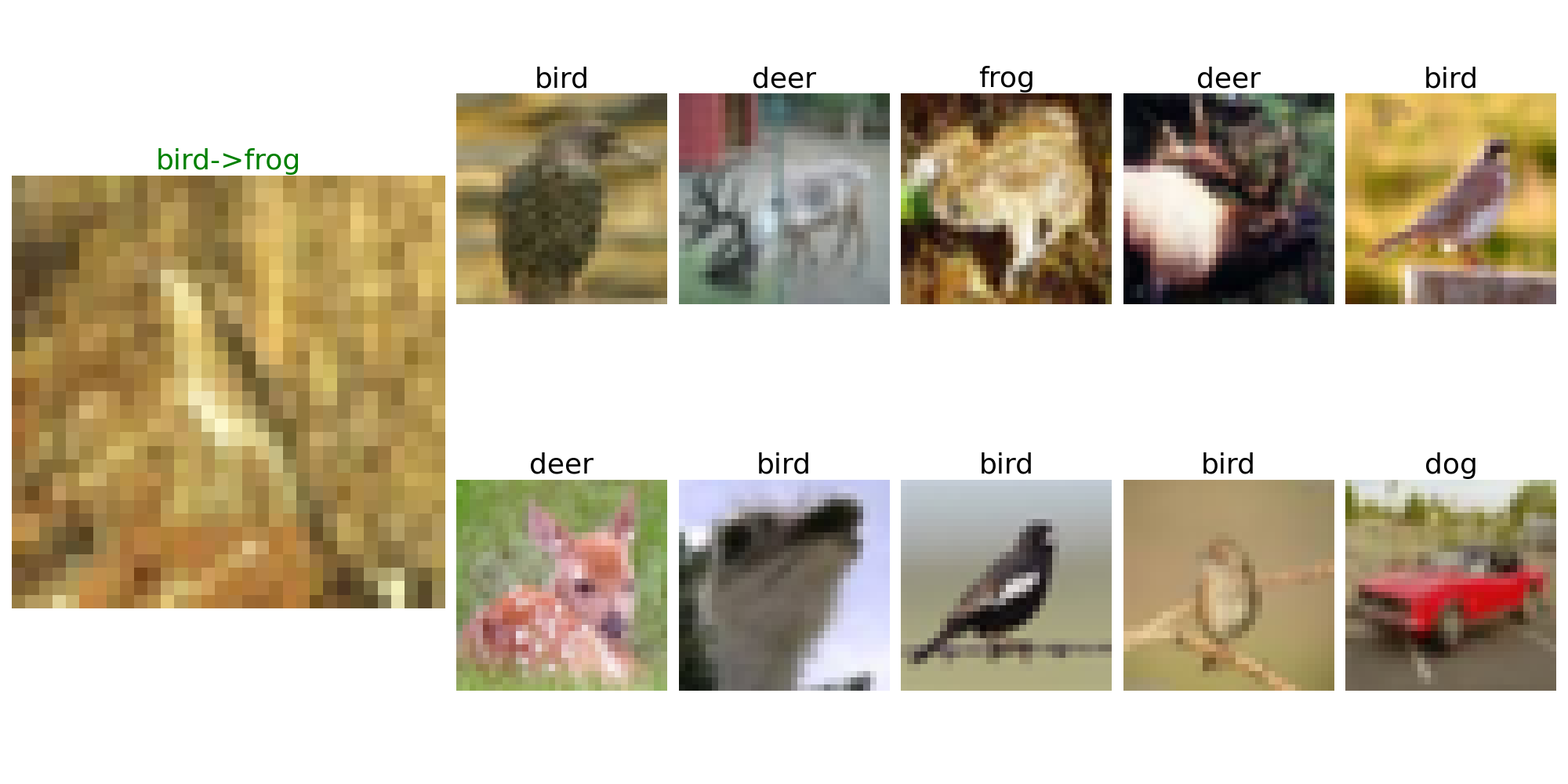}\\ \hline
    \includegraphics[width=0.45\textwidth]{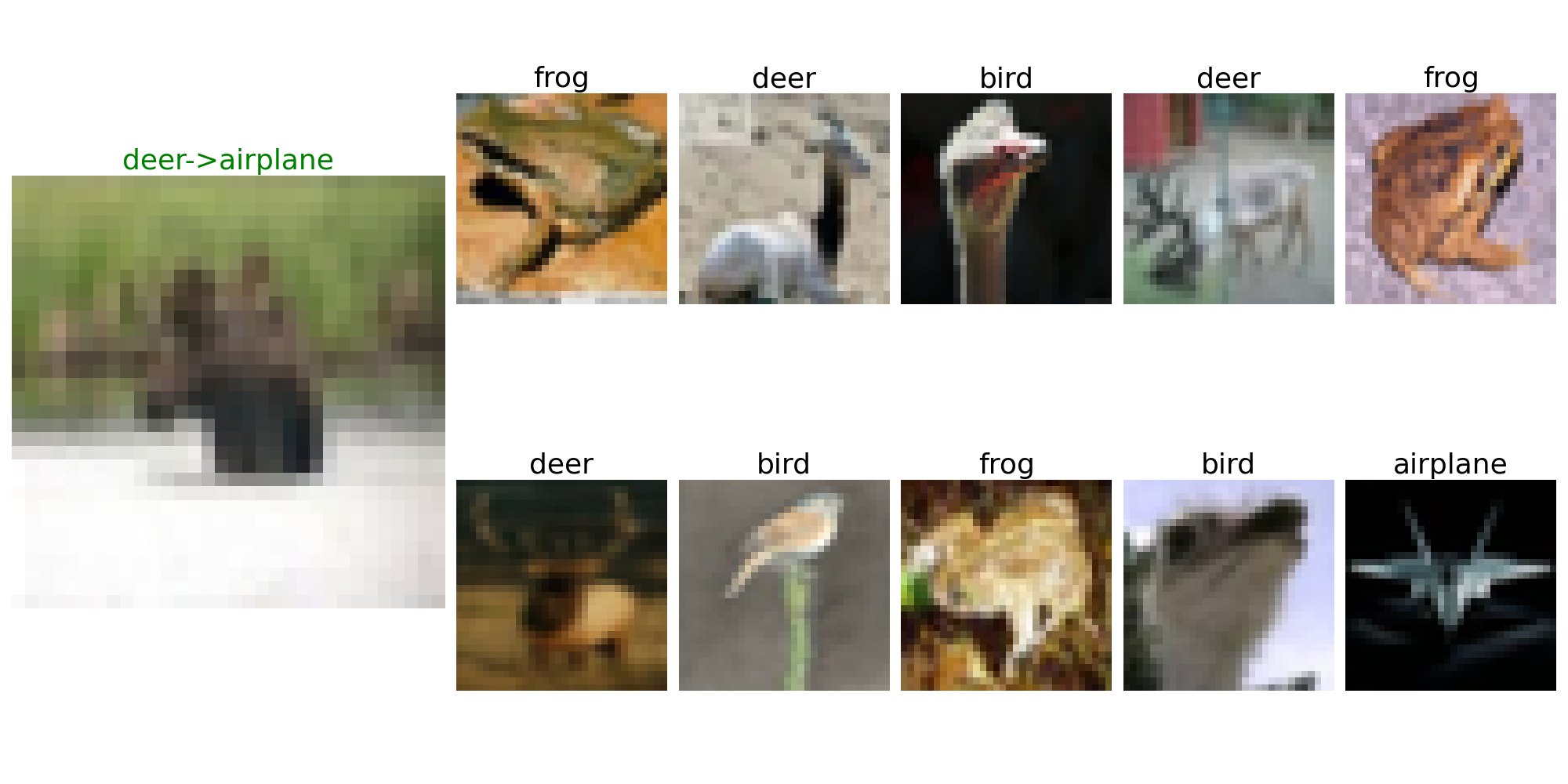} &
    \includegraphics[width=0.45\textwidth]{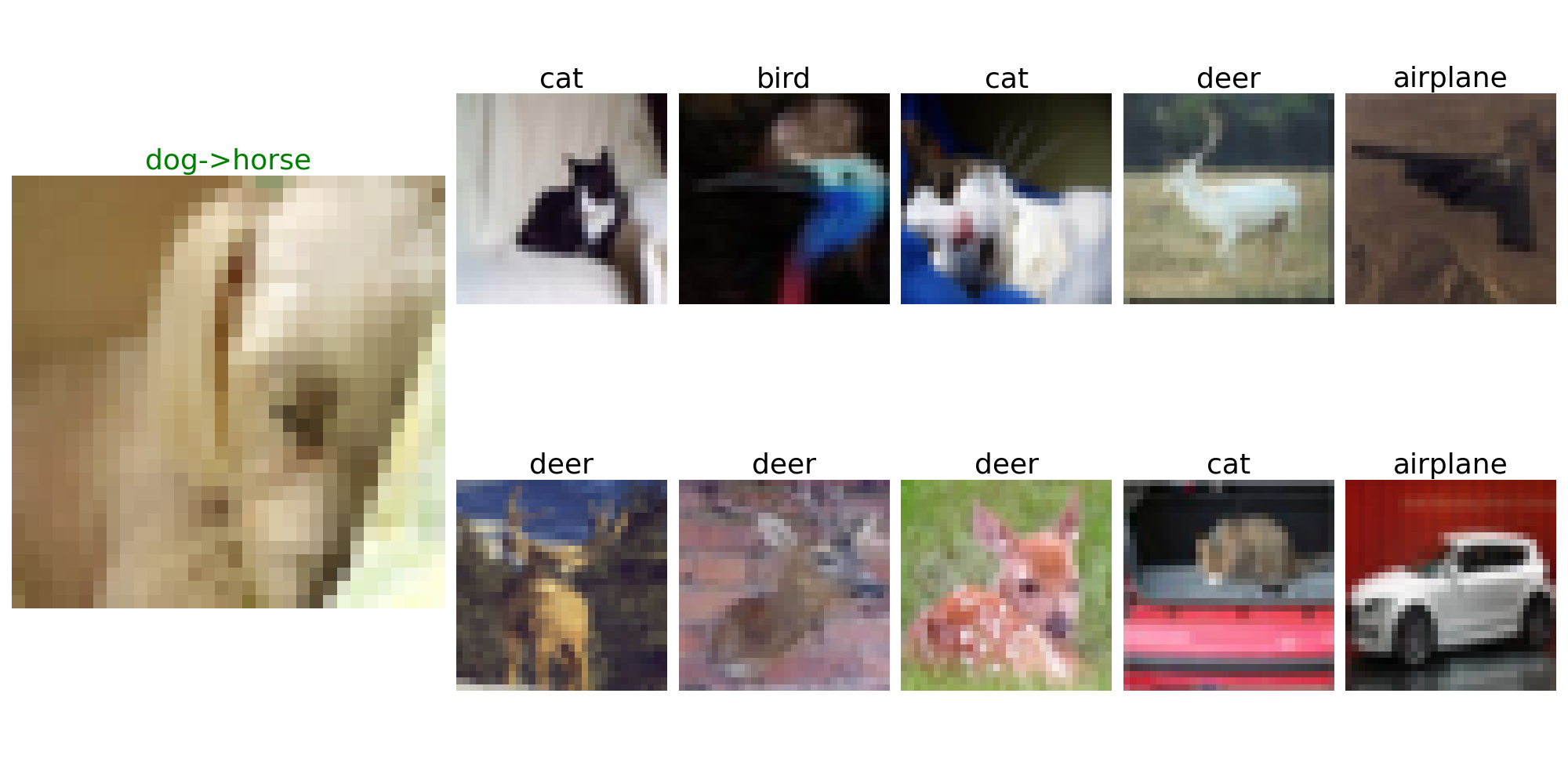}  \\ \hline
    \includegraphics[width=0.45\textwidth]{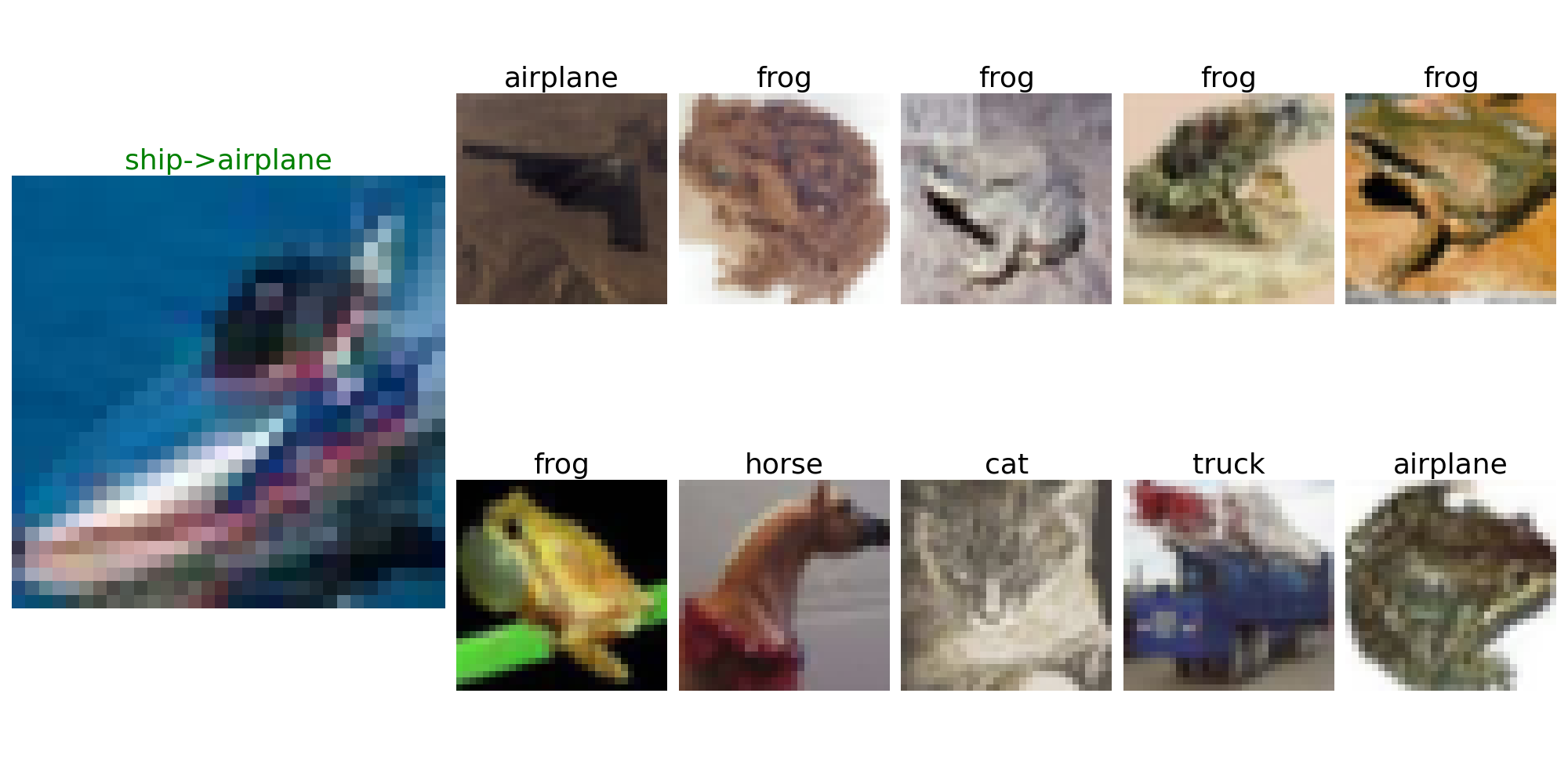} &
    \includegraphics[width=0.45\textwidth]{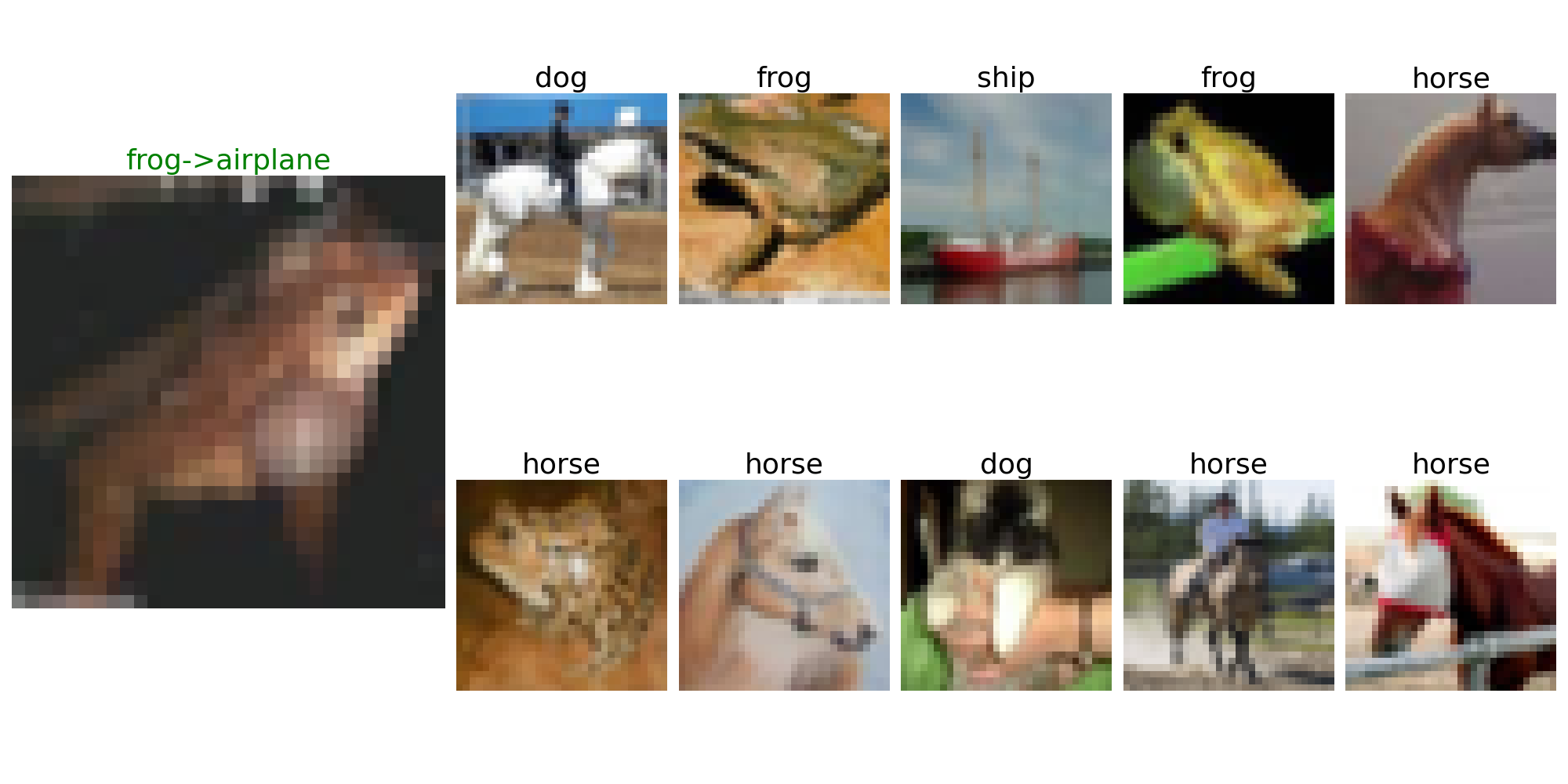} \\ \hline
    \includegraphics[width=0.45\textwidth]{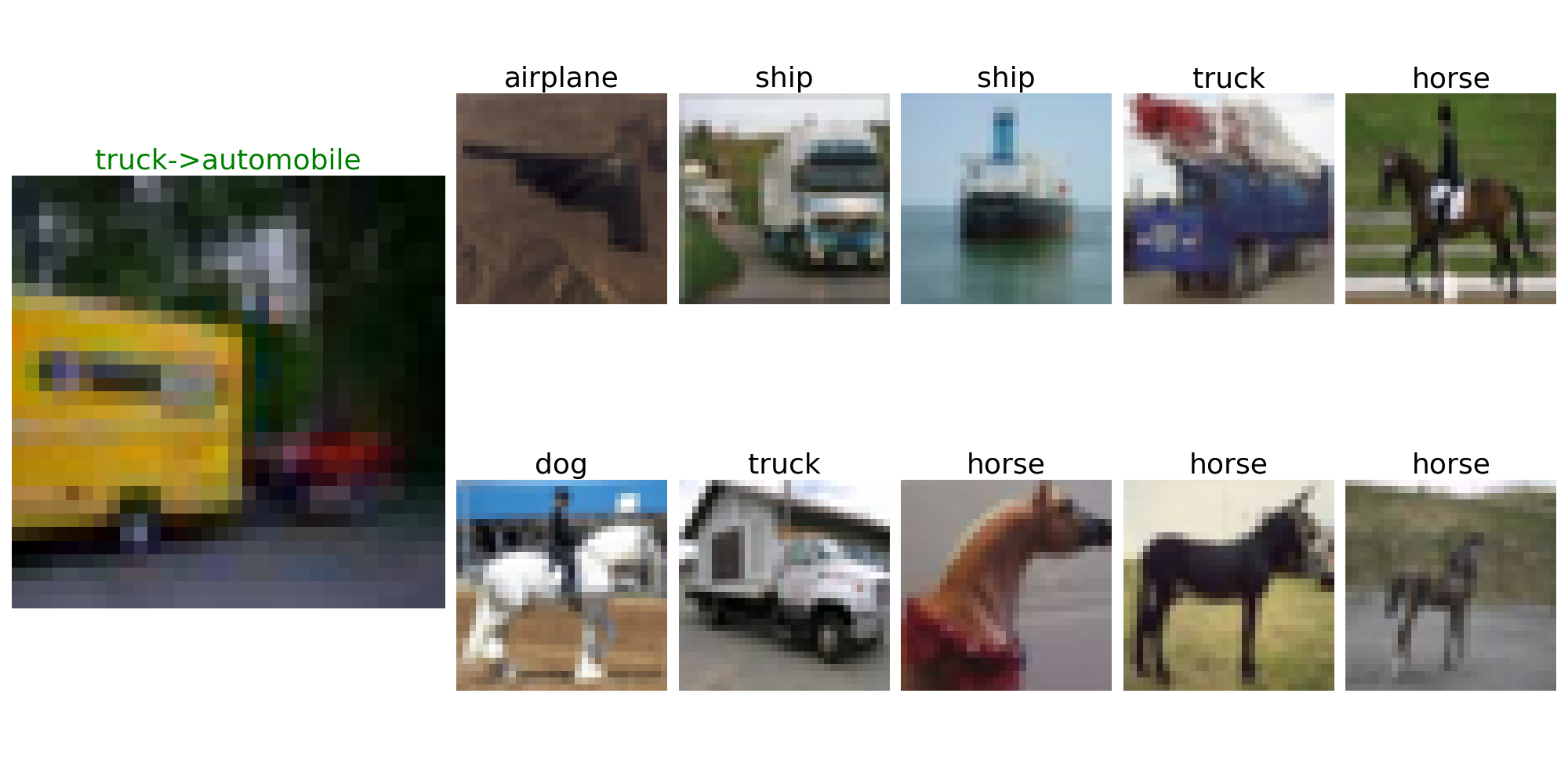} &
    \includegraphics[width=0.45\textwidth]{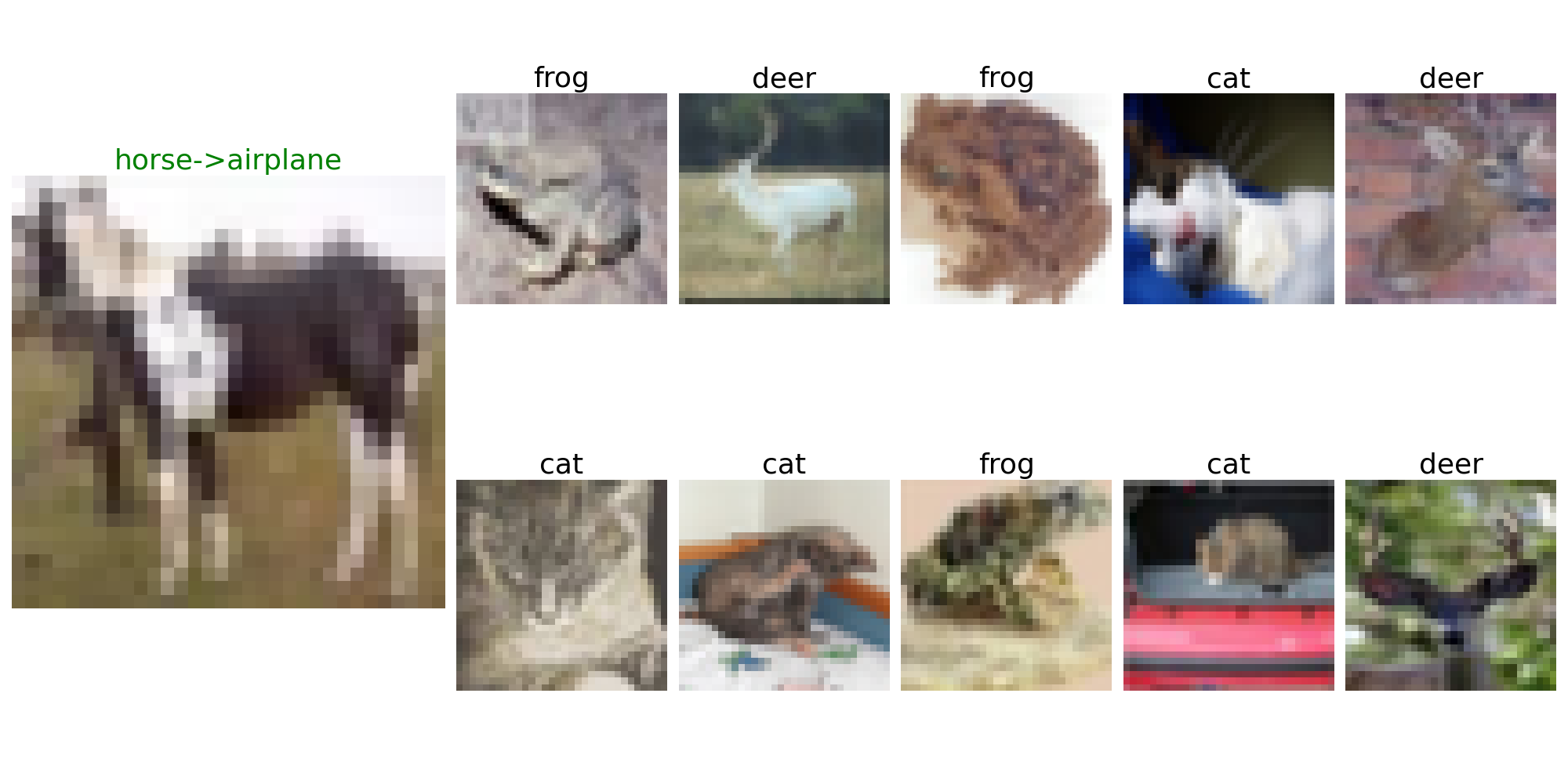}
      \end{tabular}
    \caption{Top-10 related test data tracing of mispredicted data on cifar-10 dataset with 10\%  noise data.}
    \label{tab:trace_vis1}
\end{table}

\begin{table}[htbp]
    \centering
    \begin{tabular}{c|c}
    \includegraphics[width=0.45\textwidth]{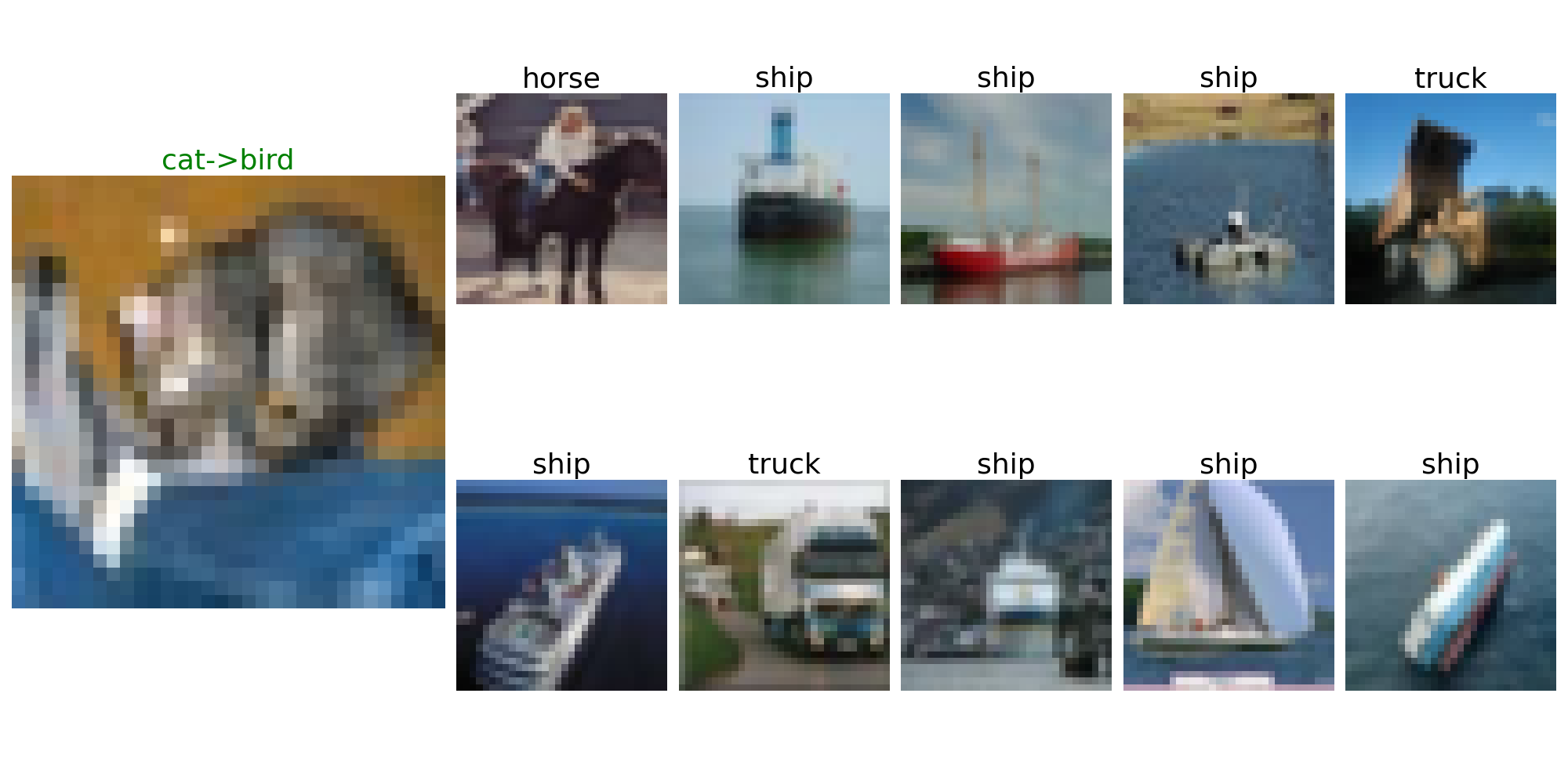}    & \includegraphics[width=0.45\textwidth]{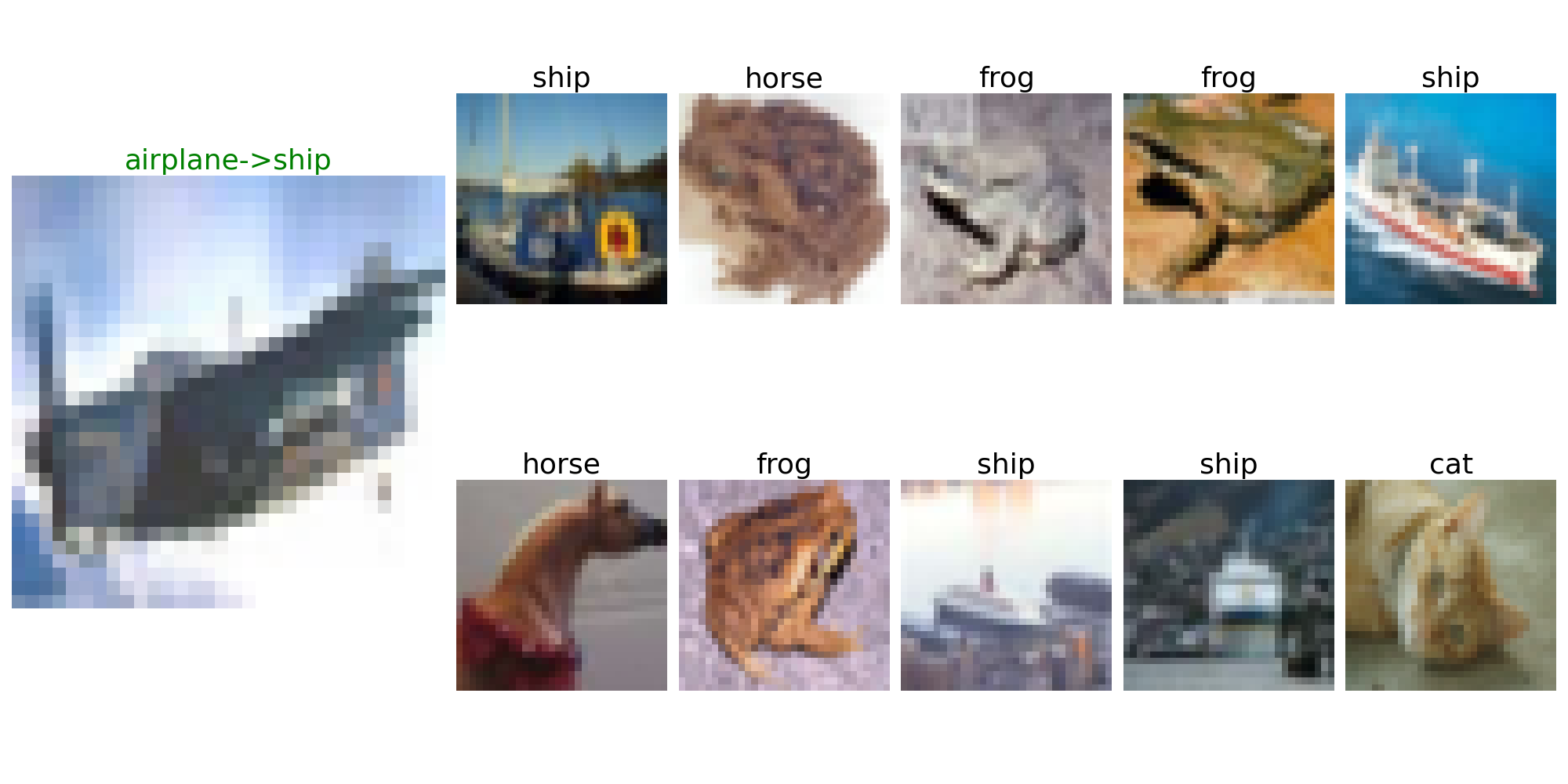} 
    \\ \hline 
    \includegraphics[width=0.45\textwidth]{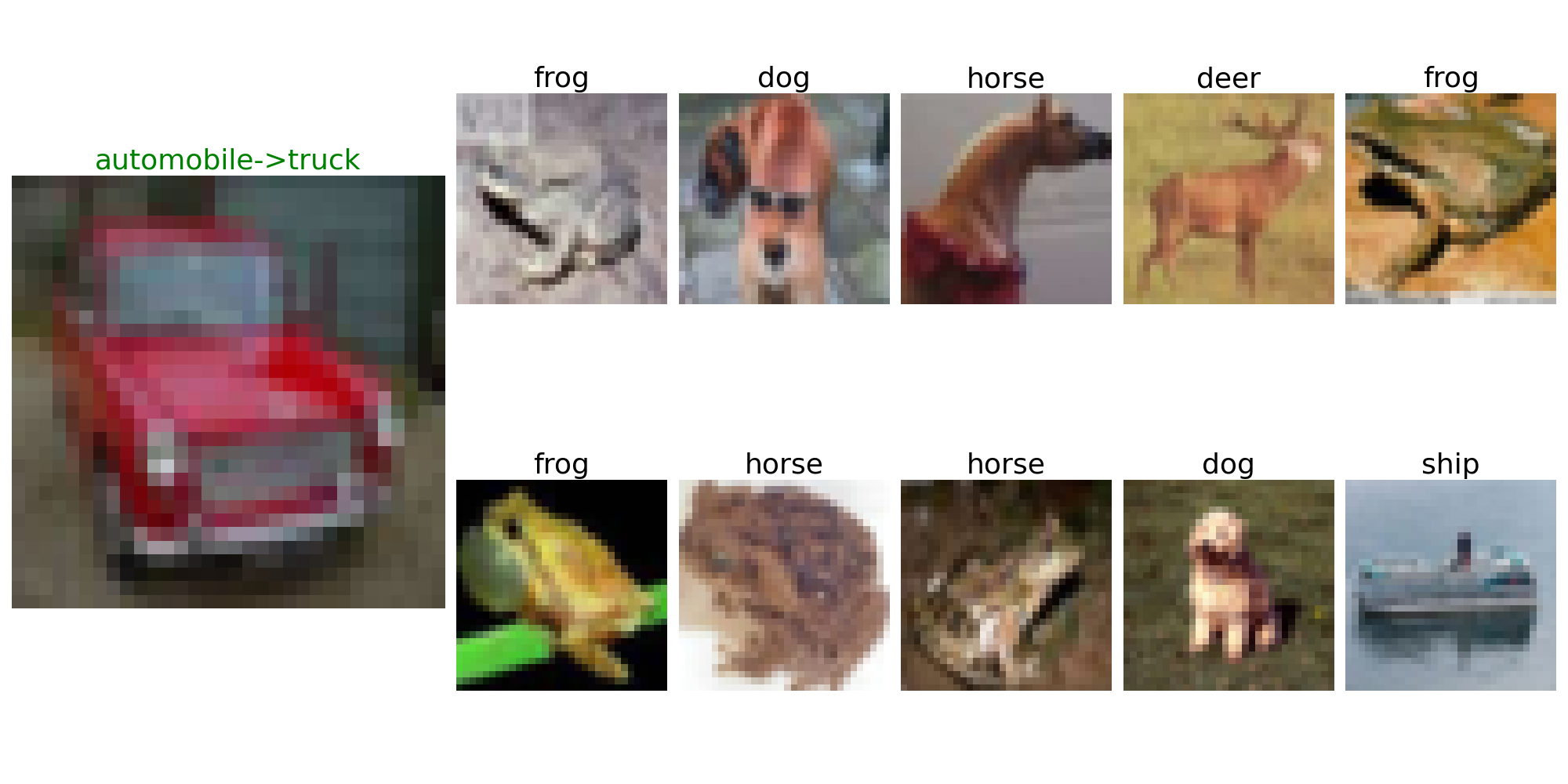}   & \includegraphics[width=0.45\textwidth]{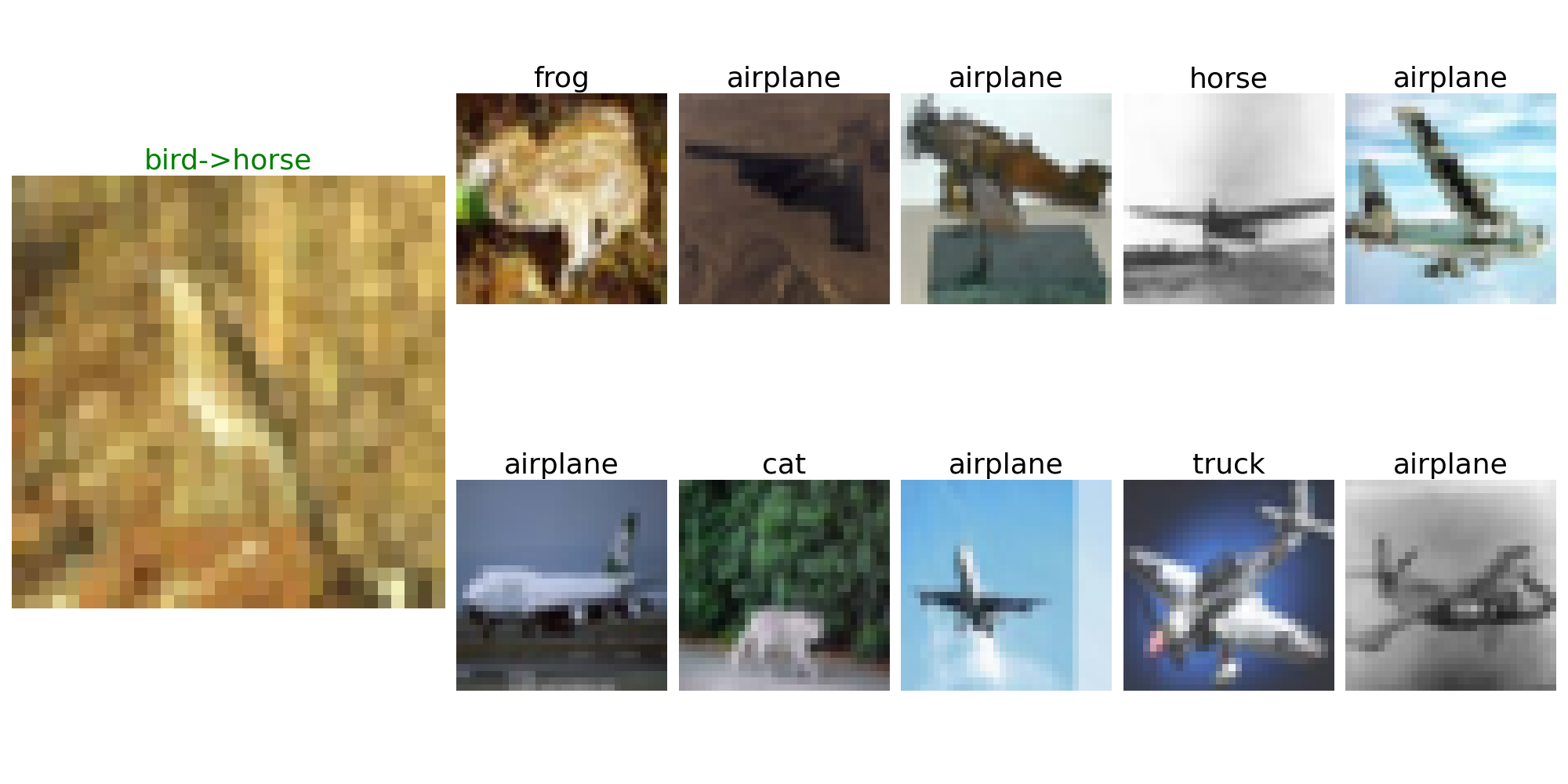}\\ \hline
    \includegraphics[width=0.45\textwidth]{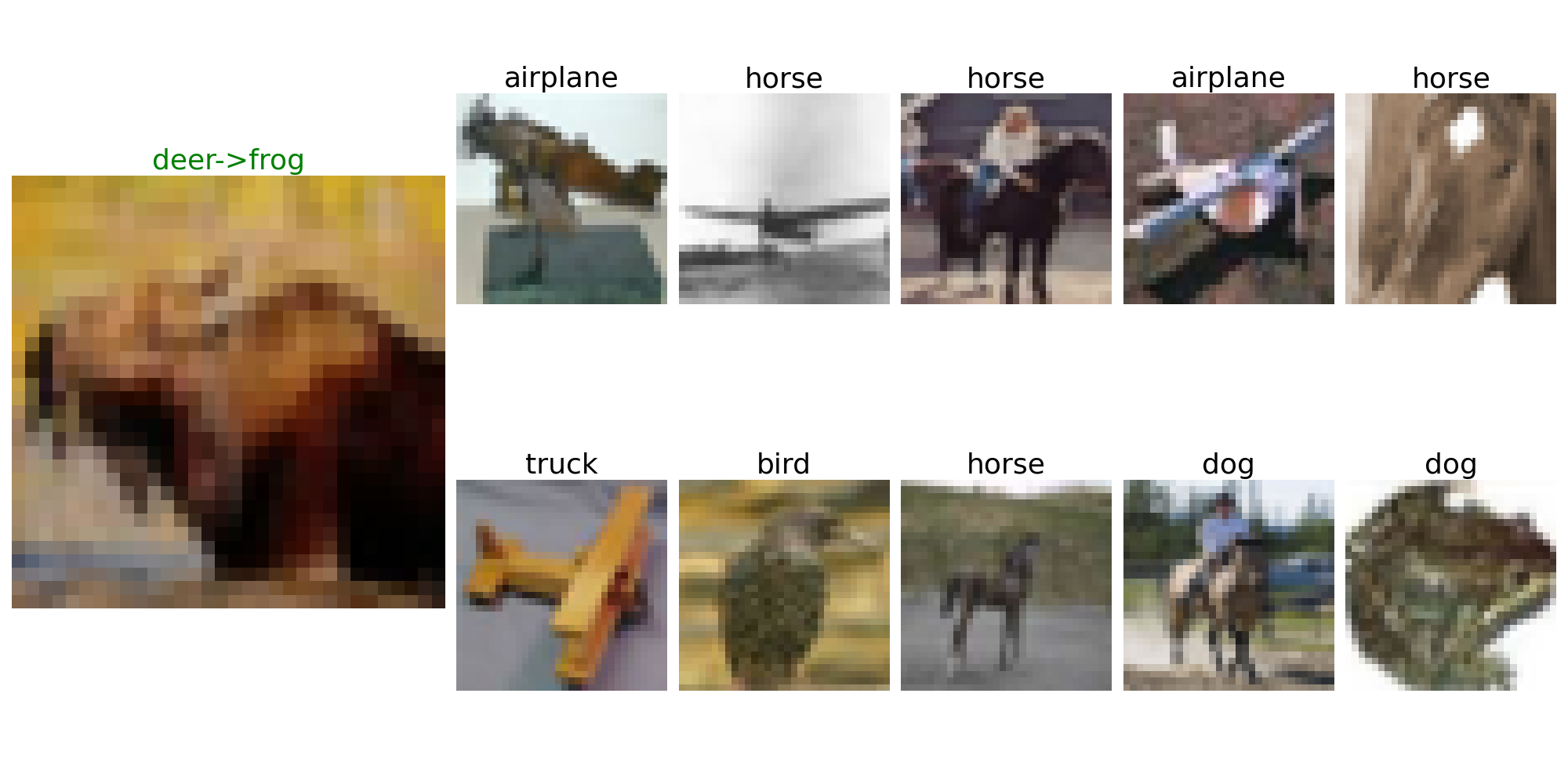}   & \includegraphics[width=0.45\textwidth]{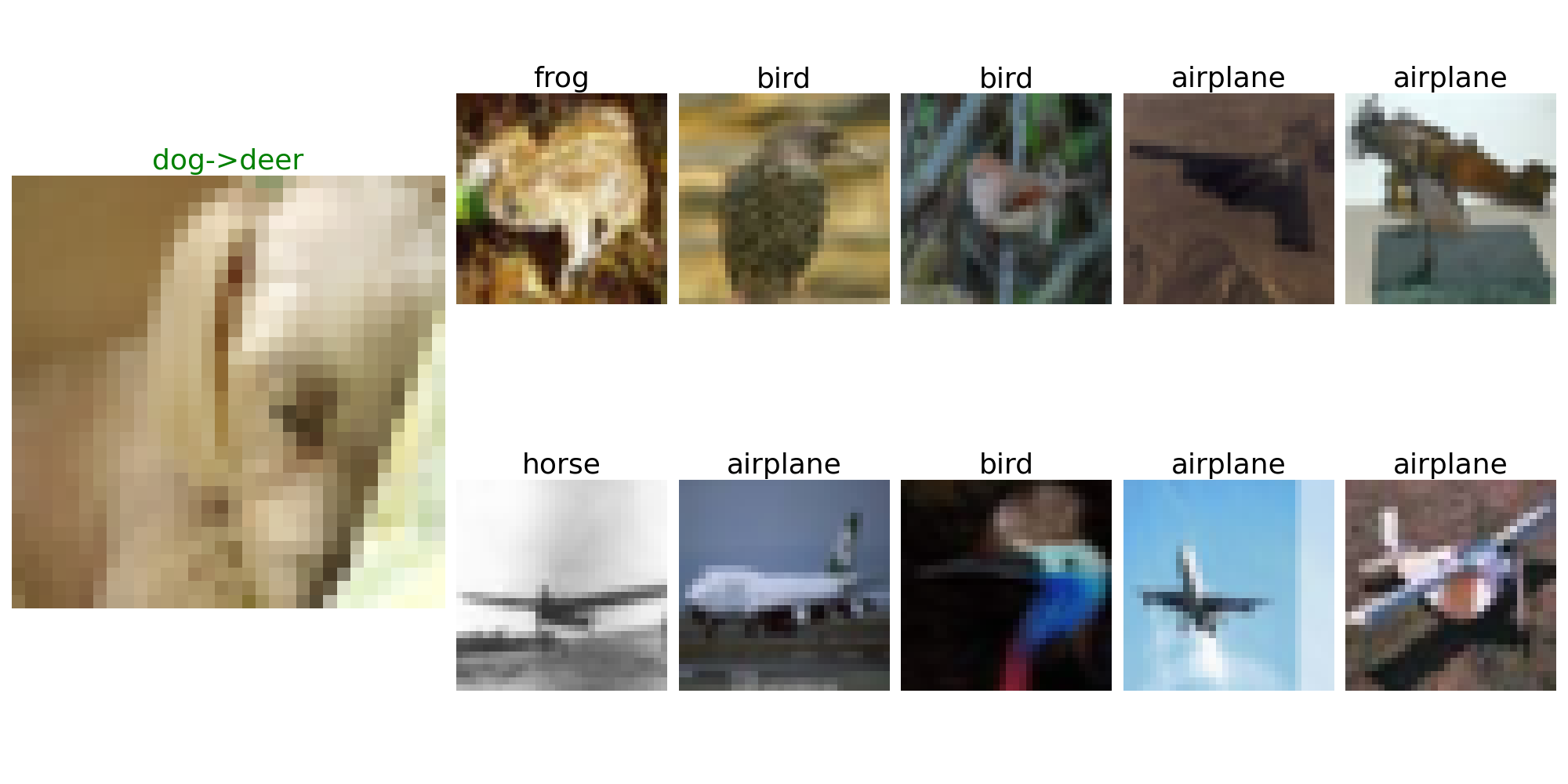}  \\ \hline
    \includegraphics[width=0.45\textwidth]{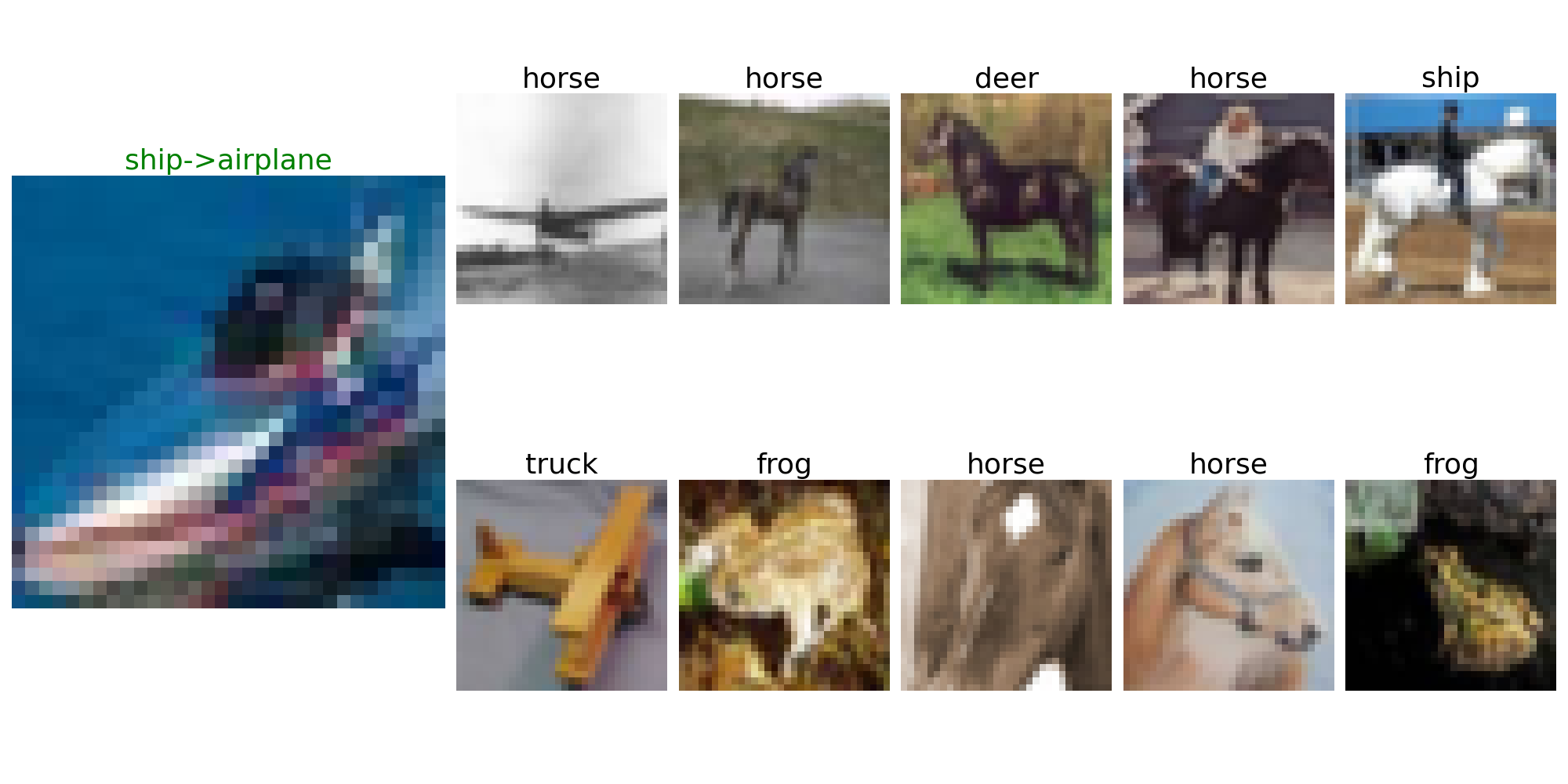} &
  \includegraphics[width=0.45\textwidth]{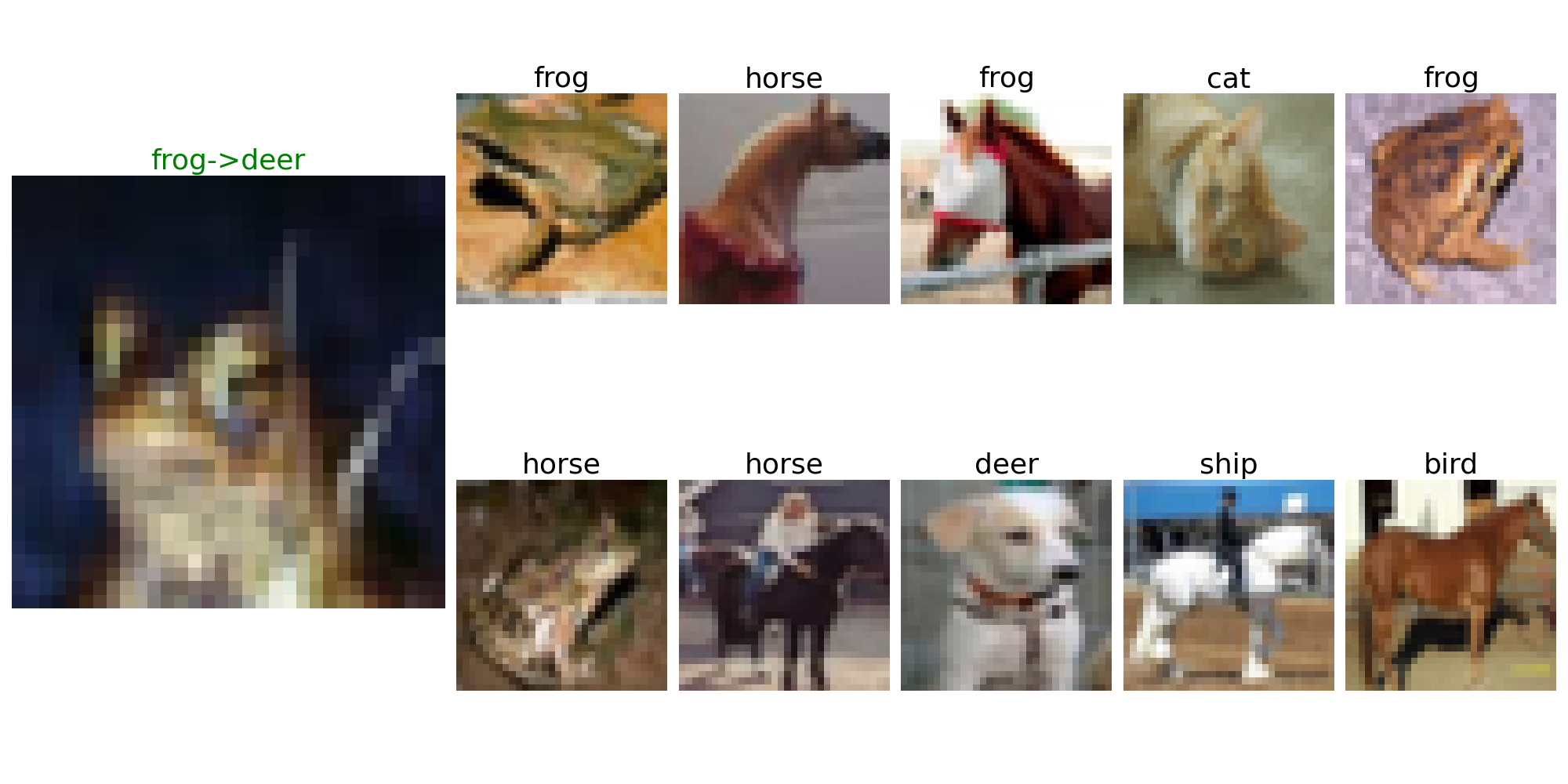} \\ \hline
    \includegraphics[width=0.45\textwidth]{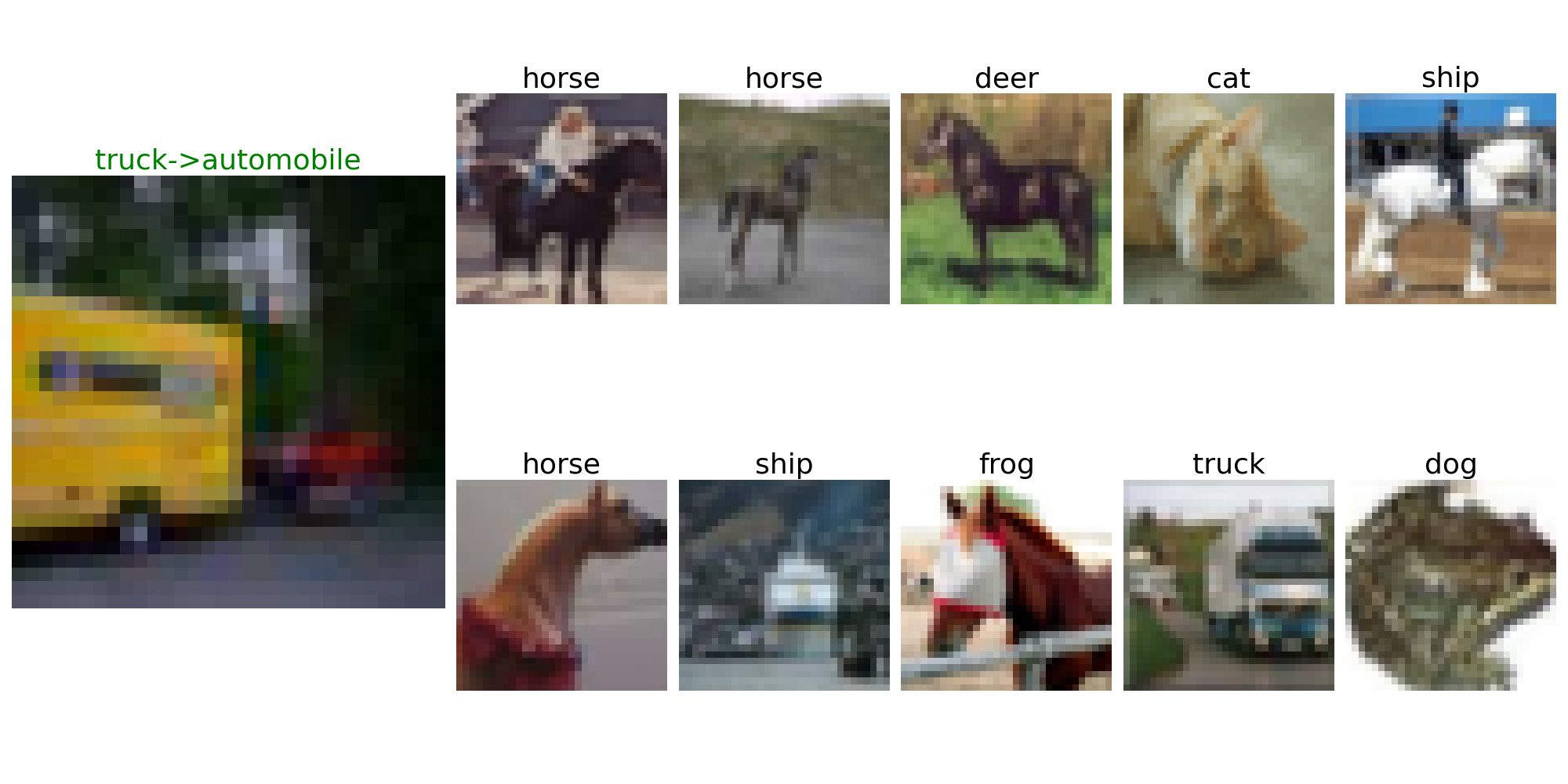} &
    \includegraphics[width=0.45\textwidth]{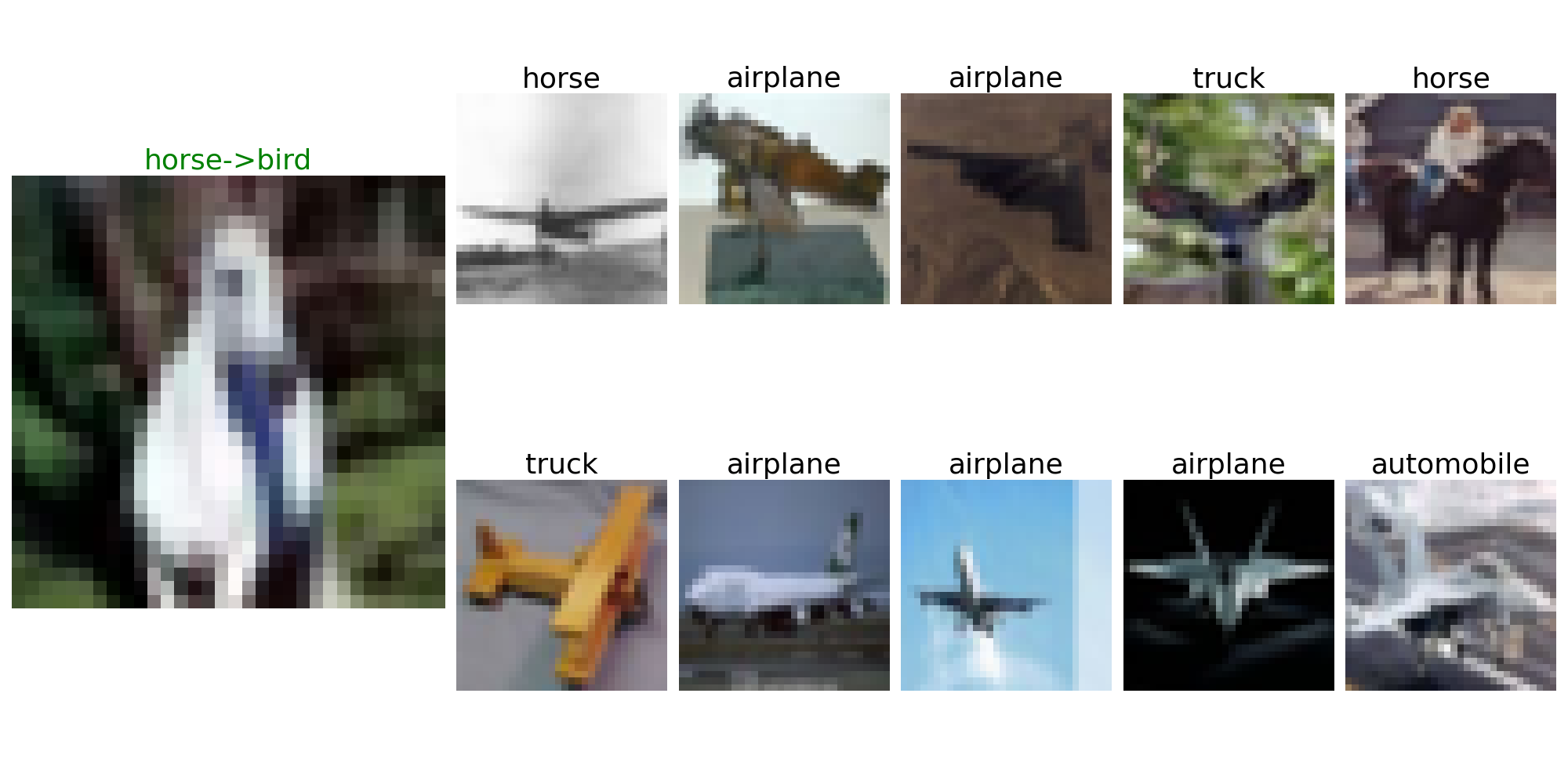}
    \end{tabular}
    \caption{Top-10 related test data tracing of mispredicted data on cifar-10 dataset with 20\%  noise data.}
    \label{tab:trace_vis2}
\end{table}

\begin{table}[htbp]
  \centering
    \begin{tabular}{c|c}
    \includegraphics[width=0.45\textwidth]{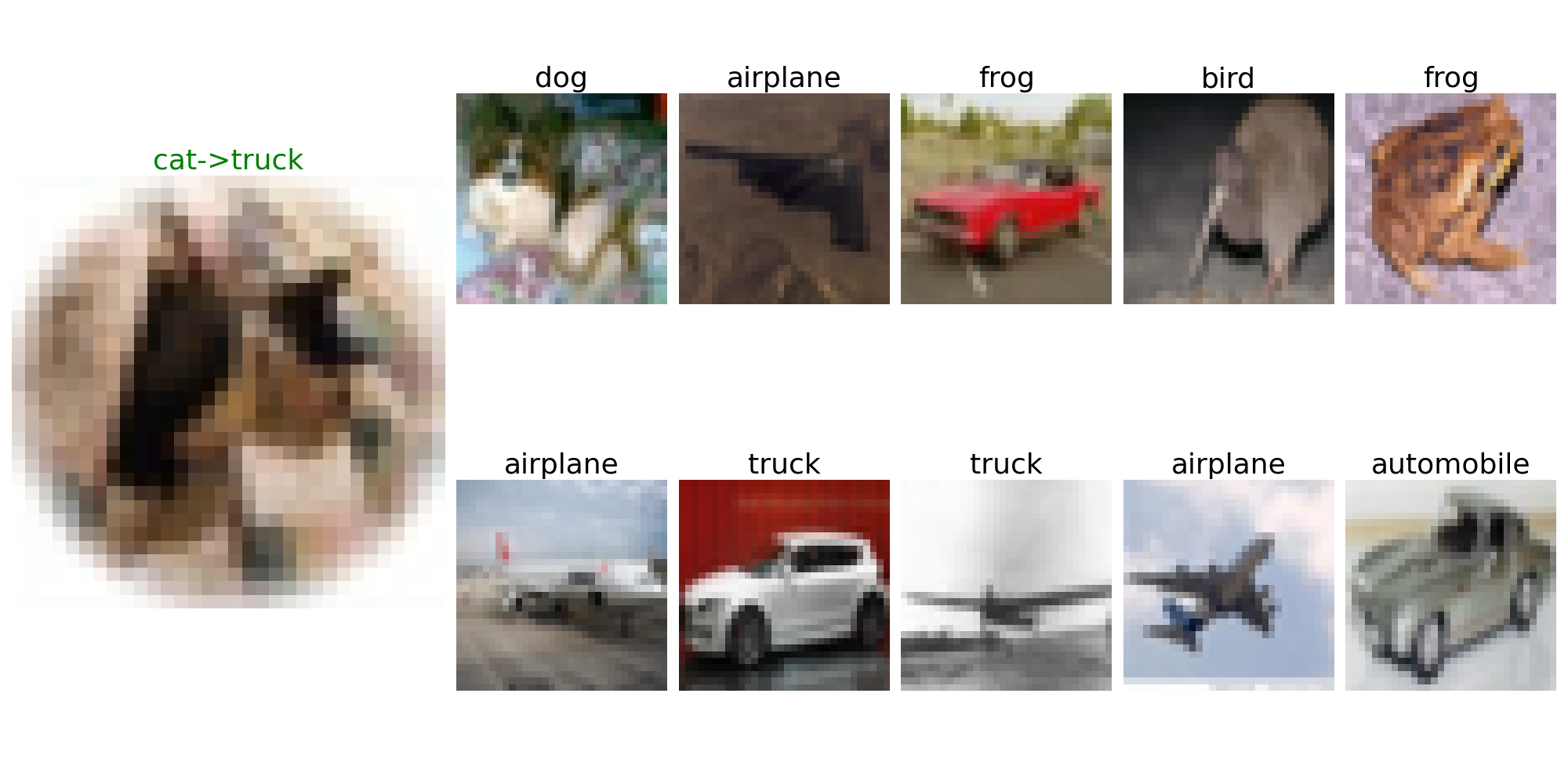}    & \includegraphics[width=0.45\textwidth]{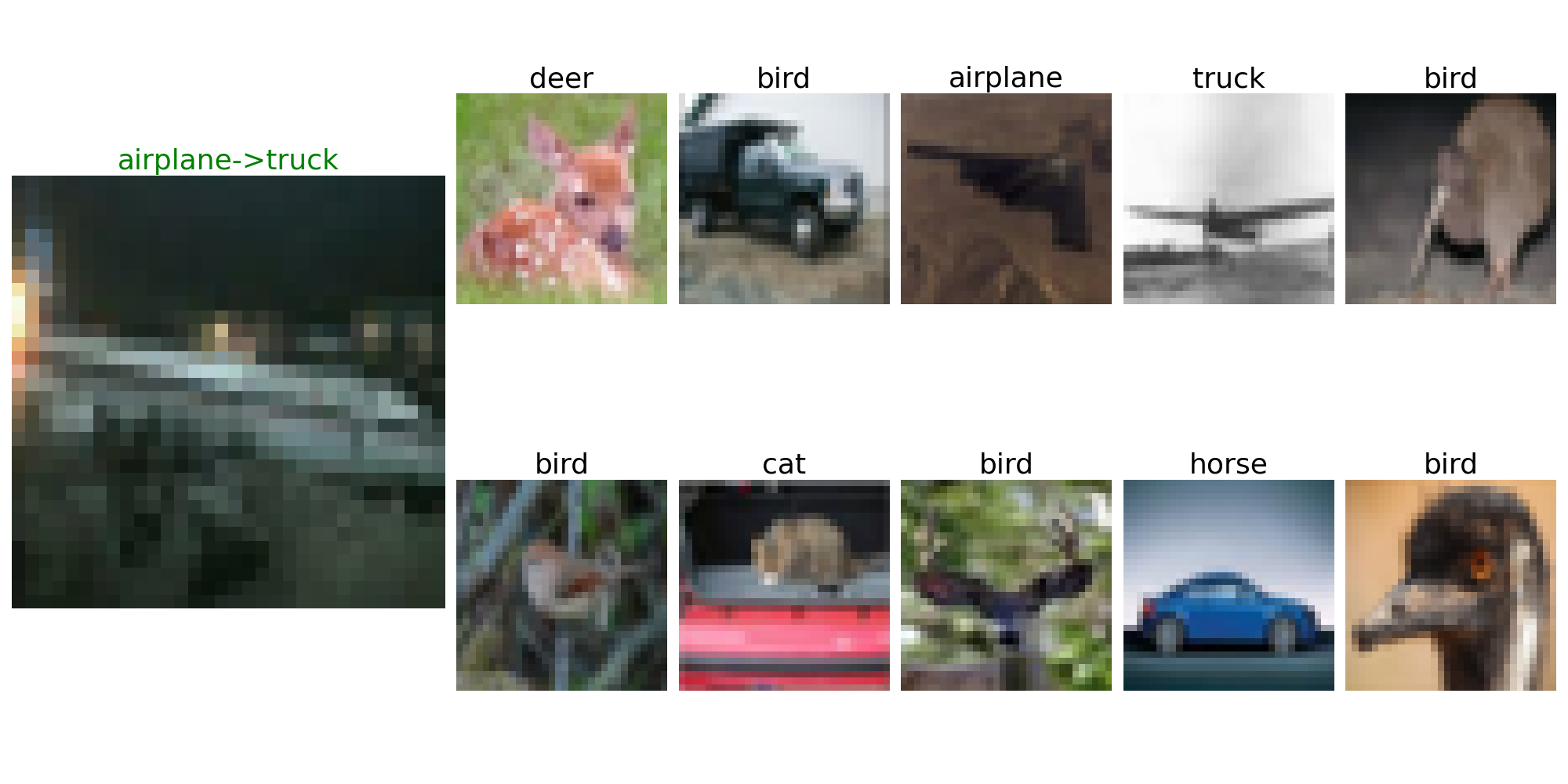} 
    \\ \hline 
    \includegraphics[width=0.45\textwidth]{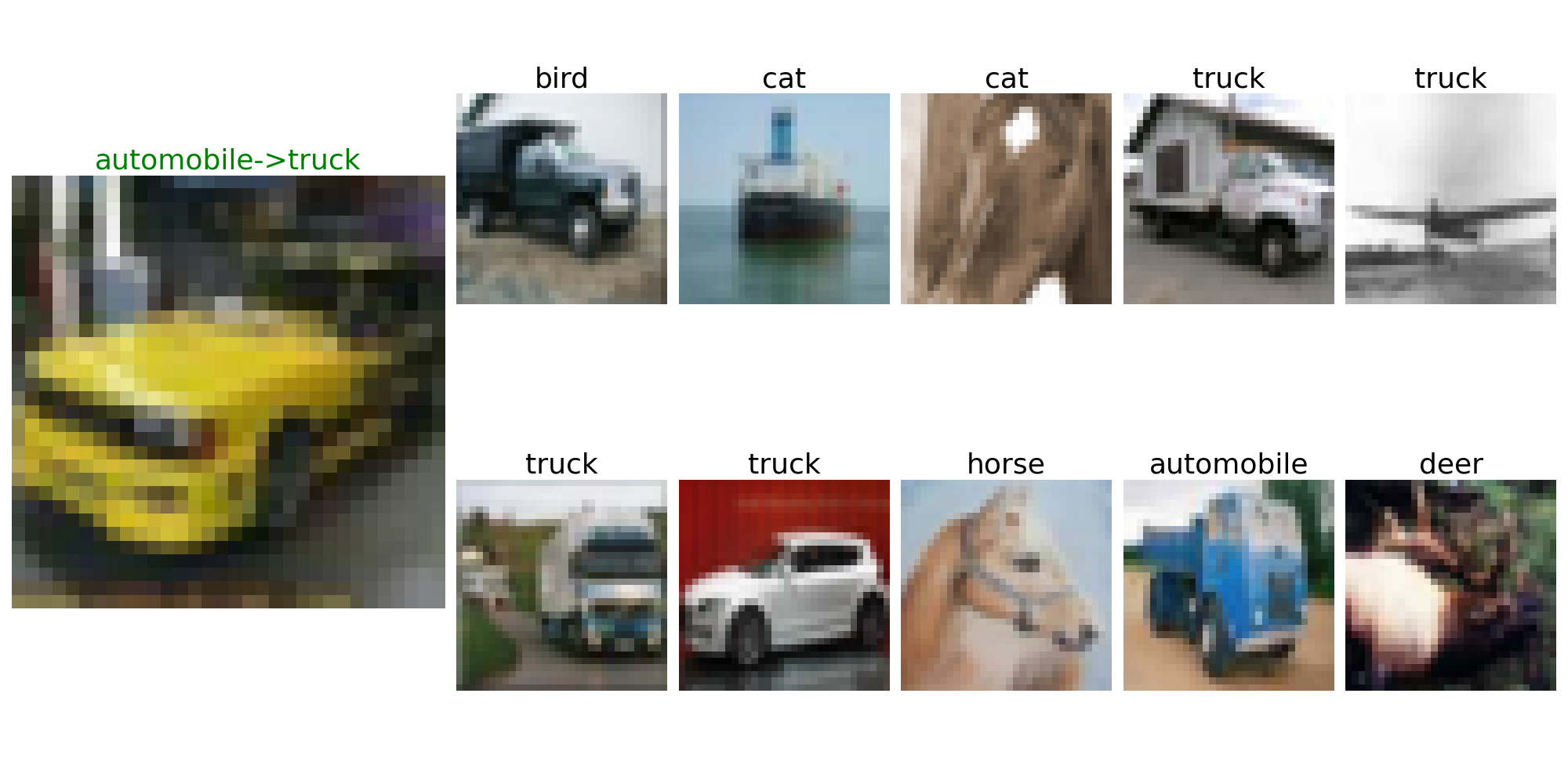}   & \includegraphics[width=0.45\textwidth]{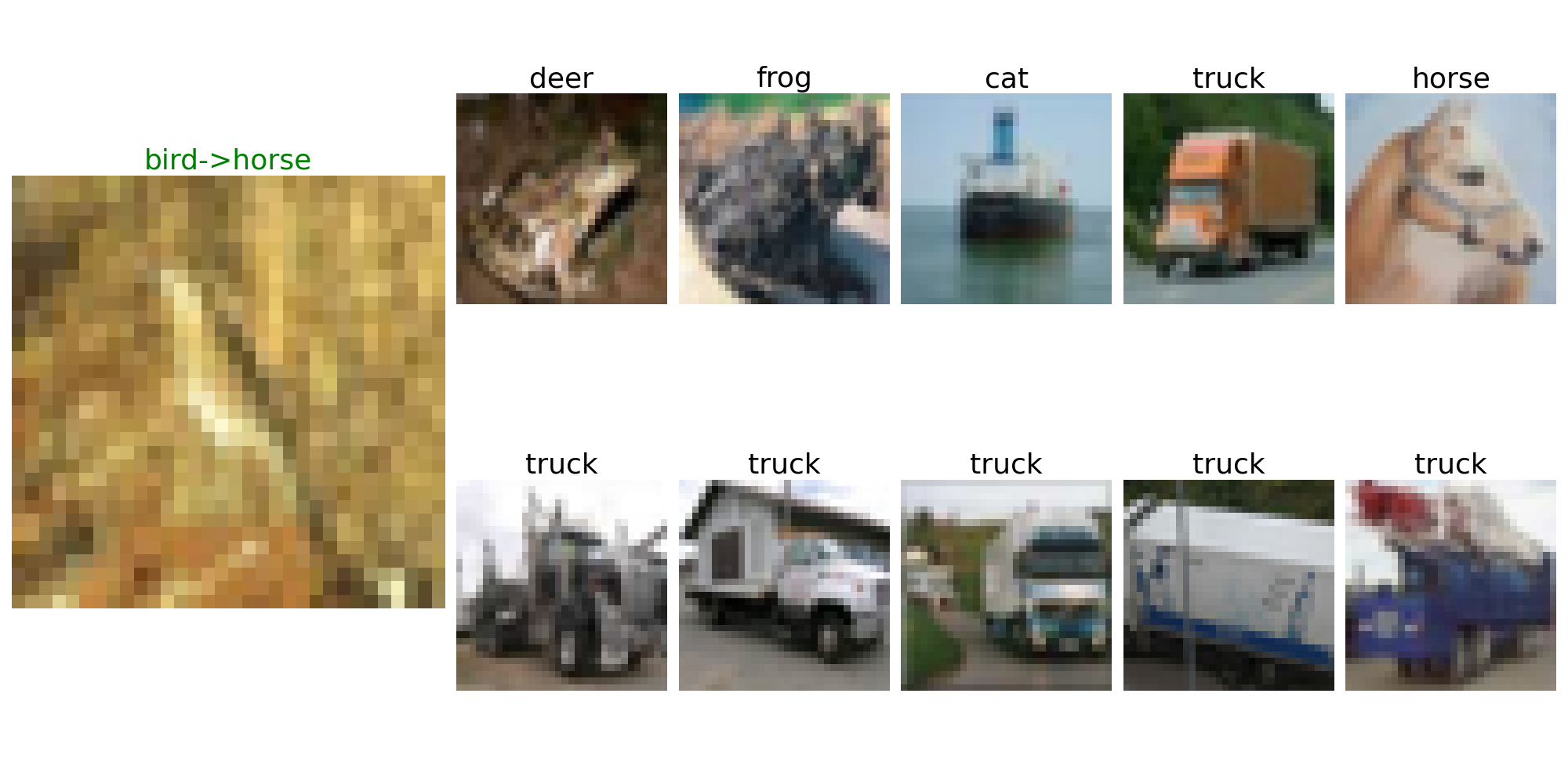}\\ \hline
    \includegraphics[width=0.45\textwidth]{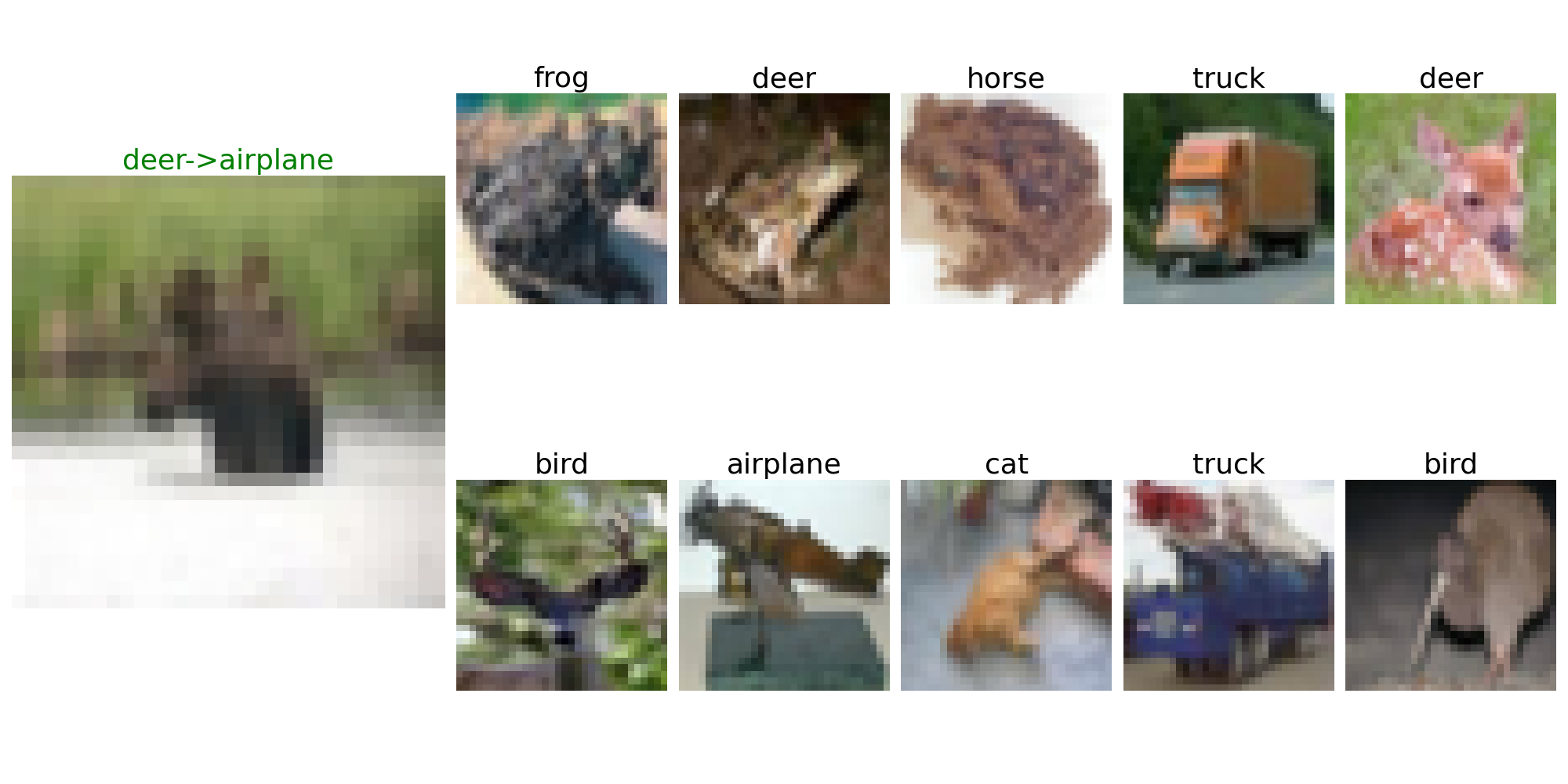}   & \includegraphics[width=0.45\textwidth]{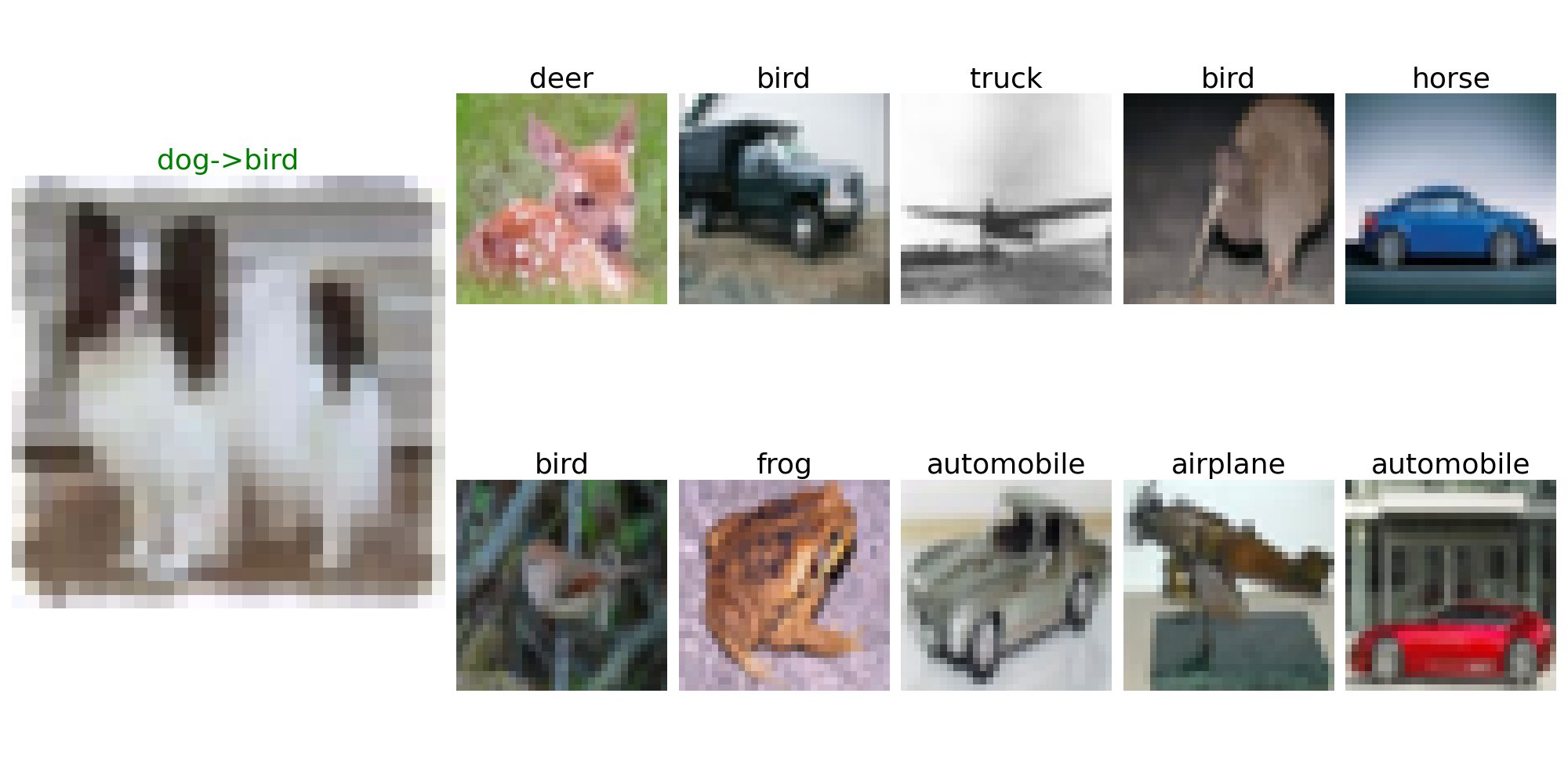}  \\ \hline
    \includegraphics[width=0.45\textwidth]{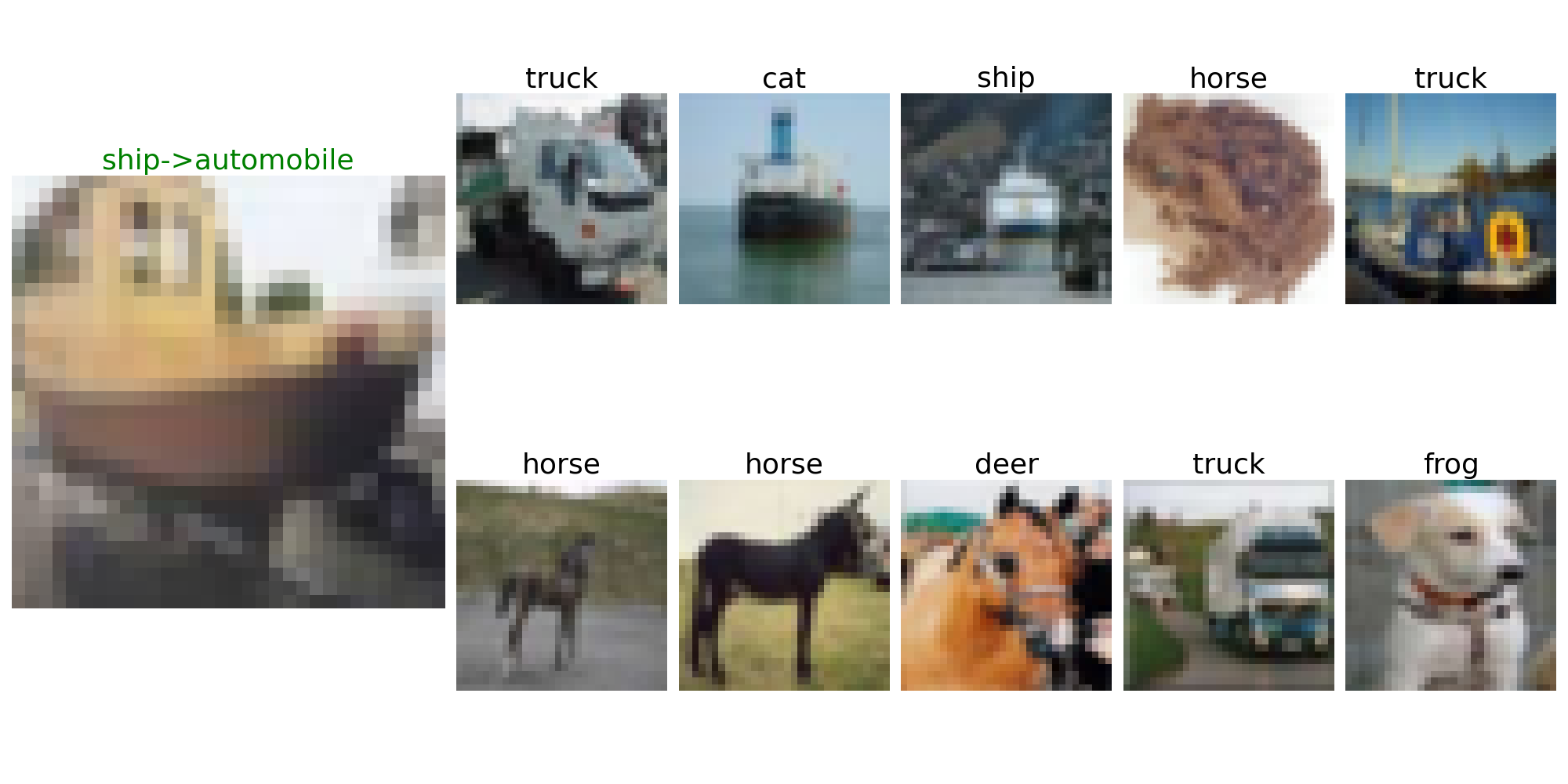} &
    \includegraphics[width=0.45\textwidth]{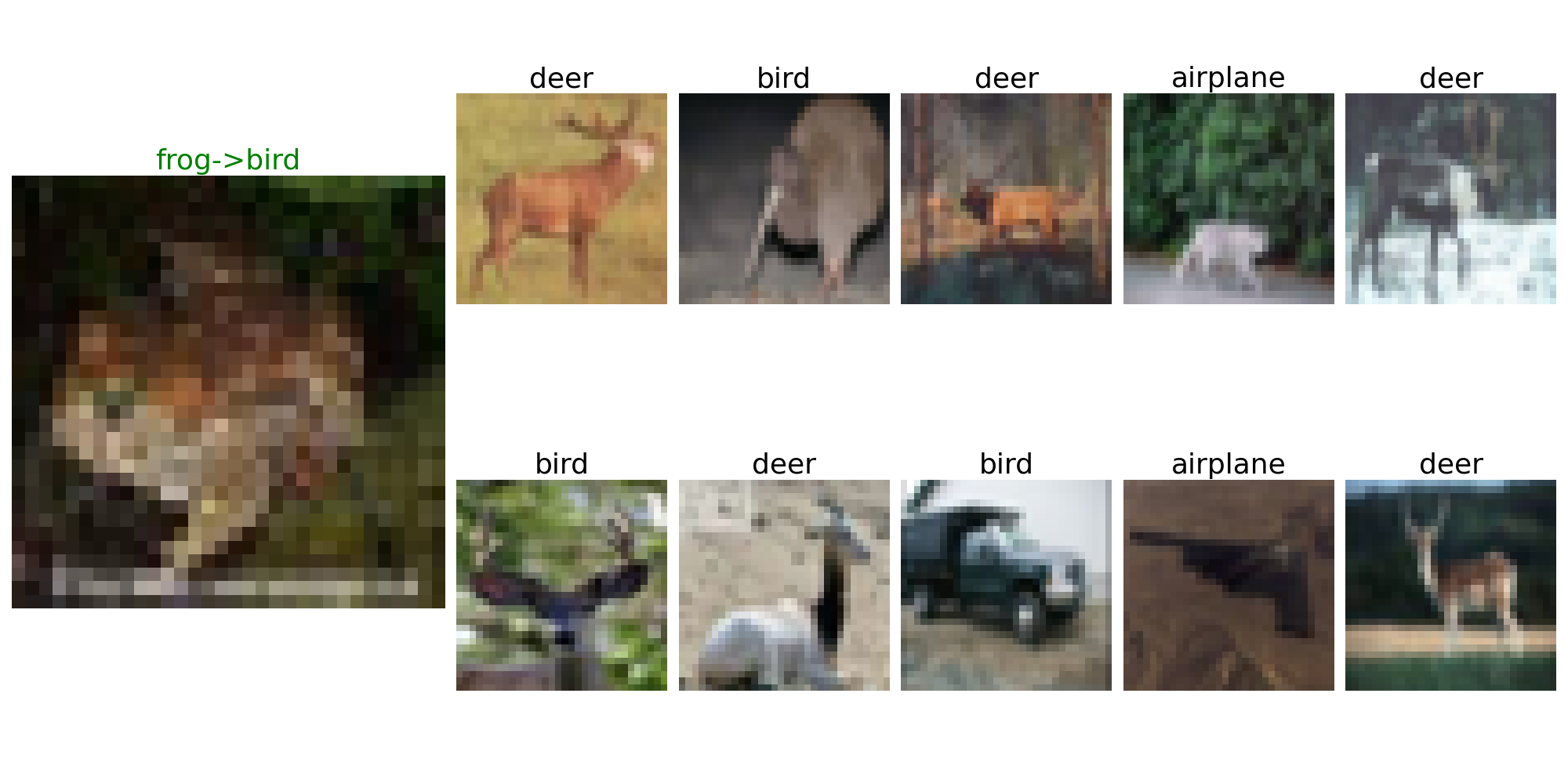} \\ \hline
    \includegraphics[width=0.45\textwidth]{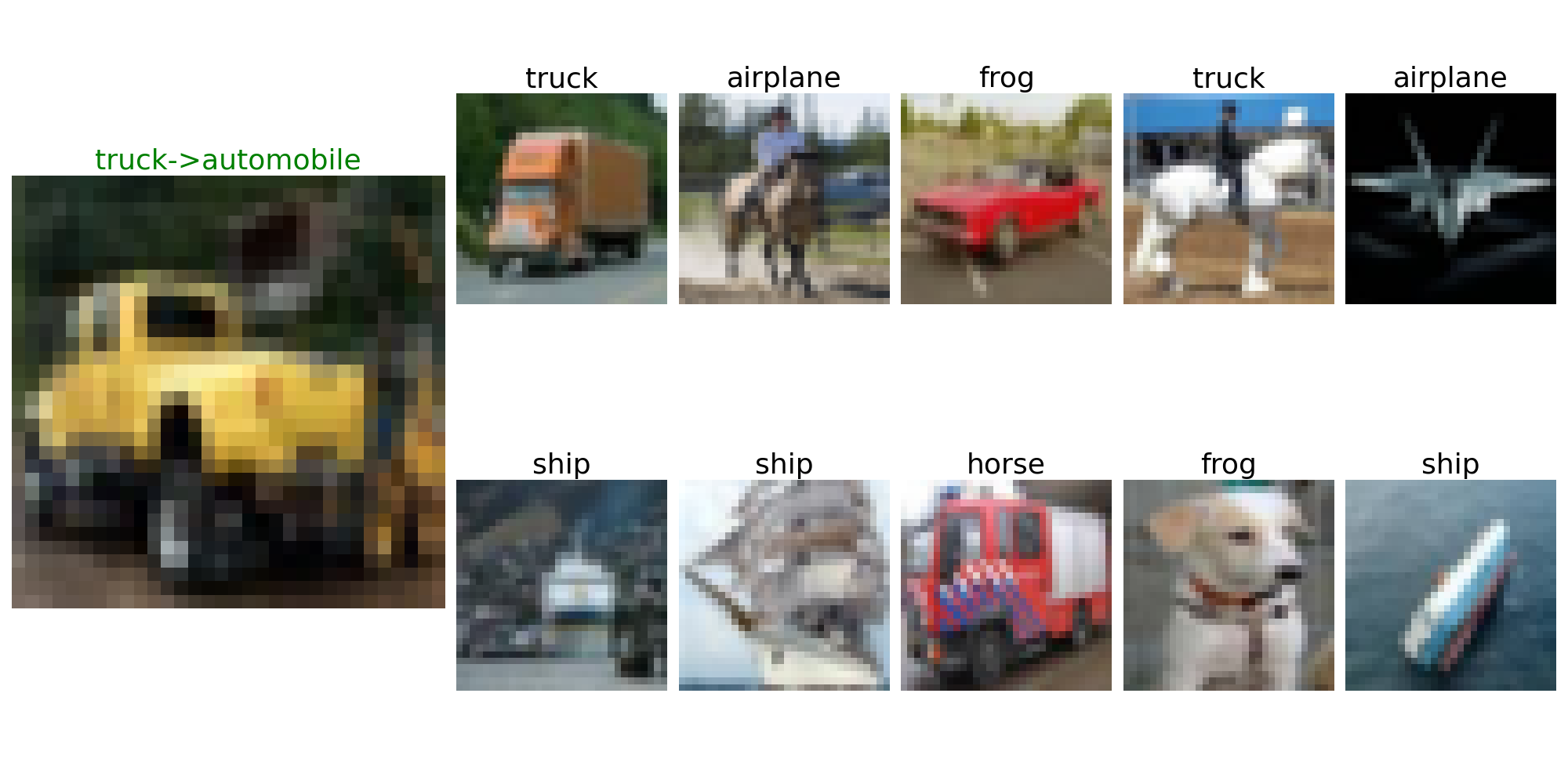} &
    \includegraphics[width=0.45\textwidth]{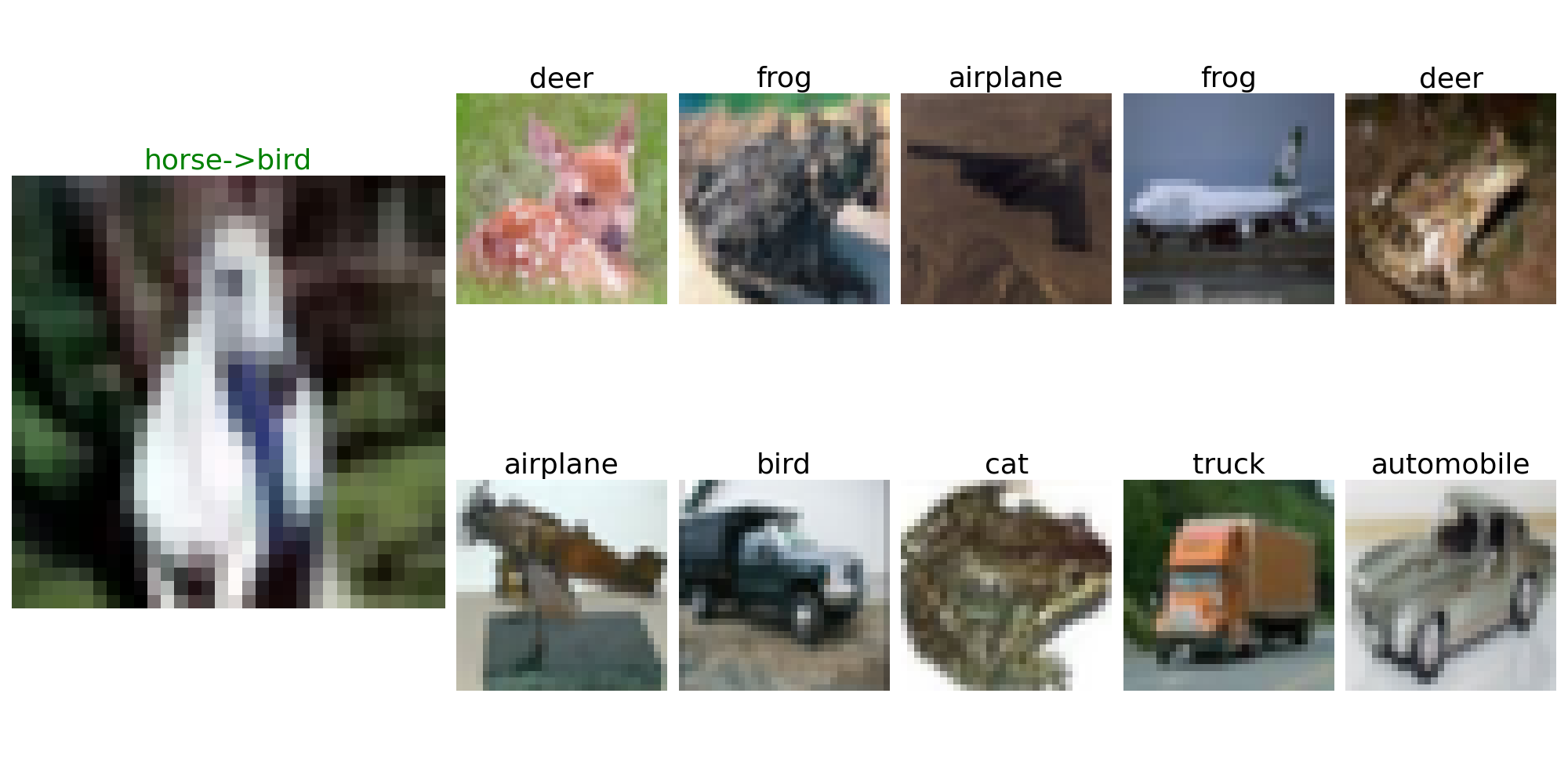}
    \end{tabular}
    \caption{Top-10 related test data tracing of mispredicted data on cifar-10 dataset with 30\%  noise data.}
      \label{tab:trace_vis3}
\end{table}

\end{document}